\documentclass[final,12pt]{colt2026}

\title[Unified Framework of Distributional Regret in Multi-Armed Bandits and RL]{Unified Framework of Distributional Regret in Multi-Armed Bandits and Reinforcement Learning}

\usepackage{times}
\usepackage{abbreviations}
\usepackage{booktabs}
\usepackage{makecell}
\usepackage{algorithm}
\RestyleAlgo{linesnumbered,ruled}
\LinesNumbered
\usepackage{enumitem}

\newcommand{\gap}{\mathrm{gap}}
\newcommand{\rgap}{\overline{\gap}}
\newcommand{\vgap}{v_{\gap}}
\newcommand{\Vmax}{V_{\max}}

\newcommand{\sigexp}{\sigma_{\mathrm{exp}}}
\newcommand{\sigsub}{\sigma_{\mathrm{max}}}
\newcommand{\Qrange}{\WW^*}
\newcommand{\Qdiff}{\WW_{\mathrm{diff}}^*}
\newcommand{\Wdiff}[2]{W_{{#1}, \mathrm{diff}}^{#2}}
\newcommand{\Wmin}[2]{W_{{#1}, \mathrm{min}}^{#2}}
\newcommand{\Valpha}{V_{\alpha}}
\newcommand{\kap}{\kappa}
\newcommand{\kapgap}{\kappa_{\gap}}

\coltauthor{
 \Name{Harin Lee} \Email{leeharin@cs.washington.edu}\\
 \addr University of Washington
 \AND
 \Name{Min-hwan Oh} \Email{minoh@snu.ac.kr}\\
 \addr Seoul National University
}

\begin{document}

\maketitle

\begin{abstract}
We study the distribution of regret in stochastic multi-armed bandits and episodic reinforcement learning through a unified framework.
We formalize a \emph{distributional regret bound} as a probabilistic guarantee that holds \emph{uniformly} over all confidence levels $\delta \in (0,1]$, thereby characterizing the regret distribution across the full range of $\delta$.
We present a simple UCBVI-style algorithm with exploration bonus
$\min\{c_{1,k}/N,\; c_{2,k}/\sqrt{N}\}$, where $N$ denotes the visit count and $(c_{1,k},c_{2,k})$ are user-specified parameters.
For arbitrary parameter sequences, we derive general gap-independent and gap-dependent distributional regret bounds, yielding a principled characterization of how the parameters control the trade-off between expected performance, tail risk, and instance-dependent behavior.
In particular, our bounds achieve optimal trade-offs between expected and distributional regret in both minimax and instance-dependent regimes.
As a special case, for multi-armed bandits with $A$ arms and horizon $T$, we obtain a distributional regret bound of order $\mathcal{O}\big(\sqrt{AT}\log(1/\delta)\big)$, confirming the conjecture of \citet[Section~17.1]{lattimore2020bandit} for the first time.
\end{abstract}

\begin{keywords}
  distributional regret, multi-armed bandit, reinforcement learning
\end{keywords}

\section{Introduction}

In online decision-making problems---including stochastic multi-armed bandits (MAB)~\citep{lattimore2020bandit} and reinforcement learning (RL)~\citep{sutton1998}---an agent repeatedly interacts with an unknown environment by selecting actions and observing stochastic rewards.
A standard measure of performance is \emph{regret}, defined as the gap between the cumulative reward achieved by an optimal strategy and that obtained by the agent.
Since the environment is stochastic and the agent may employ a randomized action-selection strategy, regret is a random variable; consequently, performance is typically assessed via expected regret or via high-probability guarantees at a prescribed confidence level.

While expectation and fixed-confidence bounds are useful quantities, they do not fully describe the distribution of regret.
Recent work has highlighted that tail behavior can be subtle even in classical settings: \citet{fan2025fragility} show that asymptotically optimal bandit algorithms can exhibit heavy-tailed regret distributions of $\PP(\mathrm{Regret} > x) \approx \frac{1}{x}$.
\citet{simchi2023regret,simchi2025simple} 
propose bonus designs
that accelerates the decay of the tail to $\Ocal( \exp( - x^\gamma))$ for a tunable  parameter $\gamma \in (0, 1]$ and further investigate the trade-off between the tail distribution and the instance-dependent regret.
For RL, the distributional picture is considerably less complete: while a recent work~\citep{khodadadian2025tail} studies a notion of distributional regret in RL, near-optimal characterizations and the corresponding optimal trade-offs remain unknown. 
Even in MAB, the optimality of existing distributional guarantees is not fully resolved.

In this work, we provide a unified framework of distributional regret analysis for MAB and RL, and we establish regret bounds that meet existing lower bounds in multiple regimes.

\paragraph{Main Contribution.}
We propose a simple and flexible algorithm for MAB and RL, $\AlgName$, and analyze the distribution of its cumulative regret.
The algorithm is a $\texttt{UCBVI}$-type method with a bonus term of $\frac{c_{1, k}}{N^k(s, a)} \land \frac{c_{2, k}}{\sqrt{N^k(s, a)}}$, where $c_{1, k}$ and $c_{2, k}$ are input parameters and $N^k(s, a)$ is the visit count of the state-action pair $(s, a)$.
Our theoretical guarantees offer the following novelties.

\begin{itemize}
    \item 
    We study a distributional regret bound, defined as a function of $\delta$ that upper-bounds the cumulative regret with probability at least $1-\delta$ \emph{simultaneously for all} $\delta \in (0,1]$ (see Section~\ref{sec:distribution regret}).
    The resulting $\delta$-dependence directly captures the distributional properties of the regret.
    A distributional regret bound implies (i) high-probability bounds at any prescribed confidence level, (ii) an expected regret bound via integration over $\delta$, and (iii) the light-tailed risk notion of \citet{simchi2025simple}.
    In particular, converting our uniform-in-$\delta$ bounds to expectation avoids the extra $\log K$ factors that commonly arise when one derives expected regret, 
    where $K$ is the number of episodes (see Corollary~\ref{cor:c_1 only 1}).
    
    \item 
    We introduce a regularity assumption that bounds the sub-exponential norm of the reward and the optimal value of the next state.
    This assumption strictly generalizes (and is not limited to) the standard sub-Gaussian noise assumption in stochastic MAB and the bounded-reward assumption in RL, enabling a single unified framework for regret analysis that covers both settings (see Section~\ref{sec:assumption}).
    
    \item We establish both gap-independent and gap-dependent distributional regret bounds for arbitrary input parameters $\{c_{1, k} \}_{k=1}^\infty, \{c_{2, k}\}_{k=1}^\infty$ (Theorems~\ref{thm:gap-independent bound} and~\ref{thm:gap-dependent bound}).
    Our results rigorously characterize how these two bonus parameters balance the trade-offs between the expected bound and the tail distribution of worst-case and instance-dependent regret, and we show that our analyses achieve optimal trade-offs in both.

    \item 
    In the MAB setting with $A$ arms and horizon $T$, we obtain a distributional regret bound of $\mathcal{O}(\sqrt{AT}\log(1/\delta))$ together with an expected regret bound of $\mathcal{O}(\sqrt{AT})$ (Theorem~\ref{thm:bandit 1}), 
    matching minimax lower bounds up to constant factors.
    To the best of our knowledge, this is the tightest known regret guarantee for MAB, and it confirms the conjecture of \citet[Section~17.1]{lattimore2020bandit}.
\end{itemize}

\subsection{Related Works}

\paragraph{MAB and RL with Expected and High-probability Regret Guarantees.}
Near-optimal expected and high-probability regret bounds have been established in numerous works for MAB~\citep{auer2002finite,audibert2009minimax,agrawal2012analysis,bubeck2012regret,degenne2016anytime,menard2017minimax,lattimore2018refining,lattimore2020bandit,jin2023thompson} and episodic RL~\citep{azar2017minimax,zanette2019tighter,simchowitz2019non,dann2021beyond,zhang2021reinforcement,tiapkin2022dirichlet,zhou2023sharp,zhang2024settling,lee2025minimax}.
However, the distributional behavior of regret has been far less studied.

\paragraph{Comparison with \citet{lee2025minimax}.}
The work most closely related to ours is \citet{lee2025minimax}, and our results generalize the analysis of \citet{lee2025minimax}.
The primary focus of \citet{lee2025minimax} is to achieve a minimax optimal high-probability regret bound 
under fixed input parameters.
We generalize and improve their analysis framework, providing distributional regret bounds and instance-dependent bounds under arbitrary input parameters.

\paragraph{Distributional Regret Bound.}
\citet{neu2015explore} shows that in adversarial bandits, \texttt{EXP3-IX}~\citep{kocak2014efficient} achieves a high-probability regret bound of $\Ocal\big(\sqrt{AT}\big(\sqrt{\log A} + \tfrac{\log(1/\delta)}{\sqrt{\log A}}\big)\big)$ uniformly for all $\delta \in (0,1]$.
\citet{simchi2022simple,simchi2023stochastic} and their extensions~\citep{simchi2025simple,simchi2023regret} study stochastic MAB.
These works design specific bonus terms for UCB-type algorithms and upper-bound $\PP(\mathrm{Regret} > x)$, yielding nearly exponential decay.
Their analyses extend to instance-dependent bounds and provide lower bounds for the trade-off between the regret distribution and expected regret.
\citet{zhu2025adaptive} apply similar ideas to Thompson sampling, obtaining exponential decay for regret and error rates in best-arm identification.
However, these approaches typically incur additional logarithmic factors in expected regret and offer limited flexibility in balancing distributional and expected guarantees; moreover, their analyses are based on an independent framework disabling unified analysis.
\citet{khodadadian2025tail} extend related ideas to RL, but their bound appears loose, incurring both $\gap_{\min}^{-1}$ and $\sqrt{K}$ terms simultaneously (where $K$ is the number of episodes), as well as a particularly large $\Ocal(H^6)$ dependence on the horizon length~$H$.

\paragraph{Comparison with \citet{simchi2023regret,simchi2025simple}.}
Our framework and results generalize those of \citet{simchi2023regret,simchi2025simple}, most notably by extending the problem setting from MAB to episodic RL.
While our algorithm shares structural similarities with the bonus design in \citet{simchi2023regret} (and also \citet{lee2025minimax}), our regret analysis is substantially different.
Moreover, our theory accommodates arbitrary input parameter sequences, whereas the guarantees in \citet{simchi2023regret,simchi2025simple} are derived for specific parameter choices.
Although a recent preprint version of \citet{simchi2023regret} relaxes the constraint on $c_{2,k}$, the other parameter $c_{1,k}$ still remains less flexible.
Finally, in the MAB setting, our bounds sharpen the guarantees in these prior works.

\section{Preliminaries}

\subsection{Markov Decision Process}

We consider an episodic Markov decision (MDP) process $M = (\Scal, \Acal, P, r, H)$, where $\Scal$ is the state space, $\Acal$ is the action space, $P : \Scal \times \Acal \rightarrow \Delta(\Scal)$ is the transition probability, $r : \Scal \times \Acal \rightarrow \RR$ is the reward function, and $H$ is the horizon length in each episode.
We consider the tabular case, where $\Scal$ and $\Acal$ have finite cardinalities of $S$ and $A$, respectively.
The agent interacts with an MDP for iterations of episodes, where each episode consists of $H$ time steps.
At the beginning of the $k$-th episode, the environment chooses the initial state $s_1^k$, possibly adaptively.
For time steps $h = 1, \ldots, H$, the agent observes the current state $s_h^k$, takes an action $a_h^k$, and receives a stochastic reward of $R_h^k \in \RR$ with mean $r(s_h^k, a_h^k)$.
Then, the next state $s_{h+1}^k$ is sampled from $P(s_h^k, a_h^k)$.
\\
A policy $\pi =  \{\pi_h\}$ is a sequence of mappings $\pi_h : \Scal\rightarrow \Acal$ from the current state to an action.
We denote the set of all deterministic policies by $\Pi$.
For given policy $\pi$ and $h \in [H]$, we define the value function as $V_h^{\pi}(s) := \Expec_{\pi(\cdot \mid s_h = s)}[\sum_{j=h}^H r(s_j, a_j)]$, the expectation is taken over the trajectory starting from $s_h = s$ with the $j$-th action being $a_j= \pi_j(s_j)$.
Similarly, we define the action value function $Q_h^{\pi}(s, a):= \Expec_{\pi(\cdot \mid s_h = s, a_h = a)}[\sum_{j=h}^H r(s_j, a_j)]$.
The optimal value function is defined as $V_h^*(s) := \max_{\pi \in \Pi} V_h^{\pi}(s)$ and $Q_h^*(s, a) := \max_{\pi\in \Pi}Q_h^{\pi}(s, a)$.
The optimal policy $\pi^*$ is defined as the policy that satisfies $V_h^{\pi^*}(s) = V_h^*(s)$ for all $h \in [H]$ and $s \in \Scal$.
Given an algorithm $\Alg$ that chooses $\pi^1, \pi^2, \ldots, \pi^k, \ldots$ based on the prior observations, the cumulative regret of $\Alg$ in $M$ over $K$ episodes is defined as $\Reg_M^{\Alg}(K) := \sum_{k=1}^K \big(V_1^*(s_1^k) - V_1^{\pi^k}(s_1^k) \big)$.

\subsection{Multi-Armed Bandits}

We consider MAB instances as a special case of episodic MDPs with $H = S = 1$ and no transitions.
When focusing on the MAB setting, we denote the time steps and the horizon by $t$ and $T$, respectively.
The interaction is simplified to choosing an action $a_t \in \Acal$ (with $|\Acal| = A$)  and observing a reward $R_t$ with mean $r(a_t)$ for time steps $t = 1, 2, \ldots$.
We define the optimal action $a^* = \argmax_{a \in \Acal} r(a)$ as the action with the highest reward.
The cumulative regret of $\Alg$ in $M$ over $T$ time steps can be written as $\Reg_M^{\Alg}(T) := \sum_{t=1}^T(r(a^*) - r(a_t))$.

\subsection{Definitions and Notations}

For functions $\tilde{P} : \Scal \times \Acal \rightarrow \RR^\Scal$ and $V: \Scal \rightarrow \RR$, we denote the expectation of $V$ under $\tilde{P}(s, a)$ by $\tilde{P}V(s, a) := \sum_{s' \in \Scal} \tilde{P}(s' \mid s, a) V(s')$.
We denote the variance of $V$ under the true transition probability $P(s, a)$ by $\Var(V)(s, a):=\sum_{s' \in \Scal} P(s' \mid s, a)( V(s') - PV(s, a))^2$.
We define a filtration $\{\Fcal_h^k\}_{k, h}$ as $\Fcal_h^k:= \sigma(s_1^1, a_1^1, R_1^1, \ldots, s_h^k, a_h^k)$.
Note that $\Fcal_{H+1}^k = \Fcal_0^{k+1} := \sigma(s_1^1, a_1^1, R_1^1, \ldots, s_H^k, a_H^k, r_H^k, s_{H+1}^k)$.
In the MAB setting, the gap of an action is defined as $\gap(a) := r(a^*) - r(a)$.
In the RL setting, the gap of a state-action pair at time step $h$ is defined as $\gap_h(s, a) := V_h^*(s) - Q_h^*(s, a)$, where we denote $\gap(s, a) := \min_{h \in [H]} \gap_h(s, a)$.
In addition, we denote the minimum non-zero gap by $\gap_{\min} := \min_{(h, s, a) \in [H] \times \Scal \times \Acal, \gap_h(s, a) > 0} \gap_h(s, a)$.
For a natural number $N \in \NN$, let $[N]:=\{1, 2, \ldots, N\}$.
For two real numbers $a$ and $b$, we define $a \lor b$ as $\max\{ a, b\}$ and $a \land b$ as $\min\{a, b\}$.
We also write $(a)_+$ for $a \lor 0$.

\section{Distributional Regret}
\label{sec:distribution regret}

In this section, we define \emph{distributional regret}, which corresponds to the upper-quantile function (equivalently, the inverse complementary cumulative distribution function) of the regret.

\begin{definition}[Distributional regret]
    A \textbf{distributional regret} $\Reg_M^{\Alg}(K, \delta)$ is a deterministic real-valued function of an algorithm $\Alg$, an MDP $M$, the number of episodes $K \in \NN$, and a failure probability $\delta \in (0, 1]$, defined as $\Reg_M^{\Alg}(K, \delta) := \inf\{ x \in \RR \mid \PP( \Reg_M^\Alg(K) > x) \le \delta \}$.
\end{definition}

\paragraph{Distributional regret bound.}
Equivalently, for any $\delta \in (0,1]$, the regret satisfies $\Reg_M^\Alg(K) \le \Reg_M^{\Alg}(K,\delta)$ with probability at least $1-\delta$, and $\Reg_M^{\Alg}(K,\delta)$ is the smallest value with this property.
Our goal is to provide an explicit upper bound on $\Reg_M^{\Alg}(K,\delta)$ as a function of the same inputs; we refer to such an upper bound as a \emph{distributional regret bound}.

\paragraph{Generality of distributional regret.}
The $\delta$-dependence of a distributional regret bound captures the tail distribution of regret at any given level $\delta$.
For instance, light-tailed risk defined in \citet{simchi2025simple} translates to $\Reg_M^{\Alg}(K, \delta) = \mathrm{poly}\log (\frac{1}{\delta})$.
A distributional regret bound implies high-probability bounds for any failure probability.
It also implies an expected regret bound via the identity $\EE[ \Reg_M^\Alg(K) ] = \int_0^1 \Reg_M^{\Alg} (K, \delta) \, d \delta$.
This integration technique replaces $\log \frac{1}{\delta}$ factors by constant factors in the expected regret bound, which shaves off the $\log K$ factor that typically arises in high-probability-based approaches where $\delta$ is often set as $\delta = \frac{1}{K}$.
To the best of our knowledge, Corollary~\ref{cor:c_1 only 1} achieves the sharpest logarithmic factor in the worst-case expected regret bound for RL using this technique.

\section{Assumptions}
\label{sec:assumption}

In this section, we present the assumptions for our analysis.
They encompass the standard sub-Gaussian noise assumption from the MAB literature and the bounded reward assumption standard in RL.
The first assumption regularizes the scale of the value function.

\begin{assumption}[Boundedness]
\label{assm:boundedness}
    There exists a known value $\Vmax \ge 0$ such that for all $h \in [H]$, $s \in \Scal$, and policy $\pi \in \Pi$, we have $0 \le V_h^{\pi}(s) \le \Vmax$.
\end{assumption}

The second assumption is for the concentration of measure.
The standard assumptions in recent MAB and RL literature are slightly different.
The bounded reward assumption in the RL literature assumes $\sum_{h=1}^H R_h \in [0, \Vmax]$~\citep{zanette2019tighter} or $R_h \in [0, \Vmax]$~\citep{lee2025minimax}, which allows leveraging the reward variance via Bernstein's inequality, and simultaneously guarantees that the sum of the variances over a trajectory is at most $\Vmax^2$ without $H$ dependence, incentivizing the use of variances.
However, this property is difficult to capture with a uniform sub-Gaussian assumption, making the two frameworks seemingly incompatible.
Addressing this issue, we propose a unified setting that assumes bounded sub-exponential norms on the random variables $R_h^k + V_{h+1}^*(s_{h+1}^k)$, whose means are $Q_h^*(s_h^k, a_h^k)$.
We first provide the definition of sub-exponential random variables, and then formally present the assumption.

\begin{definition}[Sub-exponential random variable~\citep{wainwright2019high}]
    Let $X$ be a random variable, $\Fcal$ be a $\sigma$-algebra, and $\sigma, \alpha \ge 0$ be $\Fcal$-measurable random variables.
    $X$ is \textbf{$\Fcal$-conditionally $(\sigma, \alpha)$-sub-exponential} if $\EE[ \exp( \lambda X ) \mid \Fcal] \le \exp(\frac{\sigma^2 \lambda^2}{2} )$ holds almost surely for all $\lambda \in [-\frac{1}{\alpha}, \frac{1}{\alpha}]$.
    If $\alpha = 0$, we assume $[-\frac{1}{0}, \frac{1}{0}]:= \RR$, and we say $X$ is \textbf{$\Fcal$-conditionally $\sigma^2$-sub-Gaussian}.
\end{definition}

\begin{assumption}[Conditional sub-exponentiality]
\label{assm:subexponential}
    There exist a variance proxy function $\sigexp : [H] \times \Scal \times \Acal \rightarrow \RR_{\ge 0}$, and a constant $\Valpha$ unknown to the agent such that $R_h^k + V_{h+1}^*(s_{h+1}^k) - Q^*(s_h^k, a_h^k)$ is $\Fcal_h^k$-conditionally $(\sigexp(h, s_h^k, a_h^k), \Valpha)$-sub-exponential.
\end{assumption}

The sub-exponential norm of a bounded random variable is proportional to its variance when $\alpha$ is set to the range of the random variable (see Lemma~\ref{lma:bounded is subexponential}).
Hence, under the bounded reward assumption, $\sigexp^2(h, s, a)$ can be set as the variance of $R_h^k + V_{h+1}^*$ with $\Valpha$ proportional to $\Vmax$.
When the reward noise is independently $\sigma^2$-sub-Gaussian, $\sigexp^2(h, s, a)$ can be set as $\sigma^2 + 2 \Var(V_{h+1}^*)(s, a)$.
Furthermore, our assumption extends to sub-exponential reward noise, such as the exponential, chi-squared, and Poisson distributions.
Analogous to cumulative variance in RL, we define the sum of variance proxies.

\begin{definition}[Total variance proxy function]
    For $h \in [H]$, $s \in \Scal$, and $\pi \in \Pi$, the \textbf{total variance proxy function} is defined as $W_h^{\pi}(s) := \Expec_{\pi(\cdot \mid s_h = s)} [ \sum_{j=h}^H \frac{1}{2} \sigexp^2(j, s_j, a_j)]$.
\end{definition}
We further define additional notations related to this function.
We denote the maximum of the function by $\Qrange := \max_{h \in [H], s \in \Scal, \pi \in \Pi} W_h^{\pi}(s)$.
Since we also consider the range of $W_h^{\pi}$, we define the minimum of the function $\Wmin{h}{\pi} := \min_{s \in \Scal} W_h^{\pi}(s)$ for fixed $h \in [H]$ and $\pi \in \Pi$, the difference from the minimum $\Wdiff{h}{\pi}(s) := W_h^{\pi}(s) - \Wmin{h}{\pi}$, and the maximum range of the function $\Qdiff := \max_{h \in [H], s \in \Scal, \pi \in \Pi} \Wdiff{h}{\pi}(s)$.
Table~\ref{tab:assumption example} provides example bounds for $\sigexp^2(h, s, a)$, $\Valpha$, $\Qrange$, and $\Qdiff$ under the conventional settings.
Their derivation is presented in Appendix~\ref{appx:proof of tab 1}.
We additionally define a constant $\sigsub := \max_{h\in [H], (s, a) \in \Scal \times \Acal} 2 \sigexp(h, s, a) \lor \sqrt{2 \Valpha \Vmax}$ that serves as a threshold for $c_{2, k}$, which arises from Lemma~\ref{lma:time uniform Hoeffding for exponential}.
Note that we do not assume the agent knows these values.

\begin{table}[tb]
    \centering
    \caption{Examples of upper bounds on the values $\sigexp^2$, $\Valpha$, $\Qrange$, and $\Qdiff$.}
    \begin{tabular}{l | c c c c}
        \toprule
         Setting & $\sigexp^2(h, s, a)$ & $\Valpha$ & $\Qrange$ & $\Qdiff$ \\
         \midrule
         $0 \le R_h \le \Vmax$ & $2 \Var(R_h + V_{h+1}^*)(s, a)$ & $2 \Vmax$ & $ 2 \Vmax^2$ & $ 2\Vmax^2$\\
         \hline
         $R_h$ $\sigma^2$-sub-Gaussian & $\sigma^2 + 2 \Var(V_{h+1}^*)(s, a)$ & $\Vmax$ &  $ \sigma^2 H + \Vmax^2$ & $\Vmax^2$\\
         \hline
         $R_h$ $\sigma^2$-sub-Gaussian, MAB & $\sigma^2$ & $0$ & $\sigma^2$ & 0\\
         \bottomrule
    \end{tabular}    
    \label{tab:assumption example}
\end{table}

\section{Algorithm: \texttt{EQO+}}

We introduce $\AlgName$ algorithm with a flexible bonus term, taking a sequence $\{(c_{1, k}, c_{2, k})\}_{k=1}^\infty$ as input and using the bonus term as $\frac{c_{1, k}}{N^k(s, a)} \land \frac{c_{2, k}}{\sqrt{N^k(s, a)}}$, where $N^k(s, a)$ is the visit count of the state-action pair.
Algorithm~\ref{alg:eqo2} describes the specific procedure for RL.
Its simpler version for MAB is analogous, where Lines 4, 6, 7, and 11 are unnecessary.
We provide a refined description for MAB in Appendix~\ref{appx:bandits}.
$\AlgName$ is an extension of $\texttt{EQO}$ in \citet{lee2025minimax}, where their algorithm uses the $\frac{c_{1, k}}{N^k(s, a)}$ term only.
We recover $\texttt{EQO}$ by setting $c_{2, k}= \infty$, which disables the $\frac{1}{\sqrt{N}}$-bonus term.

\begin{algorithm}[ht!]
\caption{$\AlgName$ (Exploration via Quasi-Optimism Plus)}
\label{alg:eqo2}
\SetKwInOut{Input}{Input}
\Input{$\{ (c_{1, k} ,c_{2, k})\}_{k=1}^\infty$}

\For{$k = 1, 2, \ldots,$}{
    $N^k(s, a) \gets \sum_{i=1}^{k-1} \sum_{h=1}^H \ind\{ (s_h^i, a_h^i) = (s, a) \}$ for all $(s, a) \in \Scal \times \Acal$\;

    $\hat{r}^k(s, a) \gets \frac{1}{N^k(s, a)} \sum_{i=1}^{k-1} \sum_{h=1}^H R_h^i \ind\{ (s_h^i, a_h^i) = (s, a)\}$ for all $(s, a) \in \Scal \times \Acal$\;

    $\hat{P}^k(s' | s, a) \gets \frac{1}{N^k(s, a)} \sum_{i=1}^{k-1} \sum_{h=1}^H \ind\{(s_h^i, a_h^i, s_{h+1}^i) = (s, a, s') \}, \forall (s, a, s') \in \Scal \times \Acal \times \Scal$\;

    $b^k(s, a) \gets \frac{c_{1, k}}{ N^k(s, a)} \land \frac{c_{2, k} }{\sqrt{N^k(s, a)}}$ for all $(s, a) \in \Scal \times \Acal$\;

    $V_{H+1}^k(s) \gets 0$ for all $s \in \Scal$\;
    
    \For{$h = H, H-1, \ldots, 1$}{
        \ForEach{$(s, a) \in \Scal \times \Acal$}{
            
            $Q_h^k(s, a) \gets \begin{cases}
                \big( \big( \hat{r}^k(s, a) + \hat{P}^k V_{h+1}^k(s, a) \big)_+ + b^k(s, a) \big) \land \Vmax
                & \text{if } N^k(s, a) > 0 \\
                \Vmax & \text{if } N^k(s, a) = 0
            \end{cases}
            $\;
        }
        $V_h^k(s) \gets \max_{a \in \Acal} Q_h^k(s, a)$ for all $s \in \Scal$\;
        
        $\pi^k_h(s) \gets \argmax_{a \in \Acal} Q_h^k(s, a)$ for all $s \in \Scal$\;
    }
    
    Execute $\pi^k$ and obtain $(s_1^k, a_1^k, R_1^k, \ldots, s_H^k, a_H^k, R_H^k, s_{H+1}^k)$\;
}
\end{algorithm}

\paragraph{Generality of $\AlgName$ algorithm.}
$\AlgName$ recovers several existing algorithms through specific choices of the input parameters.
When 
$c_{1, k} = \tilde{\Ocal}\big( \Vmax (\frac{k}{SA} \log \frac{1}{\delta_0})^{1/2} \big)$
and $c_{2, k} = \infty$ for some specific $\delta_0 \in (0, 1]$, we recover the bonus term considered in \citet{lee2025minimax}.
When $c_{1, k} = \infty$ and 
$c_{2, k} = \tilde{\Ocal}\big( \Vmax (\log \frac{1}{\delta_0})^{1/2}\big)$, 
we obtain $\texttt{UCBVI-CH}$ in \citet{azar2017minimax}, whose confidence bound is also known as the Hoeffding-style bound.
When $c_{1, k} = \infty$ and $c_{2, k} = \Ocal(S + K^{\alpha})$ for some $\alpha \in [0, 1]$, we have the algorithm by \citet{khodadadian2025tail}.
In the MAB setting, assigning $c_{1, t} = \Ocal\big( ( \frac{t}{A})^{\alpha}\sqrt{\log A} \big)$ and $c_{2, t} = \sqrt{t^\beta}$ for some constants $0 < \beta \le \alpha < 1$ recovers the bonus term in \citet{simchi2023regret}.

\section{Main Results}

In this section, we present distributional regret bounds for $\AlgName$.
In Section~\ref{sec:main result bandits}, we first present key results for the MAB setting, which also illustrate the role of the input parameters.
Then, in Section~\ref{sec:main result RL}, we turn to general distributional regret bounds for the RL setting, which also apply to the MAB setting, and allow arbitrary input parameters.
In Section~\ref{sec:corollaries}, we consider several representative parameter choices and discuss their implications.

Throughout this section, the algorithm under consideration is $\Alg := \AlgName(\{(c_{1,k},c_{2,k})\}_{k=1}^\infty)$, where the values of $\{(c_{1,k},c_{2,k})\}_{k=1}^\infty$ are specified in each theorem.

\subsection{Distributional Regret Bounds for Multi-Armed Bandits}
\label{sec:main result bandits}
In this section, we provide distributional regret bounds for the MAB setting.
Let $\Bcal(\sigma)$ denote the set of MAB instances with $\sigma^2$-sub-Gaussian reward noise and $\Vmax = 1$.
We present two special cases of input parameters that provide insights into their roles, which carry over to the RL setting.
The first result shows how the parameter $c_{1,t}$ controls the minimax regret.
Extended theorems and proofs are provided in Appendix~\ref{appx:bandits}.

\begin{theorem}
\label{thm:bandit 1}
    Suppose $M \in \Bcal(\sigma)$.
    Set $c_{1, t} = c_1$ for a constant $c_1 > 0$ and $c_{2, t} = \infty$, meaning that we use bonus term $\frac{c_1}{N}$ only.
    Then, the distributional regret of $\AlgName$ is bounded as
    \begin{align*}
        \Reg_M^{\Alg}(T, \delta) \le \frac{\sigma^2 T \log \frac{4}{\delta}}{c_1} + \frac{3}{2} c_1 A + \sum_{a \in \Acal} \gap(a).
    \end{align*}
    Taking $c_{1, t} = \sigma\sqrt{T/A}$ yields $\Reg_M^{\Alg}(T, \delta) =\Ocal(\sigma \sqrt{AT}\log \frac{1}{\delta}) $ and $\Expec [ \Reg_M^{\Alg}(T) ] = \Ocal( \sigma \sqrt{AT})$.
\end{theorem}
\paragraph{Discussion of Theorem~\ref{thm:bandit 1}.}
The distributional regret bound in Theorem~\ref{thm:bandit 1} shows that $c_1$ balances the $(T \log \tfrac{1}{\delta})/c_1$ term and the $c_1 A$ term.
To reduce the coefficient of the $\log \tfrac{1}{\delta}$ term (i.e., to obtain stronger distributional guarantees), $c_1$ must increase.
However, a larger $c_1$ increases the total regret through the $c_1 A$ term, and hence increases the expected regret bound linearly in $c_1$.
This inverse trade-off between the coefficient of the $\log \tfrac{1}{\delta}$ factor and the expected regret is optimal for all choices of $c_1 = \Omega(\sqrt{T/A})$ by Theorem~17.1 of \citet{lattimore2020bandit}.
\\
In particular, setting $c_1 = \sigma\sqrt{T/A}$ yields a distributional bound of $\Ocal\big(\sqrt{AT}\log \tfrac{1}{\delta}\big)$ and an expected bound of $\Ocal(\sqrt{AT})$, without any logarithmic factors in $A$ or $T$.
The expected regret bound is minimax-optimal up to only constant factors, and the distributional bound is optimal given this expected bound.
To the best of our knowledge, these rates are the tightest possible for MAB (up to constant factors).
The existence of an algorithm achieving them was conjectured based on the lower bound in \citet[Section~17.1]{lattimore2020bandit} but had not been established.
To the best of our knowledge, this is the first result to attain these bounds, thereby confirming that conjecture.\\

The next theorem shows how $\frac{c_{2, t}}{\sqrt{N}}$-type bonus term achieves an optimal gap-dependent regret bound when used together with a proper confidence bound.
To specify $c_{2, t}$, we first define the following concept.

\begin{definition}[Time-uniform UCB]
    A function $u(t, \delta) : \RR_{\ge 0} \times (0, 1] \rightarrow \RR_{\ge 0}$ induces a \textbf{time-uniform upper confidence bound (UCB)} if $\PP \big( \exists t  \in \NN, \exists a \in \Acal : | \hat{r}^t(a) - r(a) | > \frac{u(t, \delta)}{\sqrt{N^t(a)}} \big) \le \delta$ always holds.
\end{definition}
While many different functions can induce time-uniform UCBs, we refer to those with asymptotically optimal constants and logarithmic factors as \emph{tight time-uniform UCBs}; we provide a formal definition and an example in Appendix~\ref{appx:proof of bandit 2}.
The following theorem considers $c_{2,t} = u(t,\delta_t)$ for a sequence $\{\delta_t\}_{t=1}^\infty$.

\begin{theorem}
\label{thm:bandit 2}
    Suppose $M \in \Bcal(\sigma)$ and $u(t, \delta)$ induces a time-uniform UCB.
    For a decreasing sequence $\{\delta_t\}_{t=1}^\infty$, set $c_{2, t} = u(t, \delta_t)$ and $c_{1, t} = \infty$, meaning that the bonus term is $\frac{u(t, \delta_t)}{\sqrt{N}}$.
    Then,
    \begin{align*}
        \Reg_M^{\Alg}(T, \delta) \le \tau_2(\delta) \land T + \sum_{a \in \Acal}\gap(a) + \sum_{a \in \Acal, \gap(a) \ne 0} \frac{(c_{2, T} + u(t, \delta))^2}{\gap(a)}
        \, ,
    \end{align*}
    where $\tau_2(\delta) := \max\{ t \in \NN : \delta_t > \delta\}$.
    If $u(t, \delta)$ is tight and $\delta_t$ is appropriately decreasing, e.g., $\delta_t = (t \log t)^{-1}$, we have $\limsup_{T \rightarrow \infty} \frac{\Expec[\Reg_M^{\Alg}(T)]}{\log T} \le \sum_{a \in \Acal, \gap(a) \ne 0}\frac{2 \sigma^2 }{\gap(a)}$.
\end{theorem}

\paragraph{Discussion of Theorem~\ref{thm:bandit 2}.}
Theorem~\ref{thm:bandit 2} shows how $c_{2,t}$ governs the gap-dependent distributional regret, exhibiting behavior different from that of $c_{1,t}$.
When $c_{2,t} < u(t,\delta)$ (i.e., for the first $\tau_2(\delta)$ time steps), 
no non-trivial guarantee is provided.
After $c_{2, t}$ passes the threshold, the bound increases with $\sum_{a} c_{2, t}^2 / \gap(a)$, showing a more discrete behavior than the case of $c_{1, t}$.
\\
The sequence $\{\delta_t\}_t$ can be understood as an intermediate parameter that balances $c_{2, t}$ and $\tau_2(\delta)$.
Slowly diminishing $\{\delta_t\}_t$ decreases $c_{2, t}$ but increases $\tau_2(\delta)$, and vice versa.
For instance, setting $\delta_t = (t \log t)^{-1}$ yields the tightest constant factor in the asymptotic regret, but results in a super-linear dependence on $\frac{1}{\delta}$ as $\tau_2(\delta) \approx \frac{1}{\delta} \log \frac{1}{\delta}$.
If we accelerate the decay as $\delta_t = t^{-p}$ for some $p > 1$, then we obtain smaller $\tau_2(\delta) = (1/\delta)^{\frac{1}{p}}$, but the total regret scales linearly in $p$.
Considering $\delta_t = \exp(-t^\beta)$ for some $\beta > 0$ achieves a poly-logarithmic dependence on $\delta$ as $\tau_2(\delta) = \Ocal((\log(1/\delta))^{\frac{1}{\beta}})$ but the total regret scales with $T^\beta$.
This is consistent with the result of \citet{simchi2023regret}, where they show that this order of trade-off is optimal.
\begin{remark}
    Theorem~\ref{thm:bandit 2} states that a UCB algorithm with any tight time-uniform UCB and properly decreasing failure probabilities achieves the asymptotically optimal regret bound of \citet{lai1985asymptotically}, which may be of independent interest.
    Refer to Appendix~\ref{appx:proof of bandit 2} for the exact conditions.
\end{remark}

\subsection{Distributional Regret Bounds for Reinforcement Learning}
\label{sec:main result RL}

In this section, we provide distributional regret bounds for $\AlgName$ with arbitrary input in the RL setting, which naturally applies also to MAB due to our unified framework.
The proofs of the theorems are provided in Appendix~\ref{appx:rl}.
Let $\ell(\delta):= \log \frac{c HSA\log k}{\delta}$ denote certain logarithmic factor with an absolute constant $c$, where the dependence on $\delta$ is emphasized.
We define a quantity $\kap(\delta)$ that is analogous to $\tau_2(\delta)$ in Theorem~\ref{thm:bandit 2}.
For $\delta \in (0, 1]$, let
\begin{align*}
    \textstyle
    \kap(\delta) := \max \Big\{ k \in \NN : c_{1, k} <  (2 \sqrt{13 \Qdiff} \lor 2 \Valpha )\ell(\delta) \text{ or } 
    \textstyle
    c_{2, k} < \left(\sigsub + \frac{6\Qdiff \ell(\delta)}{c_{1, k}} \right) \sqrt{ \ell(\delta)} \Big\}
    \, .
\end{align*}
$\kap(\delta)$ is the number of episodes where the parameters $c_{1, k}$ and $c_{2, k}$ are too small to provide guarantees.
Such a term is unavoidable since for any fixed input parameters and $K$, we can take $\delta$ small enough to invalidate the analysis; note that the thresholds in the definition go to infinity as $\delta \rightarrow 0$.

We provide gap-independent and gap-dependent bounds for $\AlgName$ under arbitrary positive, non-decreasing input sequences $\{c_{1, k}\}_{k=1}^\infty$ and $\{c_{2, k}\}_{k=1}^\infty$.

\begin{theorem}
\label{thm:gap-independent bound}
    For any $K \in \NN$ and $\delta \in (0, 1]$, the distributional regret of $\Alg$ satisfies
    \begin{align*}
        \Reg_{\Mcal}^{\Alg} (K, \delta) & \le  \Vmax (\kap(\delta) \land K) + ( 16 c_{1, K} SA \log KH)) \land \left( 16 \sqrt{2} c_{2, K} \sqrt{HSAK} \right)
        \\
        &\qquad + \sum_{k=\kap(\delta) + 1}^K \frac{18 \Qrange \ell(\delta)}{c_{1, k}} + 72 \Vmax S^2 A \ell(\delta) \log 2KH
        \, .
    \end{align*}
\end{theorem}

\begin{theorem}
\label{thm:gap-dependent bound}
    Define 
    $\kapgap(\gap_{\min}, \delta) := \max\{ k \in \NN : c_{1, k} < \frac{36 \Qrange \ell(\delta)}{\gap_{\min}} \} $.
    For any $K \in \NN$ and $\delta \in (0, 1]$, the distributional regret of $\Alg$ satisfies
    \begin{align*}
        \Reg_{\Mcal}^{\Alg} (K, \delta) & \le \Vmax (\kap(\delta) \land K) + \!\!\!\! \sum_{(s, a) \in \Scal \times \Acal}  \!\!\!\left( \frac{4096 c_{2, K}^2}{\gap(s, a) \!\lor\! \frac{\gap_{\min}}{ H}} \right)\! \land \!\left(64 c_{1, K} \log \frac{64 c_{1, K} }{\gap(s, a)\! \lor\! \frac{\gap_{\min}}{ H}} \right)
        \\
        & \qquad + \sum_{k = \kap(\delta)+1}^{\kapgap(\gap_{\min}, \delta) \land K} \frac{144 \Qrange \ell(\delta)}{c_{1, k}} + 288 \Vmax S^2 A \ell(\delta) \log 2KH
    \end{align*}
\end{theorem}

\paragraph{Discussion of Theorems~\ref{thm:gap-independent bound} and~\ref{thm:gap-dependent bound}.}
Both bounds involve the minimum of $c_{1, K}$-related and $c_{2, K}$-related terms, which arises from the bonus term $(c_{1, k} / N) \land (c_{2, k} / \sqrt{N})$.
For most parameter choices we consider, the $c_{1, K}SA \log KH$ term is smaller in Theorem~\ref{thm:gap-independent bound}, whereas the $\sum_{s, a}c_{2, K}^2/\gap(s, a) $ term is smaller in Theorem~\ref{thm:gap-dependent bound}, dominating their respective bounds.
In this sense, we see that $c_{1, k}$ mainly controls the worst-case behavior, while $c_{2, k}$ mainly controls the instance-dependent behavior, corroborating the presentation of Section~\ref{sec:main result bandits}.
\\
In Theorem~\ref{thm:gap-independent bound}, the parameter $c_{1, k}$ balances the sum $ \sum_k (\Qrange \ell(\delta) / c_{1, k})$ and the term $c_{1, K} SA (\log KH)$, creating a trade-off parallel to Theorem~\ref{thm:bandit 1}.
We expect this trade-off to be optimal up to logarithmic factors, considering that certain MDP instances behave analogous to MAB instances with $SA$ actions~\citep{domingues2021episodic}.
In addition, setting $c_{1, k} = \tilde{\Theta}(\Vmax\sqrt{k/(SA)})$ yields the minimax optimal regret bound of $\tilde{\Ocal}(\Vmax \sqrt{SAK})$ (see Corollary~\ref{cor:c_1 only 1}).
\\
The parameter $c_{2, k}$ affects both bounds in a manner analogous to Theorem~\ref{thm:bandit 2}, where its growth rate balances the $\delta$-dependence in $\kap(\delta)$ and the overall distributional regret bound.
\\
In Theorem~\ref{thm:gap-dependent bound}, the sum $\sum_{k} \Qrange\ell(\delta) / c_{1, k}$ appears, which might seem problematic as it scales with $\sqrt{K}$ when $c_{1, k} = \Theta_k(\sqrt{k})$.
We show that the summand stops affecting the order of the distributional regret bound once $\Qrange\ell(\delta) / c_{1, k} \lesssim \gap_{\min}^{-1}$ holds, implying that the sum becomes independent of $K$.
When $c_{1, k} = \Theta_k(\sqrt{k})$, this sum scales moderately with $\frac{1}{\gap_{\min}}$ (see Corollary~\ref{cor:eqo and hoeffding}).

\begin{remark}
    The condition that $\{c_{1, k}\}_{k=1}^\infty$ and $\{c_{2, k}\}_{k=1}^\infty$ be increasing can be easily relaxed by redefining $\kap(\delta)$ accordingly and by replacing $c_{1, K}$ and $c_{2, K}$ in the bounds with $\max_{k \in [K]} c_{1, k}$ and $\max_{k \in [K]}c_{2, k}$, respectively. 
    Then, the bounds apply to arbitrary positive input parameters.
\end{remark}

\subsection{Corollaries for Specific Parameter Choices}
\label{sec:corollaries}

In this section, we propose exemplary parameter choices and present the resulting bounds.
Full versions of the corollaries and their proofs are presented in Appendix~\ref{appx:corollaries}.
\\
We let $\Mcal$ be a set of MDPs with the time horizon $H$, $S$ states, $A$ actions, and $\Vmax$ as the range of the value function, and that further satisfy $\Qrange \le 2 \Vmax^2$ and $\Valpha \le 2 \Vmax$.
$\Mcal$ includes MDPs with $\Vmax$-bounded reward or $\sigma^2$-sub-Guassian reward noise with $\sigma^2 H \le 2 \Vmax^2$.

First, we present a standard choice of $c_{1, k} \approx \Vmax \sqrt{k/ (SA)}$ with logarithmic factors.

\begin{corollary}
\label{cor:c_1 only 1}
    Set $c_{1, k} = c_1 \Vmax \sqrt{\frac{k \ell(1) }{SA \log KH}}$ for some constant $c_1 > 0$ and $c_{2, k} = \infty$.
    Then,
    \begin{equation}
        \begin{split}
        \sup_{M \in \Mcal} \Reg_M^{\Alg}(K, \delta) & \le \left( \frac{36 \log \frac{1}{\delta}}{c_1 \log HSA} + \frac{36}{c_1} + 8 c_{1, k} \right)\Vmax \sqrt{KSA \ell(1) \log KH}
        \\
        & \qquad + 72 \Vmax S^2 A \ell(\delta) \log 2KH
        \, .
        \end{split}
        \label{eq:c_1 only}
    \end{equation}
    The expected regret is bounded as
    \begin{align*}
        \sup_{M \in \Mcal} \Expec[\Reg_M^{\Alg}(K)] 
        =  \Ocal \Big( & \Vmax \sqrt{SAK (\log KH) (\log HSA(\log K))} 
        \\
        & \qquad + \Vmax S^2 A(\log KH) (\log HSA(\log K)) \Big)
        \, .
    \end{align*}
    Furthermore, if $K$ is known and $c_{1, k} = c_1 \Vmax \sqrt{\frac{K \log HSA}{SA \log KH}}$, then the $\log HSA(\log K)$ factor in the square root reduces to $\log HSA$.
\end{corollary}

This distributional regret bound achieves the minimax optimal $\tilde{\Ocal}(\Vmax \sqrt{SAK})$ bound, while simultaneously achieving a linear dependence on $\log \frac{1}{\delta}$.
The logarithmic factor in the leading term of the expected bound is $\sqrt{(\log KH)(\log HSA(\log K))}$, where the $\sqrt{\log \log K}$ factor can be removed when $K$ is known.
To the best of our knowledge, this $\sqrt{\log K}$-dependence is the sharpest among the minimax optimal RL algorithms.

The next corollary shows that we can combine Hoeffding's bound in addition to Corollary~\ref{cor:c_1 only 1}.

\begin{corollary}
\label{cor:eqo and hoeffding}
   Assume that $\sigsub^2 \le 2 \Vmax^2$, and set $c_{1, k} = c_1 \Vmax \sqrt{\frac{k \ell(1) }{SA \log KH}}$ for some constant $c_1 > 0$ and 
   $c_{2, k} = c_2 \Vmax \sqrt{\log 32HSAk}$ for some constant $c_2  \ge 2$.
    Then, we have
    \begin{align*}
        \Reg_M^{\Alg} (K, \delta) & = \Vmax \exp\left( \Ocal\left(  \frac{1}{c_2^2} \log \frac{1}{\delta} \right)\right) +  \Ocal \left( \sum_{(s, a) \in \Scal \times \Acal} \frac{c_2^2 \Vmax^2 \log HSAK}{\gap(s, a) \lor \frac{\gap_{\min}}{H}} \right)
        \\
        & \quad + \Ocal \left( \frac{\Vmax^2SA(\log KH) (\ell(\delta))^2}{c_1^2 \gap_{\min} (\log HSA)} + \Vmax S^2 A (\log KH) \ell(\delta) \right)
        \, .
    \end{align*}
    The worst-case bound is the sum of the bound in Eq.~\eqref{eq:c_1 only} and  $\exp(\Ocal(c_2^{-2} \log(1/\delta)))$ term.
\end{corollary}

Corollary~\ref{cor:eqo and hoeffding} shows an interesting result: by taking the minimum of the bonus terms of $\texttt{EQO}$~\citep{lee2025minimax} and $\texttt{UCBVI-CH}$~\citep{azar2017minimax}, one can achieve both minimax optimal and instance-dependent $\log K$ regret bounds.
As a trade-off, the distributional bound grows polynomially in $1/\delta$, which is consistent with Theorem~\ref{thm:bandit 2}.
We note that the orders of $\Vmax$, $\log K$, and $\gap(s, a)$ match the lower bound of Proposition 2.2 in~\citet{simchowitz2019non}.
\\
The $\Ocal \left( \frac{\Vmax^2SA(\log KH) (\ell(\delta))^2}{c_1^2 \gap_{\min} (\log HSA)} \right) $ term comes from the sum of $\frac{\Qrange\ell(\delta)}{c_{1, k}}$ terms, where the $\log K$ factor can be improved to $K$-independent logarithmic factors.
Compared to the bound of \citet{simchi2023regret}, we improve the $\gap_{\min}^{-2}$ dependence to $\gap_{\min}^{-1}$.
Despite this, this term may still appear large and could potentially be the leading term.
We show that the $\gap_{\min}$ dependence can be improved for MDPs with specific structures in Appendix~\ref{appx:proof of gap-dependent bound}.

We also consider the parameter choice by \citet{simchi2023regret}.

\begin{corollary}
\label{cor:simchi levi}
    Let $c_1 > 0, c_2 > 0$, $\alpha \in [\frac{1}{2}, 1]$ and $0 < \beta \le \alpha$ be constants.
    Set $c_{1, k} = c_1 \Vmax (\frac{k}{SA})^{\alpha}$ and $c_{2, k} = c_2 \sqrt{k^{\beta}}$.
    Then, we have
    \begin{align*}
        \sup_{M \in \Mcal} \Reg_M^{\Alg}(K, \delta)
        & = \tilde{\Ocal}_{K, \delta} \left( \left( \log \frac{1}{\delta} \right)^{\frac{1}{\beta}} + K^{1-\alpha} \log \frac{1}{\delta} + K^{\alpha} \land K^{\beta + \frac{1}{2}} \right)
        \qquad \text{and}
    \end{align*} 
    \begin{align*}
        \Reg_M^{\Alg} (K, \delta) & = \tilde{\Ocal}_{K, \delta, \gap} \left( \left( \log \frac{1}{\delta} \right)^{\frac{1}{\beta}} + \left(\frac{1}{\gap_{\min}}\right)^{\frac{1}{\alpha} - 1} \left( \log \frac{1}{\delta} \right)^{\frac{1}{\alpha}} + \sum_{s, a} \frac{K^\beta}{\gap(s, a) \lor \frac{\gap_{\min}}{H}} \right)
        \, .
    \end{align*}
\end{corollary}

The hyper-parameters $\alpha$ and $\beta$ provide control over the trade-offs between different types of regret bounds we have discussed.
$\alpha$ balances the coefficient of $\log \frac{1}{\delta}$ and the expected regret bound in the gap-independent case, whereas $\beta$ controls the orders of the $\log \frac{1}{\delta}$ factor and the instance-dependent regret.
When restricted to the MAB setting, our results achieve the same order of the gap-independent bound as \citet{simchi2023regret} while improving the gap-dependent bound.
Specifically, we improve the previous $(\frac{1}{\gap_{\min}})^{\frac{1}{\alpha}}$ factor to $(\frac{1}{\gap_{\min}})^{\frac{1}{\alpha} -1}$, and eliminate a $\Ocal ( (\log \frac{1}{\delta})\sum_{s, a}\frac{1}{\gap(s, a)})$ term.
A more detailed comparison with \citet{simchi2023regret} and \citet{khodadadian2025tail} is provided in Appendix~\ref{appx:comparison}.

\section{Proof Sketch}

In this section, we provide a sketch of the proof of Theorems~\ref{thm:gap-independent bound} and~\ref{thm:gap-dependent bound}.
At the end of the section, we additionally sketch the technique used in Theorem~\ref{thm:bandit 1} for the tightest logarithmic factors.

The main challenge in deriving distributional regret bounds is to remove the $\delta$-dependence in the algorithm while maintaining the high-probability guarantees.
Standard UCB-based analyses are invalidated since confidence bounds require specified failure probabilities.
In Theorem~\ref{thm:bandit 2}, we proposed a novel method of using decreasing failure probabilities and considering the corresponding warm-up time.
While this technique provides optimal trade-offs between the $\delta$-dependence and the instance-dependent regret bound in some regimes, we require additional techniques to obtain the optimal trade-off for the minimax regret bound.
To this end, we combine and extend the \emph{quasi-optimism} analysis by \citet{lee2025minimax}, where quasi-optimism means that the value estimates are almost optimistic, allowing for potential underestimation.
\\
We fix $\delta \in (0, 1]$ and set a good event $\Ecal(\delta)$ whose probability is at least $1 - \delta$.
We do not make any guarantees for the first $\kap(\delta)$ episodes of warm-up time, incurring regret of at most $\Vmax \kap(\delta)$.
For $k > \kap(\delta)$, the input parameters $c_{1, k}$ and $c_{2, k}$ become large enough to serve as (quasi-)optimistic bonus terms.
Specifically, we have the following lemma.

\begin{lemma}[Quasi-optimism (informal)]
\label{lma:quasi optimism informal}
    Under $\Ecal(\delta)$, the value estimate $V_h^k(s)$ computed in Algorithm~\ref{alg:eqo2} satisfies $V_h^*(s) - V_h^k(s) \lesssim \frac{ \Qrange \ell(\delta)}{c_{1, k}}$ for all $h \in [H]$, $s \in \Scal$, and $k > \kap(\delta)$.
\end{lemma}
\lemmaref{lma:quasi optimism informal} is the main tool that allows us to use $\delta$-independent bonus terms while maintaining high-probability guarantees.
Instead of guaranteeing overestimation of the algorithm as in the usual optimism-based analyses, it provides a distributional bound on the amount of potential underestimation of the given bonus terms.
We note that the quasi-optimism analysis in \citet[Lemma 2]{lee2025minimax} only applies to a fixed value of $\delta$ since their bonus term depends on it, whereas our result holds simultaneously for all $\delta \in (0, 1]$, which is an important distinction for deriving the distributional regret bounds.
In addition, our analysis improves the scaling of the underestimation from $\Vmax^2 / c_{1, k}$ to $\Qrange / c_{1, k}$, showing that a smaller variance of an MDP automatically leads to less underestimation.
Due to our unified framework, it is immediately derived that it may also scale with $\sigma^2 H$ when the reward noise is $\sigma^2$-sub-Gaussian.
\\
For Theorem~\ref{thm:gap-independent bound}, the remaining steps of bounding $V_1^k(s_1^k) - V_1^{\pi^k}(s_1^k)$ follow standard techniques.
We bound the regret of one episode by $V_1^*(s_1^k) - V_1^{\pi^k}(s_1^k) \lesssim \frac{\Qrange \ell(\delta)}{c_{1, k}} + U_1^k(s_1^k)$, where $U_h^k(s) \approx \Expec_{\pi(\cdot\mid s_h = s)}[ \sum_{j=h}^H b^k(s_j, a_j) + \frac{\Vmax S \ell(\delta)}{N^k(s_j, a_j)}]$.
We also derive a useful tool for bounding the expected visit counts by the visit counts of the sampled trajectory, which is a generalization of Lemma 15 in \citet{lee2025minimax}.

\begin{lemma}
\label{lma:sum of J informal}
    Let $\{X_h^k\}_{k, h}$ be a sequence of non-negative random variables adapted to a filtration $\{\Fcal_h^k\}_{k, h}$.
    Let $c > 0$ be a constant.
    Recursively define $\{J_h^k\}_{h, k}$ as $J_{H+1}^k := 0$ and $J_h^k := \left( X_h^k + \EE [ J_{h+1}^k \mid \Fcal_h^k] \right) \land c $ for all $k \in \NN$ and $h \in [H]$.
     Then, for any $\delta \in (0, 1]$, it holds that $ \sum_{k=1}^K J_1^k \le 2 \sum_{k=1}^K \sum_{h=1}^H X_h^k + 6c \log \frac{2}{\delta}$ for all $K \in \NN$ with probability at least $1 - \delta$.
\end{lemma}
We bound $\sum_k U_1^k(s_1^k)$ by taking $X_h^k \approx b^k(s_h^k, a_h^k) + \frac{\Vmax S \ell(\delta)}{N^k(s_h^k, a_h^k)}$, which yields Theorem~\ref{thm:gap-independent bound}.

The remaining steps for the proof of Theorem~\ref{thm:gap-dependent bound} are more subtle.
We incorporate the clipping framework of~\citet{simchowitz2019non,dann2021beyond} that shows that the bonus term $b^k(s_h^k, a_h^k)$ may not be added to the regret bound once $b^k(s_h^k, a_h^k) \lesssim \frac{1}{2} \gap_h(s_h^k, a_h^k)$ holds.
However, this framework is designed specifically for optimistic algorithms, so the quasi-optimism term requires additional attention.
The main idea is to show that the quasi-optimism term $\frac{\Qrange \ell(\delta)}{c_{1, k}}$ is also negligible beyond a certain threshold $k \ge \kappa'$, where $\kappa'$ should be as small as possible.
\\
One method is to define the minimum non-zero regret $\Reg_{\min} := \min_{s \in \Scal, \pi \in \Pi, V_1^*(s) - V_1^{\pi}(s) > 0} V_1^*(s) - V_1^{\pi}(s)$, then let $\kappa'$ be the time step where $\frac{\Qrange \ell(\delta)}{c_{1, k}} \lesssim \frac{1}{2}\Reg_{\min}$, implying $V_1^*(s_1^k) - V_1^{\pi^k}(s_1^k) \lesssim 2 U_1^k(s_1^k)$ for $k > \kappa'$.
While $\Reg_{\min}$ coincides with $\gap_{\min}$ in the bandit setting, it may be arbitrarily smaller than $\gap_{\min}$ in an MDP when the policy only makes a sub-optimal action at a state that is reachable with an arbitrarily small probability, hence $\kappa'$ could be arbitrarily large.
To address this issue, we derive a bound that probabilistically adds the quasi-optimism term only when the agent makes a sub-optimal action.
Since the sub-optimality of that action would be at least $\gap_{\min}$, we can subsume the quasi-optimism term once it goes below $\frac{1}{2}\gap_{\min}$.
Specifically, defining $B$ as the time step of the first sub-optimal action, we show that $V_1^*(s_1) - V_1^{\pi^k}(s_1) = \Expec_{\pi^k}[V_B^*(s_B) - V_B^{\pi^k}(s_B)] \lesssim \Expec_{\pi^k}[ 2 U_B^{\pi^k}(s_B)]$ when $\frac{\Qrange \ell(\delta)}{c_{1, k}} \lesssim \frac{1}{2}\gap_{\min}$.
This gives rise to the definition of $\kapgap(\delta, \gap_{\min})$ in Theorem~\ref{thm:gap-dependent bound}.
By carefully setting $X_h^k$ in Lemma~\ref{lma:sum of J informal} to properly capture this property in a way that depends on the history within the same episode, we derive Theorem~\ref{thm:gap-dependent bound}.

\paragraph{Proof Sketch of Theorem~\ref{thm:bandit 1}.}
We additionally provide a sketch of the technique that shaves off logarithmic factors in Theorem~\ref{thm:bandit 1}.
For each $a \in \Acal$, denote the last time step that the action $a$ is taken by $t_a$.
Then, at the $t_a$-th time step, we have 
\begin{align*}
    r(a^*) & \le \frac{\sigma^2 \log \frac{4}{\delta}}{2 c_1} + \hat{r}^{t_a}(a^*) + \frac{c_1}{N^{t_a}(a^*)}
    \\
    & \le \frac{\sigma^2 \log \frac{4}{\delta}}{2 c_1} + \hat{r}^{t_a}(a) + \frac{c_1}{N^{t_a}(a)}
    \, ,
\end{align*}
where the first inequality is due to quasi-optimism and the second inequality is from the action-selection rule.
There is no logarithmic factor of $A$ or $T$ since quasi-optimism only requires the concentration of the optimal action's noise and the inequality is time-uniform.
We then have $\gap(a) \le \frac{\sigma^2 \log \frac{4}{\delta}}{2c_1} + \hat{r}^{t_a}(a) - r(a) + \frac{c_1}{N^{t_a}(a)}$.
If we individually bound $\hat{r}^t(a) - r(a)$ for each $a\in \Acal$, we must take the union bound over $a \in \Acal$, which incurs a $\log A$ factor.
We observe that the union bound is overly pessimistic here, as it assumes that violations of the confidence bounds are mutually exclusive for all arms.
We avoid this by bounding the sum of noises over $a \in \Acal \setminus \{a^*\}$, requiring only two concentration inequalities instead of $A$.
Taking the sum of $\gap(a) N^{t_a}(a)$ over $a \in \Acal \setminus\{a^*\}$, we have
\begin{align*}
    \sum_{a \in \Acal\setminus\{a^*\}} \gap(a) N^{t_a}(a)
     \le \frac{\sigma^2 T \log \frac{4}{\delta}}{2c_1} + c_1 A + \sum_{a \in \Acal \setminus\{a^*\}} N^{t_a}(a) (\hat{r}^{t_a}(a) - r(a))
     \, .
\end{align*}
The last sum represents the sum of reward noises from all time steps with sub-optimal action selections, excluding the last noise of each action, which is sampled at time step $t_a$.
We show that this sum is bounded by $\Ocal(\sqrt{AT \log \frac{1}{\delta}})$ using concentration results for the total noise from sub-optimal actions and the noise specifically at the final time steps.
Although the latter requires taking the union bound over all possible such time steps, we show that it does not affect the leading term.
Combining the terms yields the regret bound of Theorem~\ref{thm:bandit 1}.

\section{Conclusion}

In this paper, we study the distributional properties of regret under the algorithm $\AlgName$.
We provide a distributional regret bound that holds for arbitrary failure probability.
We provide very generic theorems for arbitrary input, which allow us to study the trade-off controlled by the input parameters.
We also propose a framework that unifies the bandit and RL settings.
While we achieve optimal results in the bandit setting, the lack of lower bound results specific to the RL settings leaves an open question of whether our guarantee is optimal.
Extending the work to function approximation would be interesting future work.

\newpage
\acks{
This work was supported by the National Research Foundation of Korea~(NRF) grant and the Institute of Information \& communications Technology Planning \& Evaluation~(IITP) grant both funded by the Korea government~(MSIT) (No. RS-2022-NR071853, RS-2023-00222663, RS-2025-25463302, RS-2026-25507282).
}

\bibliography{references}

\newpage

\appendix

\section{Additional Definitions and Notations}
\label{appx:notations and definitions}

In this section, we define additional notations for the analysis.

We define logarithmic factors $\ell_1(i, \delta) := \frac{32 HSA i^2}{\delta}$ and $\ell_2(k, \delta) = \frac{32 HSA (\log e^2kH)^2}{\delta}$.
Using these logarithmic factors, we provide a more precise definitions of $\kappa(\delta)$ and $\kappa_{\gap}(\gap, \delta)$.
\begin{align*}
    & \kap(\delta) := \max \Big\{ k \in \NN : c_{1, k} <  (2 \sqrt{13 \Qdiff} \lor 2 \Valpha )\ell_1(1, \delta) \text{ or } 
    \textstyle
    c_{2, k} < \left(\sigsub + \frac{6\Qdiff \ell_1(\iota_k, \delta)}{c_{1, k}} \right) \sqrt{ \ell_2(k, \delta)} \Big\}
    \, ,\\
    & \kapgap(\gap, \delta) := \max \left\{ k \in \NN : c_{1, k} < \frac{36 \Qrange \ell_1(\iota_k, \delta)}{\gap} \right\}
    \, .
\end{align*}
\\
We let $N_h^k(s, a) := N^k(s, a) + \sum_{j=1}^h \ind\{ (s_j^k, a_j^k) = (s, a)\}$, which is the visit count of $(s, a)$ up to the $h$-th time step of the $k$-th episode.
Let $\eta^k \in [H+1]$ be a stopping time with respect to $\{ \Fcal_h^k\}_{h=1}^{H+1}$ defined as the first time step $h$ within the $k$-th episode such that $N_h^k(s_h^k, a_h^k) \ge 2 N^k(s_h^k, a_h^k)$, where $\eta^k = H+1$ if such a time step does not exist.
This definition is from \citet{lee2025minimax} and is useful for handling the possibility that one state-action pair may be visited multiple times within an episode.

\subsection{Auxiliary Sequence for Analysis}

In this section, we define auxiliary sequences $\{\lambda_i\}_{i=1}^\infty$ and $\{ \iota_k\}_{k=1}^\infty$ for the proof that depends on $\{c_{1, k}\}_{k=1}^\infty$ and $\delta$.
For most parameter choices we consider, we have $\iota_k \approx \log k$ and $\lambda_{\iota_k} \approx \frac{\ell_1(\iota_k, \delta)}{c_{1, k}}$.
If $c_{1, k}$ is a fixed constant, then we have $\iota_k = 1$.

\begin{lemma}
\label{lma:choice of lambda}
    Suppose $\delta \in (0, 1]$ and a positive sequence $\{c_{1, k}\}_{k=1}^\infty$ is given.
    Then, there exist sequences $\{\lambda_i\}_{i=1}^\infty$ and $\{\iota_k\}_{k=1}^\infty$ such that (i) $0 < \lambda_i \le \frac{1}{\Valpha}$ for all $i \in \NN$, and (ii) for all $k \in \NN$ with $c_{1, k} \ge (2 \sqrt{13 \Qdiff} \lor 2 \Valpha ) \ell_1(1, \delta)$, we have
    \begin{align}
        \frac{\ell_1(\iota_k, \delta)}{\lambda_{\iota_k}} + 13 \lambda_{\iota_k} \Qdiff \ell_1(1, \delta) \le c_{1, k} \le \frac{4 \ell_1(\iota_k, \delta)}{\lambda_{\iota_k}}
        \label{eq:condition on lambda}
        \, .
    \end{align}
    Furthermore, if $c_{1, k}$ is non-decreasing, then we have $\iota_k \le k \land (3 + \log_2 (c_{1, k} / c_{1, 1}))$.
    For $k \in \NN$ with $c_{1, k} < (2 \sqrt{13 \Qdiff} \lor 2 \Valpha ) \ell_1(1, \delta)$, we define $\iota_k := 1$.
\end{lemma}

\begin{proof}
    Our goal is to construct $\{\lambda_i\}_{i=1}^\infty$ such that the union of ranges $\cup_i [ \frac{\ell_1(i, \delta)}{\lambda_i} + 13 \lambda_i\Qdiff \ell(1, \delta), \frac{4\ell_1(i, \delta)}{\lambda_i} ]$ covers all $c_{1, k}$ with $c_{1, k} \ge (2 \sqrt{13 \Qdiff} \lor 2 \Valpha) \ell_1(1, \delta)$.
    We iteratively choose $\lambda_i$ to be the value $\lambda$ that satisfies $\frac{\ell_1(i, \delta)}{\lambda} + 13 \lambda \Qdiff \ell_1(1, \delta) = c_{1, k}$ for the least $c_{1, k}$ that is not covered by the previous $\lambda_{i'}$, which is equivalent to the least $c_{1, k} $ with $c_{1, k} \ge \frac{4 \ell_1(i-1, \delta)}{\lambda_{i-1}}$ for $i \ge 2$.
    The existence of such $\lambda$ is guaranteed by \lemmaref{lma:choosing lambda}.
    We also have $\lambda \le \sqrt{\frac{\ell_1(i, \delta)}{13 \Qdiff \ell_1(1, \delta)}}$ by \lemmaref{lma:choosing lambda}, which implies $\frac{\ell_1(i, \delta)}{\lambda} + 13 \lambda \Qdiff \ell_1(1, \delta) \le \frac{2 \ell_1(i, \delta)}{\lambda}$.
    Then, the right ends of the intervals $\frac{4 \ell_1(i, \delta)}{\lambda_i}$ increase at least exponentially by the following reasoning:
    \begin{align*}
        \frac{4 \ell_1(i, \delta)}{\lambda_i}
        & \ge 2 \cdot \frac{2 \ell_1(i, \delta)}{\lambda_i}
        \\
        & \ge 2 \left( \frac{\ell_1(i, \delta)}{\lambda_i} + 13 \lambda_i \Qdiff \ell_1(1, \delta) \right)
        \\
        & \ge 2 \cdot \frac{4 \ell_1(i-1, \delta)}{\lambda_{i-1}}
        \, ,
    \end{align*}
    where the last inequality holds by the choice of $\lambda_i$.
    Hence, for any $c_{1, k}$, there exists $i \in \NN$ such that $c_{1, k} \le \frac{4\ell_1(i, \delta)}{\lambda_i}$, and by the construction of $\lambda_i$, the least such $i$ also satisfies $\frac{\ell_1(i, \delta)}{\lambda_i} + 13 \lambda_i \Qdiff \ell_1(1, \delta) \le c_{1, k}$.
    By setting $\iota_k = i$, Eq.~\eqref{eq:condition on lambda} is satisfied.
    \\
    Now, suppose $c_{1, k}$ is non-decreasing.
    Since each range covered by $\lambda_i$ covers at least one $c_{1, k}$, we obtain $\iota_k \le k$.
    Also, the fact that $\frac{4\ell_1(i, \delta)}{\lambda_i}$ increases exponentially implies $\iota_k \le 3 + \log_2(c_k / c_1)$.
    Formally, let $\kap(\delta) + 1$ be the least $k \in \NN$ that satisfies $c_{1, k} \ge (2 \sqrt{13 \Qdiff} \lor 2 \Valpha) \ell_1(1, \delta)$.
    Then, for $k > \kap(\delta)$, we have
    \begin{align*}
        c_{1, k} \ge \frac{\ell_1(\iota_k, \delta)}{\lambda_{\iota_k}} 
        & \ge 2^{\iota_k - 1} \cdot \frac{\ell_1(1, \delta)}{\lambda_1}
        \\
        & = 2^{\iota_k - 3}\cdot \frac{ 4\ell_1(1, \delta)}{\lambda_1}
        \\
        & \ge 2^{\iota_k -3} c_{1, \kap(\delta)+1}
        \, .
    \end{align*}
    The inequality above implies that $\iota_k \le 3 + \log_2( c_{1, k} / c_{1, \kap(\delta)+1}) \le 3 + \log_2(c_{1, k} / c_{1, 1})$.
    For $k \le \kappa(\delta)$, we have $\iota_k = 1 \le 3 + \log_2(c_{1, k} / c_{1, 1})$.
\end{proof}

\begin{lemma}
\label{lma:choosing lambda}
    For given $i \in \NN$, $\delta \in (0, 1]$, and $c \ge (2\sqrt{13 \Qdiff} \lor 2 \Valpha) \ell_1(1, \delta)$, there exists $\lambda$ such that $0 < \lambda \le \frac{1}{\Valpha} \land \sqrt{\frac{\ell_1(i, \delta)}{13 \Qdiff \ell_1(1, \delta)}}$ and $\frac{\ell_1(i, \delta)}{\lambda} + 13 \lambda \Qdiff\ell_1(1, \delta) = c$.
\end{lemma}

\begin{proof}
    The proof is simply using that $f(\lambda):= \frac{\ell_1(i, \delta)}{\lambda} + 13 \lambda \Qdiff\ell_1(i, \delta)$ is continuous.
    \\
    Take $\gamma := \frac{1}{\Valpha} \land \sqrt{\frac{\ell_1(i, \delta)}{13 \Qdiff \ell_1(1, \delta)}}$.
    From $\gamma \le \sqrt{\frac{\ell_1(i, \delta)}{13 \Qdiff \ell_1(1, \delta)}}$, we have $13 \gamma \Qdiff \ell_1(1, \delta) \le \frac{\ell_1(i, \delta)}{\gamma}$, and hence $f(\gamma) \le \frac{2 \ell_1(i, \delta)}{\gamma}$.
    Then, we have.
    \begin{align*}
        f(\gamma) \le \frac{2 \ell_1(i, \delta)}{\gamma} \le \left(2 \Valpha \lor 2 \sqrt{\frac{13 \Qdiff \ell_1(1, \delta)}{\ell_1(i, \delta)}} \right)\ell_1(i, \delta) \le (2 \Valpha \lor 2 \sqrt{13 \Qdiff}) \ell_1(i, \delta)
        \, .
    \end{align*}
    Since $\lim_{\lambda \rightarrow 0}f(\gamma) = \infty$ and $f$ is continuous, the intermediate value theorem implies that for any $c \ge  (2 \Valpha \lor 2 \sqrt{13 \Qdiff}) \ell_1(i, \delta)$, there exists $\lambda \in (0, \gamma]$ such that $f(\lambda) = c$.
\end{proof}

\subsection{Properties of Maximums and Minimums}

In this paper, we frequently do operations with maximum and minimum over values.
We state the following facts in case the readers find some steps in the proof non-trivial.
When ambiguous, we assume that the priorities of the $\land$ and $\lor$ operators are in between addition and multiplication.
For example, we have $a\land b + c = \min\{a, b\} + c$ and $ab \lor c = \max\{ ab, c\}$.

\begin{fact}
\label{fact:max and min}
    For real numbers $a, b, c, d$, and a non-negative real number $x$, the followings are true:
    \begin{enumerate}[label=(\roman*)]
        \item $(a\land b) + (c \land d) \le (a\land c) \land (b \land d)$.
        \item $(a + c) \lor (b + d) \le (a \lor b) + (c \lor d)$.
        \item $(a + c) \land (b + d) \le (a \lor b) + (c \land d) \le (a + c) \lor (b + d)$
        \item $(a + x) \land b \le a \land b + x$. 
        \item $(a + x) \lor b \le a \lor b + x$. 
        \item $a \land b - x \le (a - x) \land b$.
        \item $a \lor b - x \le (a - x) \lor b$.
        \item $(a \land b) \lor c = (a \lor c) \land (b \lor c)$.
    \end{enumerate}
\end{fact}

\section{Derivation of Table~\ref{tab:assumption example}}
\label{appx:proof of tab 1}
In this section, we show how the values in Table~\ref{tab:assumption example} can be derived.

\textbf{Bounded reward $0 \le R_h \le \Vmax$:}
We define $R_h(s_h, a_h)$ as the distribution of $R_h$ given $s_h, a_h$, which may depend on the history.
For given $(s, a) \in \Scal\times \Acal$, the random variable $R_h + V_{h+1}^*(s')$ with $R_h \sim R_h(s, a)$ and $s' \sim P(s, a)$ lies in $[0, 2 \Vmax]$.
By Lemma~\ref{lma:bounded is subexponential}, $R_h + V_{h+1}^*(s')$ is $(2 \Var(R_h + V_{h+1}^*)(s, a), 2\Vmax)$-sub-exponential.
$\Qdiff \le \Qrange$ is trivial, so it remains to prove $\Qrange \le 2 \Vmax^2$.
For any $h \in [H]$, $s \in \Scal$, and $\pi \in \Pi$, we have
\begin{align*}
    W_h^{\pi}(s)
    & = \Expec_{\pi(\cdot \mid s_h = s)} \left[\sum_{j=h}^H \Var(R_j + V_{j+1}^*)(s_j, a_j) \right]
    \\
    & = \Expec_{\pi(\cdot \mid s_h = s)} \left[\sum_{j=h}^H\left( \Expec_{\substack{R_j \sim R(s_j, a_j)\\s' \sim P(s_j, a_j)}}[(R_j + V_{j+1}^*(s'))^2] -  \Expec_{\substack{R_j \sim R(s_j, a_j)\\s' \sim P(s_j, a_j)}}[ R_j + V_{j+1}^*(s')]^2 \right) \right]
    \\
    & = \Expec_{\pi(\cdot \mid s_h = s)} \left[ \sum_{j=h}^H (R_j + V_{j+1}^*(s_{j+1}))^2 - \sum_{j=h}^H (Q_j^*(s_j, a_j))^2  \right]
    \\
    & = \Expec_{\pi(\cdot \mid s_h = s)} \left[\sum_{j=h}^{H } \left( (R_j + V_{j+1}^*(s_{j+1}))^2 - (Q_{j+1}^*(s_{j+1}, a_{j+1}))^2  \right) \right] -  \Expec_{\pi(\cdot \mid s_h = s)} \left[ (Q_h(s_h, a_h))^2 \right]
    \, ,
\end{align*}
where the third equality uses the law of total expectation and that $Q_j^*(s_j, a_j) = \Expec_{\substack{R_j \sim R(s_j, a_j)\\s' \sim P(s_j, a_j)}}[R_j + V_{j+1}^*(s')]$.
The summand can be modified as follows:
\begin{align*}
    &  \Expec_{\pi(\cdot \mid s_h = s)} \left[(R_j + V_{j+1}^*(s_{j+1}))^2 - (Q_{j+1}^*(s_{j+1}, a_{j+1}))^2 \right]
    \\
    & =  \Expec_{\pi(\cdot \mid s_h = s)} \left[( R_j + V_{j+1}^*(s_{j+1}) + Q_{j+1}^*(s_{j+1}, a_{j+1})) ( R_j + V_{j+1}^*(s_{j+1})- Q_{j+1}^*(s_{j+1}, a_{j+1})) \right]
    \\
    & \le  \Expec_{\pi(\cdot \mid s_h = s)} \left[3 \Vmax ( R_j + V_{j+1}^*(s_{j+1})- Q_{j+1}^*(s_{j+1}, a_{j+1})) \right]
    \\
    & = \Expec_{\pi(\cdot \mid s_h = s)} \left[3 \Vmax (Q_j^*(s_j, a_j) - Q_{j+1}^*(s_{j+1}, a_{j+1})) \right]
    \, ,
\end{align*}
where the inequality holds since $0 \le R_j + V_{j+1}^*(s_{j+1}) + Q_{j+1}^*(s_{j+1}, a_{j+1}) \le 3 \Vmax$ and $R_j + V_{j+1}^*(s_{j+1}) - Q_{j+1}^*(s_{j+1}, a_{j+1}) = R_j + Q_{j+1}^*(s_{j+1}, a_{j+1}^*) - Q_{j+1}^*(s_{j+1}, a_{j+1}) \ge 0$.
Taking the sum of $\Expec_{\pi(\cdot \mid s_h = s)} \left[3 \Vmax (Q_j^*(s_j, a_j) - Q_{j+1}^*(s_{j+1}, a_{j+1})) \right]$ and telescoping yields
\begin{align*}
    W_h^{\pi}(s)
    & \le \Expec_{\pi(\cdot \mid s_h = s)} \left[\sum_{j=h}^{H} 3 \Vmax (Q_j^*(s_j, a_j) - Q_{j+1}^*(s_{j+1}, a_{j+1})) \right] - \Expec_{\pi(\cdot \mid s_h = s)} \left[ (Q_h(s_h, a_h))^2 \right]
    \\
    & =\Expec_{\pi(\cdot \mid s_h = s)} \left[ 3 \Vmax Q_h^*(s_h, a_h) - (Q_h(s_h, a_h))^2 \right]
    \\
    & \le 2 \Vmax^2
    \, ,
\end{align*}
where the last inequality holds due to that $3 \Vmax x  - x^2 \le 2\Vmax^2$ for all $x \in [0, \Vmax]$.

\textbf{Sub-Gaussian reward noise:}
If $R_h$ is $\sigma^2$-sub-Gaussian and is independent of $s_{h+1}$, then for any $\lambda \in [-\frac{1}{\Vmax}, \frac{1}{\Vmax}]$, we have
\begin{align*}
    \EE_{s_h, a_h}[ \exp( \lambda (R_h + V_{h+1}^*(s_{h+1})))]
    & = \EE_{s_h, a_h} [ \exp(\lambda R_h)] \EE_{s_h, a_h}[ \exp(\lambda V_{h+1}^*(s_{h+1}))]
    \\
    & \le \exp\left( \frac{\lambda^2 \sigma^2}{2} \right) \exp \left( \lambda^2 \Var(V_{h+1}^*)(s_h, a_h)\right)
    \\
    & = \exp \left( \frac{\lambda^2}{2}\left( \sigma^2 + 2 \Var(V_{h+1}^*)(s_h, a_h) \right) \right)
    \, ,
\end{align*}
where the first equality uses independence, and the following inequality uses the sub-Gaussianity of $R_h$ and Lemma~\ref{lma:bounded is subexponential}.
Therefore, we have that $R_h + V_{h+1}^*(s_{h+1})$ is $\big(\sqrt{\sigma^2 + 2 \Var(V_{h+1}^*)(s_h, a_h)}, \Vmax \big)$-sub-exponential.
Then, the total variance proxy function is bounded as
\begin{align*}
    W_h^{\pi}(s)
    & = \Expec_{\pi(\cdot \mid s_h = s)} \left[ \sum_{j=h}^H \left( \frac{\sigma^2}{2} + \Var(V_{j+1}^*)(s_j, a_j) \right) \right]
    \\
    & = \frac{\sigma^2 (H - h +1)}{2} + \Expec_{\pi(\cdot \mid s_h = s)} \left[ \sum_{j=h}^H  \Var(V_{j+1}^*)(s_j, a_j) \right]
    \, .
\end{align*}
The expected sum is bounded similarly to the first case.
\begin{align*}
    & \Expec_{\pi(\cdot \mid s_h = s)} \left[ \sum_{j=h}^H  \Var(V_{j+1}^*)(s_j, a_j) \right]
    \\
    & = \Expec_{\pi(\cdot \mid s_h = s)} \left[ \sum_{j=h}^H  (V_{j+1}^*(s_{j+1}))^2 - \sum_{j=h}^H (PV_{j+1}^*(s_j, a_j))^2\right]
    \\
    & = \Expec_{\pi(\cdot \mid s_h = s)} \left[ \sum_{j=h}^H\left(  (V_{j}^*(s_{j}))^2 - (PV_{j+1}^*(s_j, a_j))^2 \right)\right] - (V_h^*(s_h))^2
    \\
    & = \Expec_{\pi(\cdot \mid s_h = s)} \left[ \sum_{j=h}^H( V_j^*(s_j) + PV_{j+1}^*(s_j, a_j))(V_j^*(s_j) - PV_{j+1}^*(s_j, a_j))\right] - (V_h^*(s_h))^2
    \\
    & \le \Expec_{\pi(\cdot \mid s_h = s)} \left[ \sum_{j=h}^H 2\Vmax (V_j^*(s_j) - PV_{j+1}^*(s_j, a_j))\right] - (V_h^*(s_h))^2
    \\
    & = 2\Vmax V_h^*(s_h) - (V_h^*(s_h))^2
    \\
    & \le \Vmax^2
    \, ,
\end{align*}
where we use that $V_j^*(s_j) - PV_{j+1}^*(s_j, a_j) \ge Q_j^*(s_j, a_j^*) - PV_{j+1}^*(s_j, a_j) = r(s_j, a_j^*) \ge 0 $, where $r(s, a) \ge 0$ is guaranteed by Assumption~\ref{assm:boundedness} since $r(s, a) = V_{H}^{\pi'}(s) \ge 0$ for a policy $\pi'$ with $\pi_H'(s) = a$.
We have derived that $\frac{\sigma^2(H - h + 1)}{2} \le W_h^{\pi}(s) \le \frac{\sigma^2(H - h + 1)}{2} + \Vmax^2$, hence we have $\Qrange \le \frac{\sigma^2H}{2} + \Vmax^2$ and $\Qdiff \le \Vmax^2$.

\textbf{Sub-Gaussian noise MAB:} The proof follows directly from the definitions.

\section{Proof of Bandit Theorems in Section~\ref{sec:main result bandits}}
\label{appx:bandits}

In this section, we provide a bandit version of $\AlgName$ in Algorithm~\ref{alg:eqo bandit} and prove formal versions of Theorems~\ref{thm:bandit 1} and~\ref{thm:bandit 2}.

\begin{algorithm}[ht!]
\caption{$\AlgName$ for bandits}
\label{alg:eqo bandit}
\setcounter{AlgoLine}{0}
\SetKwInOut{Input}{Input}
\Input{$\{ (c_{1, k} ,c_{2, k})\}_{k=1}^\infty$}

\For{$t = 1, 2, \ldots, A$}{
    Take $t$-th action in $\Acal$
}

\For{$t = A+1, A+2, \ldots,$}{
    $N^t(a) \gets \sum_{i=1}^{t-1} \ind\{ a_i = a \}$ for all $a \in \Acal$\;

    $\hat{r}^t(a) \gets \frac{1}{N^t(a)} \sum_{i=1}^{t-1} R_i \ind\{ a_i = a\}$ for all $a \in \Acal$\;

    $b^t(a) \gets \left( \frac{c_{1, t}}{ N^t(a)} \right) \land \left( \frac{c_{2, t} }{\sqrt{N^t(a)}} \right)$ for all $a \in \Acal$\;
    
    $a_t \gets \argmax_{a \in \Acal} \hat{r}^t(a) + b^t(a)$\;
    
    Take action $a_t$ and observe reward $R_t$\;
}
\end{algorithm}

\subsection{Proof of Theorem~\ref{thm:bandit 1}}

We present a formal version of Theorem~\ref{thm:bandit 1}.

\begin{theorem}[Restatement of Theorem~\ref{thm:bandit 1}]
\label{thm:bandit 1 formal}
    Suppose $M \in \Bcal(\sigma)$.
    Take $c_{1, t} = c_1$ and $c_{2, t} = \infty$ for all $t \in \NN$ for some constant $c_1 > 0$.
    Then, we have
    \begin{align*}
        \Reg_M^{\Alg}(T, \delta) \le \frac{\sigma^2 T \log \frac{4}{\delta}}{2 c_1} + c_1 (A - 1) + \sigma \sqrt{AT \log \frac{4}{\delta}} + \sum_{a \in \Acal} \gap(a)
        \, .
    \end{align*}
    In particular, when $c_1 = \sigma\sqrt{ \frac{T}{A}}$, the distributional regret is bounded as
    \begin{align*}
        \Reg_M^{\Alg}(T, \delta) \le \frac{1}{2} \sigma \sqrt{AT} \log \frac{4e^2}{\delta} + \sigma \sqrt{AT \log \frac{4}{\delta}} + \sum_{a \in \Acal} \gap(a)
    \end{align*}
    and the expected regret bound is bounded as
    \begin{align*}
        \Expec [ \Reg_M^{\Alg}(T) ] \le 4 \sigma \sqrt{AT} + \sum_{a \in \Acal} \gap(a)
        \, .
    \end{align*}
\end{theorem}

\begin{remark}
    In Theorem~\ref{thm:bandit 1}, we further bounded $ \sigma \sqrt{AT \log \frac{4}{\delta}} \le \frac{\sigma^2 T \log \frac{4}{\delta}}{2c_1} + \frac{c_1 A}{2}$ using the AM-GM inequality for a simpler form.
    We note that $\sigma \sqrt{AT \log \frac{4}{\delta}}$ can be tightened to $\sigma\sqrt{2T \log \frac{4}{\delta}} + \sigma \sqrt{A^2 \log \frac{eT}{A} + A \log \frac{2}{\delta}}$, although the order of the bound does not change.
\end{remark}

The high-probability event we impose is a union of three concentration inequalities, where we take a novel approach to avoid incurring a $\log A$ or $\log T$ factor.

\begin{lemma}
    Define $\hat{\Ecal}_1(\frac{\delta}{4})$ as the event that
    \begin{align*}
        r(a^*) - \hat{r}^t(a^*) \le \frac{\sigma^2 \log \frac{4}{\delta}}{2 c_1 } + \frac{c_1}{N^t(a^*)}
        \, .
    \end{align*}
    holds for all $t \in \NN$.
    Then, $\PP(\hat{\Ecal}_1(\frac{\delta}{4})) \ge 1 - \frac{\delta}{4}$.
\end{lemma}

\begin{proof}
    Let $\{R_i^*\}_{i}$ be the sequence of sampled rewards from the optimal action.
    We apply \lemmaref{lma:subexponential concentration} to $\{R_i^*\}_{i}$ with $\lambda = \frac{\log \frac{4}{\delta}}{c_1}$ and obtain that with probability at least $1 - \frac{\delta}{4}$, the following inequality holds for all $n \in \NN$:
    \begin{align*}
        \sum_{i=1}^n (r(a^*) - R_i^*) \le \frac{\sigma^2 n \log \frac{4}{\delta}}{c_1} + c_1
        \,. 
    \end{align*}
    Dividing both sides by $n$ and plugging in $n = N^t(a^*)$ completes the proof.
\end{proof}

\begin{lemma}
    For fixed $T \in \NN$, define $\hat{\Ecal}_2(\frac{\delta}{4})$ as the event that
    \begin{align*}
        \sum_{t=1}^T \ind\{ a_t \ne a^*\} (R_t - r(a_t)) \le \sigma \sqrt{2 T \log \frac{4}{\delta}}
    \end{align*}
    Then, $\PP(\hat{\Ecal}_2(\frac{\delta}{4})) \ge 1 - \frac{\delta}{4}$.
\end{lemma}
\begin{proof}
    We apply Lemma~\ref{lma:Hoeffdings inequality} to $X_t = \ind\{a_t \ne a^*\} (R_t - r(a_t))$.
\end{proof}

\begin{lemma}
    For fixed $T \in \NN$, define $\hat{\Ecal}_3(\frac{\delta}{2})$ as the event that
    \begin{align*}
        \sum_{t\in \Tcal} (r(a_t)- R_t) \le \sigma \sqrt{|\Tcal| T \log \frac{4}{\delta}}
        \, .
    \end{align*}
    holds for all subsets $\Tcal \subset [T]$. 
    Then, $\PP(\hat{\Ecal}_3(\frac{\delta}{2})) \ge 1 - \frac{\delta}{2}$.
\end{lemma}
\begin{proof}
    For a fixed set $\Tcal$, we apply Lemma~\ref{lma:Hoeffdings inequality} and obtain that
    \begin{align*}
        \sum_{t\in \Tcal} (r(a_t) - R_t) \le \sigma \sqrt{|\Tcal| \log \frac{2}{\delta}}
    \end{align*}
    holds with probability at least $1 - \frac{\delta}{2}$.
    We take the union bound over all $\Tcal$, where there are at most $2^T$ subsets.
    Then, the logarithm term is bounded as $\log \frac{2^T \cdot 2}{\delta} \le T \log \frac{4}{\delta}$, which completes the proof.
\end{proof}

\begin{proof}[Proof of Theorem~\ref{thm:bandit 1 formal}]
    For each $a \in \Acal$, let $t(a) \in [T]$ be the last time step the action $a$ is taken, so that $a_{t(a)} = a$ and $N^{T+1}(a) = N^{t(a)} + 1$.
    Assuming $t(a) > A$, the following inequality must have held for the agent to take action $a$:
    \begin{align*}
        \hat{r}^{t(a)}(a^*) + \frac{c_1}{N^{t(a)}(a^*)} \le \hat{r}^{t(a)}(a) + \frac{c_1}{N^{t(a)}(a)}
        \, .
    \end{align*}
    Under $\hat{\Ecal}_1(\frac{\delta}{4})$, we have $r(a^*) - \hat{r}^{t(a)}(a^*) \le \frac{\sigma^2 \log \frac{4}{\delta}}{2c_1} + \frac{c_1}{N^{t(a)}(a^*)}$, which implies that
    \begin{align*}
        r(a^*) \le \frac{\sigma^2 \log \frac{4}{\delta}}{2c_1} + \hat{r}^{t(a)}(a^*) + \frac{c_1}{N^{t(a)}(a^*)}  \le \frac{\sigma^2 \log \frac{4}{\delta}}{2c_1} +\hat{r}^{t(a)}(a)  + \frac{c_1}{N^{t(a)}(a)}
        \, .
    \end{align*}
    Subtracting both sides by $r(a)$, we obtain that
    \begin{align*}
        \gap(a) \le \frac{\sigma^2 \log \frac{4}{\delta}}{2c_1} + \frac{c_1}{N^{t(a)}(a)} + (\hat{r}^{t(a)} - r)(a)
        \, .
    \end{align*}
    We multiply both sides by $N^{t(a)}(a)$.
    \begin{align*}
        \gap(a) N^{t(a)}(a) & \le \frac{\sigma^2 N^{t(a)}{(a)}\log \frac{4}{\delta}}{2c_1} + c_1 + (\hat{r}^{t(a)} - r)(a) N^{t(a)}(a)
        \\
        & =\frac{\sigma^2 N^{t(a)}{(a)}\log \frac{3}{\delta}}{2c_1} + c_1 + \sum_{t=1}^{t(a) - 1} \ind\{ a_t = a\} (R_t - r(a))
        \, .
    \end{align*}
    Note that this inequality also holds when $t(a) \le A$, since it implies that the action $a$ is taken only once and hence $N^{t(a)}(a) = 0$.
    Taking the sum over $a \in \Acal \setminus \{ a^*\}$, we obtain that
    \begin{align*}
        \sum_{a \in \Acal \setminus\{a^*\}} \gap(a) N^{t(a)}(a) \le \frac{\sigma^2 T \log \frac{4}{\delta}}{2c_1} + c_1 (A - 1) + \sum_{t=1}^T \ind\{a_t \ne a^*, t \ne t(a_t) \} (R_t - r(a_t))
        \, ,
    \end{align*}
    where we use that $\sum_{a \in \Acal \setminus\{a^*\}} N^{t(a)}(a) \le T$.
    We bound the last sum as follows:
    \begin{align*}
        & \sum_{t=1}^T \ind\{a_t \ne a^*, t \ne t(a_t) \} (R_t - r(a_t))
        \\
        & = \sum_{t=1}^T \ind\{a_t \ne a^*\} (R_t - r(a_t))
         - \sum_{t \in \{ t(a) \mid a \in \Acal \setminus \{a^*\}\}} (R_t - r(a_t))
         \\
         & \le \sigma\sqrt{2T \log \frac{4}{\delta}} + \sigma \sqrt{(A-1)T \log \frac{4}{\delta}}
         \\
         & \le 2 \sigma \sqrt{AT \log \frac{4}{\delta}}
         \, ,
    \end{align*}
    where the first inequality holds under $\hat{\Ecal}_2(\frac{\delta}{4})$ and $\hat{\Ecal}_3(\frac{\delta}{2})$.
    Therefore, under an event whose probability is at least $1 - \delta$, we have
    \begin{align*}
        \Reg_M^{\Alg}(T) & = \sum_{a \in \Acal\setminus\{a^*\}} \gap(a) N^{T+1}(a)
        \\
        & = \sum_{a \in \Acal\setminus\{a^*\}} \gap(a) (N^{t(a)}(a) + 1)
        \\
        & \le \frac{\sigma^2 T \log \frac{4}{\delta}}{2c_1} + c_1 (A - 1) + 2 \sigma\sqrt{AT \log \frac{4}{\delta}} + \sum_{a \in \Acal} \gap(a)
        \, .
    \end{align*}
\end{proof}

\subsection{Proof of Theorem~\ref{thm:bandit 2}}

\label{appx:proof of bandit 2}

Recall that we defined a time-uniform UCB as a function $u(t, \delta) : \RR_{\ge 0} \times (0, 1] \rightarrow \RR_{\ge 0}$ that satisfies
\begin{align*}
    \PP \left( \exists t  > A : | \hat{r}^t(a) - r(a) | > \frac{u(t, \delta)}{\sqrt{N^t(a)}} \right) \le \delta
    \, .
\end{align*}
A sensible concentration inequality would yield a form of $u(t, \delta) = \sigma \sqrt{f(t) \log \frac{1}{\delta} + g(t)}$ for some real-valued functions $f(t)$ and $g(t)$, where we omit dependence on $A$.
We will say $u(t, \delta)$ is (asymptotically) \emph{tight} if we have $\lim_{t \rightarrow \infty} f(t) = 2$ and $\lim_{t \rightarrow \infty} \frac{g(t)}{\log t} = 0$, which is tight in the sense of meeting the lower bound of \citet{lai1985asymptotically} by the following theorem.
An example of such a concentration inequality is provided in Lemma~\ref{lma:example of tight time-uniform UCB}.
\begin{lemma}
\label{lma:example of tight time-uniform UCB}
    The following function $u(t, \delta)$ is a tight time-uniform UCB:
    \begin{align*}
        u(t, \delta) = \sigma \sqrt{2 \left(1 + \frac{1}{\log (1 + t)}\right) \left(\log \frac{1}{\delta} + \log A + 6 \log \log (1 + t)  + 8 \right)}
        \, .
    \end{align*}
\end{lemma}
\begin{proof}
    Apply Lemma~\ref{lma:all eta concentration} with $\eta = \frac{1}{\log (1 + t)}$.
    Then, for fixed $ a\in \Acal$, with probability at least $1 - \delta$, it holds that for all $t \in \NN$,
    \begin{align*}
        & \left| \hat{r}^t(a) - r(a) \right|
        \\
        & \le \sigma \sqrt{2 \left( 1 + \frac{1}{\log (1 + t)} \right) \left( \log \frac{1}{\delta} + \log 4\left(2 + \log \log (1+ t) \right)^2 \left(1 + 2e(\log (1+t))(\log N^t(a)) \right)^2\right)}
        \, .
    \end{align*}
    We bound $\log N^t(a) \le \log(1+ t)$.
    We further modify the logarithmic term as follows:
    \begin{align*}
        & \log 4\left(2 + \log \log (1+ t) \right)^2 \left(1 + 2e(\log (1+t))(\log N^t(a)) \right)^2
        \\
        & \le \log 4 + 2\log (2 + \log \log ( 1 + t)) + 2 \log (1 + 2e(\log (1+t))^2)
        \, .
    \end{align*}
    For the second term, we use $\log(1 + x) \le x$ and bound
    \begin{align*}
        2\log (2 + \log \log ( 1 + t))
        & \le 2(1 + \log \log (1 + t))
        \,. 
    \end{align*}
    For the third term, we have
    \begin{align*}
        2 \log (1 + 2e ( \log (1+t))^2)
        & \le 2 \log \left( \left(\frac{1}{(\log 2)^2} +  2e\right)  ( \log (1+t))^2 \right)
        \\
        & \le 4 \log \log (1 + t) + 2 \log 8
        \, . 
    \end{align*}
    Taking the sum of three terms, the logarithmic term is upper bounded by $6 \log \log t + 8$.
    By taking the union bound over $a \in \Acal$, we incur an additional $\log A$ factor and fully recover $u(t, \delta)$.
    \\
    We have shown that $u(t, \delta)$ is a time-uniform UCB, and the fact that it is tight follows from direct computation.
\end{proof}

\begin{theorem}[Formal extension of Theorem~\ref{thm:bandit 2}]
\label{thm:bandit 2 formal}
    Suppose $M \in \Bcal(\sigma)$.
    Suppose $u(t, \delta)$ is a time-uniform UCB.
    For any decreasing sequence $\delta_t$, let $c_{1, t} = \infty$ and $c_{2, t} = u(t, \delta_t)$ for all $t \in \NN$
    Let $\tau_2(\delta) := \argmax_{t \in \NN} \delta_t > \delta$.
    Then, we have
    \begin{align*}
        \Reg_M^{\Alg}(T, \delta) & \le \tau_2(\delta) \land T + \sum_{\substack{a \in \Acal\\ \gap(a) \ne 0}} \gap(a)   + \sum_{\substack{a \in \Acal\\ \gap(a) \ne 0}} \frac{\left(c_{2, T} +  u(T, \delta) \right)^2}{\gap(a)}  
        \, .
    \end{align*}
    If $u(t, \delta)$ admits a form of $\sigma\sqrt{f(t) \log \frac{1}{\delta} + g(t)}$ for some functions $f$ and $g$, then the expected regret is bounded as
    \begin{align*}
        \Expec[\Reg_M^{\Alg}(T)] & \le \sum_{t=1}^T \delta_t + \sum_{a \in \Acal} \gap(a) + \sum_{\substack{a \in \Acal\\ \gap(a) \ne 0}} \frac{\left(c_{2, T} +  u(T, e^{-1}) \right)^2}{\gap(a)}
        \, .
    \end{align*}
    Furthermore, if $\delta_t = o(t^{-1})$ and $\delta_t =\omega(t^{-\alpha})$ for all $\alpha > 1$, e.g., $\delta_t = (t\log t)^{-1}$, and if $u(t, \delta)$ is tight, then we have
    \begin{align*}
        \limsup_{T \rightarrow \infty} \frac{\Expec[\Reg_M^{\Alg}(T)]}{\log T} \le \sum_{\substack{a \in \Acal\\ \gap(a) \ne 0}}\frac{2 \sigma^2}{\gap(a)}
        \, .
    \end{align*}
\end{theorem}

\begin{proof}
    Fix $\delta \in (0, 1]$ and assume the event that $| \hat{r}^t(a) - r(a)| \le \frac{u(t, \delta)}{\sqrt{N^t(a)}}$ holds for all $ a\in \Acal$ and $ t \in \NN$, whose probability is at least $1 - \delta$ by the definition of time-uniform UCB.
    Suppose $t > \tau_2(\delta)$, which implies $\delta \ge \delta_t$ and consequently $c_{2, t} \ge u(t, \delta)$.
    Then, the UCB estimate of the optimal arm is at least its true mean as
    \begin{align*}
        r(a^*)
        \le 
        \hat{r}^t(a^*) + \frac{u(t, \delta)}{\sqrt{N^t(a^*)}}
        \le 
        \hat{r}^t(a^*) + \frac{c_{2, t}}{\sqrt{N^t(a^*)} } 
        \, .
    \end{align*}
    Hence, for a sub-optimal arm $a$ to be taken, it must satisfy
    \begin{align*}
        r(a^*)
        \le \hat{r}^t(a^*) + \frac{c_{2, t}}{\sqrt{N^t(a^*)} }
        \le 
        \hat{r}^t(a) + \frac{c_{2, t}}{\sqrt{N^t(a)}}
        \, .
    \end{align*}
    Subtracting $r(a)$ from both sides, we have
    \begin{align*}
        \gap(a) & = r(a^*) - r(a)
        \\
        & \le \hat{r}^t(a) - r(a) + \frac{c_{2, t}}{\sqrt{N^t(a)}}
        \\
        & \le \frac{u(t, \delta)}{\sqrt{N^t(a)}}  + \frac{c_{2, t}}{\sqrt{N^t(a)}}
        \, .
    \end{align*}
    This inequality implies that if an action $a \in \Acal$ is taken at a time step $t > \tau_2(\delta)$, then it satisfies $N^t(a) \le \frac{(u(t, \delta) + c_{2, t})^2}{\gap(a)^2}$.
    Fix $T$ and denote the last time step an action $a$ is taken by $t(a)$.
    Then, we have either $t(a) \le \tau_2(\delta)$ or $N^{t(a)} \le \frac{(u(T, \delta) + c_{2, t})^2}{\gap(a)^2}$.
    Noting that $N^{t(a)}(a) + 1= N^{T+1}(a)$ by the definition of $t(a)$, we have that
    \begin{align*}
        \Reg_M^{\Alg}(T, \delta)
        & = \sum_{\substack{a\in \Acal\\ \gap(a) \ne 0}} \gap(a) N^{T+1}(a)
        \\
        & = \sum_{\substack{a\in \Acal\\ \gap(a) \ne 0 \\ t(a) \le \tau_2(\delta) }} \gap(a) N^{T+1}(a) + \sum_{\substack{a\in \Acal\\ \gap(a) \ne 0\\t(a) > \tau_2(\delta)} } \gap(a) N^{T+1}(a)
        \\
        & \le \tau_2(\delta) \land T + \sum_{\substack{a\in \Acal\\ \gap(a) \ne 0\\t(a) > \tau_2(\delta)} } \left( \gap(a) + \frac{(c_{2, T} + u(T, \delta))^2}{\gap(a)} \right)
        \\
        & \le \tau_2(\delta) \land T + \sum_{a \in \Acal} \gap(a) + \sum_{\substack{a\in \Acal\\ \gap(a) \ne 0}} \frac{(c_{2, T} + u(T, \delta))^2}{\gap(a)}
        \, .
    \end{align*}
    This proves the first part of the lemma.
    \\
    The second part of the lemma is simply taking the integral of the first bound with respect to $\delta$.
    First, note that $(\tau_2(\delta) \land T )- (\tau_2(\delta) \land (T -1) )= \ind\{ \tau_2(\delta) \ge T \} = \ind\{ \delta_T > \delta\}$.
    Integrating this indicator function, we have $\int_0^1 \ind\{ \delta_T > \delta\} \, d\delta = \delta_T$, and hence we obtain that
    \begin{align*}
        \int_0^1 \tau_2(\delta) \land T = \sum_{t=1}^T \delta_t
        \, .
    \end{align*}
    To integrate $(c_{2, T} + u(T, \delta))^2$, we expand it as follows:
    \begin{align*}
        \int_0^1 (c_{2, T} + u(T, \delta))^2 \, d \delta
        & = \int_0^1 c_{2, T}^2 + 2 c_{2, T} u(T, \delta) + (u(T, \delta))^2 \, d \delta
        \, .
    \end{align*}
    Using that $u(T, \delta) = \sigma\sqrt{f(T) \log \frac{1}{\delta} + g(T)}$, we have
    \begin{align*}
        \int_0^1 (u(T, \delta))^2 \, d \delta
        & = \int_0^1 \sigma^2 \left( f(T) \log \frac{1}{\delta} + g(T)\right) \, d \delta
        \\
        & = \sigma^2 \left( f(T) + g(T) \right)
        \\
        & = (u(T, e^{-1}))^2
        \, ,
    \end{align*}
    where we use that $\int_0^1 \log \frac{1}{\delta} \, d \delta = 1$.
    We also have
    \begin{align*}
        \int_0^1 u(T, \delta)\, d \delta
        & \le \sqrt{ \int_0^1 (u(T, \delta))^2 \, d \delta }
        \\
        & \le u(T, e^{-1})
        \, ,
    \end{align*}
    where we use the Cauchy-Schwarz inequality for the first inequality.
    Therefore, we obtain that
    \begin{align*}
        \int_0^1 (c_{2, T} + u(T, \delta))^2 \, d \delta
        & \le c_{2, T}^2 + 2 c_{2, T} u(T, e^{-1}) + (u(T, e^{-1}))^2
        \\
        & = (c_{2, T} + u(T, e^{-1}))^2
        \, .
    \end{align*}
    This proves the second part of the lemma.
    \\
    For the last part, $\delta_t = o(t^{-1})$ implies that $\lim_{T \rightarrow \infty} \frac{1}{\log T} \sum_{t=1}^T \delta_t = 0$.
    By the condition of tightness, we have $\lim_{T \rightarrow \infty} \frac{u(T, e^{-1})}{\log T} = 0$.
    The condition $\delta_t = \omega(t^{-\alpha})$ for all $\alpha > 1$ implies that $\lim_{T\rightarrow \infty} \frac{\log \frac{1}{\delta_T}}{\log T} = 1 $, which implies that
    \begin{align*}
        \limsup_{T\rightarrow \infty} \frac{c_{2, T}^2}{\log T} & = \limsup_{T\rightarrow \infty} \frac{\sigma^2(f(T) \log \frac{1}{\delta_T} + g(T) )}{\log T}
        \\
        & = 2 \sigma^2
        \, .
    \end{align*}
    We note that the limits of $c_{2, T}$ and $u(T, \delta)$ imply that $\lim_{T \rightarrow \infty} \frac{c_{2, T}u(T, \delta)}{\log T} = 2\sigma^2 \cdot 0 = 0$.
    Therefore, we conclude that    
    \begin{align*}
        & \limsup_{T \rightarrow \infty} \frac{\Expec[\Reg_M^{\Alg}(T)]}{\log T}
        \\
        & \le \limsup_{T \rightarrow \infty}\frac{1}{\log T} \Bigg(\sum_{t=1}^T \delta_t + \sum_{a \in \Acal} \gap(a) + \sum_{\substack{a \in \Acal\\ \gap(a) \ne 0}} \frac{1}{\gap(a)}  \left(c_{2, T} +  u(T, e^{-1}) \right)^2 \Bigg)
        \\
        & = \limsup_{T \rightarrow \infty} \frac{1}{\log T}\sum_{\substack{a \in \Acal\\ \gap(a) \ne 0}} \frac{c_{2, T}^2}{\gap(a)}
        \\
        & = \sum_{\substack{a \in \Acal\\ \gap(a) \ne 0}}\frac{2 \sigma^2}{\gap(a)}
        \, .
    \end{align*}
\end{proof}

\section{Proof of RL Theorems in Section~\ref{sec:main result RL}}
\label{appx:rl}
In this section, we prove Theorems~\ref{thm:gap-independent bound} and~\ref{thm:gap-dependent bound}.
In Appendix~\ref{appx:high probability events}, we provide and prove high-probability events that constitute $\Ecal(\delta)$, which is the event that the distributional regret bounds hold.
In Appendix~\ref{appx:general properties of EQO}, we provide general properties of $\AlgName$, including the quasi-optimism lemma (\lemmaref{lma:quasi optimism}, the formal version of \lemmaref{lma:quasi optimism informal}) and its proof.
In Appendix~\ref{appx:proof of gap-independent bound}, we state and prove Theorem~\ref{thm:gap-independent bound formal}, which is the formal version of Theorem~\ref{thm:gap-independent bound}.
In Appendix~\ref{appx:proof of gap-dependent bound}, we state and prove Theorem~\ref{thm:instance dependent bound formal}, which is the formal version of Theorem~\ref{thm:gap-dependent bound}.

\subsection{High-probability Events}

\label{appx:high probability events}

In this section, we provide high-probability events under which the analysis is conducted.
In \appendixref{appx:definition of high-prob events}, we define the events and provide lemmas that state these events occur with high probabilities.
The proofs of those lemmas are provided in Appendices~\ref{appx:sub-exponential concentration event}--\ref{appx:proof of sum of J with eta}.

\subsubsection{Definition of Events}
\label{appx:definition of high-prob events}

\begin{lemma}
\label{lma:sub-exponential concentration event}
    Define $\Ecal_1(\frac{\delta}{4})$ as the event that
    \begin{align*}
        & \left| \hat{r}^k(s, a) + \hat{P}^k V_{h+1}^*(s, a) - Q_{h+1}^*(s, a) \right| \land \Vmax
        \\
        & \qquad \le \frac{\lambda_{\iota_k} \sigexp^2(h, s, a)}{2} + \frac{\ell_1(\iota_k, \delta)}{\lambda_{\iota_k} N^k(s, a)} \land \sigsub\sqrt{ \frac{\ell_2(k, \delta)}{N^k(s, a)} }
    \end{align*}
    holds for all $h \in [H]$, $(s, a) \in \Scal \times \Acal$, and $k \in \NN$.
    Then, we have $\PP( \Ecal_1(\frac{\delta}{4}) ) \ge 1 - \frac{\delta}{4}$.
\end{lemma}
The proof is presented in Appendix~\ref{appx:sub-exponential concentration event}.

\begin{lemma}
\label{lma:underestimation concentration event}
    Define $\Ecal_2(\frac{\delta}{4})$ as the event that
    \begin{align*}
        & (\hat{P}^k - P) \left( 2 W_{h+1}^{\pi^*} - \frac{1}{2 \Qdiff} \left( \Wdiff{h+1}{\pi^*} \right)^2 \right) (s, a) 
        \\
        & \qquad \le \frac{1}{2 \Qdiff} \Var(\Wdiff{h+1}{\pi^*})(s, a) + \frac{13 \Qdiff \ell_1(1, \delta)}{N^k(s, a)} \land \frac{3 \Qdiff}{2} \sqrt{\frac{\ell_2(k, \delta)}{N^k(s, a)}}
    \end{align*}
    holds for all $k \in \NN$, $h \in [H]$, and $(s, a) \in \Scal \times \Acal$.
    Then, we have $\PP( \Ecal_2(\frac{\delta}{4})) \ge 1 - \frac{\delta}{4}$.
\end{lemma}
The proof is presented in Appendix~\ref{appx:underestimation concentration event}.

\begin{lemma}[Lemma 29 in~\citet{lee2025minimax}]
\label{lma:concentration for TV}
    There exists an event $\Ecal_3(\frac{\delta}{4})$ whose probability is at least $1 - \frac{\delta}{4}$ such that for all $(s, a) \in \Scal \times \Acal$, $k \in \NN$, and constants $c, \rho > 0$, the following inequality holds for any functions $V : \Scal \rightarrow [-c, c]$ under $\Ecal_3(\frac{\delta}{4})$:
    \begin{align*}
        \left| (\hat{P}^k - P)V(s, a) \right| \le \frac{1}{c\rho} \Var(V)(s, a) + \frac{c (\rho + \frac{2}{3}) S \ell_2(k, \delta)}{N^k(s, a)}
        \, .
    \end{align*}
\end{lemma}

\begin{lemma}
\label{lma:sum of J with eta}
    Suppose $\{X_h^k\}_{k, h}$ is a sequence of non-negative random variables adapted to filtration $\{\Fcal_h^k\}_{k, h}$ for $k \in \NN$ and $h \in [H]$.
    Let $c > 0$ be a constant and $\{J_h^k\}_{k, h}$ a sequence of random variables recursively defined as $J_{H+1}^k := 0$ and
    \begin{align*}
        J_h^k := \left( X_h^k + \EE [ J_{h+1}^k \mid \Fcal_h^k] \right) \land c
    \end{align*}
     for all $k \in \NN$ and $h \in [H]$.
     Recall that $\eta^k$ is a stopping time defined in \appendixref{appx:notations and definitions}.
     Let $\Ecal_4(\{ X_h^k\}_{k, h}, c, \frac{\delta}{4})$ be the event that 
     \begin{align*}
         \sum_{k=1}^K J_1^k \le 2 \sum_{K=1}^K \sum_{h=1}^{\eta^k - 1} X_h^k + 6 cSA \log \frac{16H}{\delta}
     \end{align*}
     holds for all $K \in \NN$.
     Then, $\PP( \Ecal_4(\{X_h^k\}_{k, h}, c, \frac{\delta}{4}) ) \ge 1 - \frac{\delta}{4}$.     
\end{lemma}
The proof is presented in Appendix~\ref{appx:proof of sum of J with eta}.

\subsubsection{Proof of Lemma~\ref{lma:sub-exponential concentration event}}
\label{appx:sub-exponential concentration event}

\begin{proof}[Proof of \lemmaref{lma:sub-exponential concentration event}]
    Fix $(h, s, a) \in [H] \times \Scal \times \Acal$ and $\delta' \in (0, 1]$.
    Let $\{(R^j, s^j)\}_{j}$ be the sequence of the observed reward and next state pairs when $(s, a)$ is selected. 
    By \assumptionref{assm:subexponential}, $\{(R^j + V_{h+1}^*(s^j) - Q_h^*(s, a) \}_j$ is $(\sigexp(h, s, a), \Valpha)$-sub-exponential conditioned on the previous observations.
    For fixed $i \in \NN$, the following inequality holds for all $n \in \NN$ with probability at least $1 - \frac{\delta'}{2i^2}$ by Lemma~\ref{lma:subexponential concentration}:
    \begin{align*}
        \left| \sum_{j=1}^n \left( R^j + V_{h+1}^*(s^j) - Q_{h+1}^*(s, a) \right) \right| \le \frac{\lambda_i \sigexp^2(h, s, a) n }{2} + \frac{1}{\lambda_i} \log \frac{4i^2}{\delta'}
        \, ,
    \end{align*}
    Taking the union bound and using that $\sum_{i=1}^\infty \frac{1}{2i^2} \le 1$, the inequality above holds for all $i \in \NN$ with probability at least $1 - \delta'$.
    By \lemmaref{lma:time uniform Hoeffding for exponential}, the following inequality also holds for all $n \in \NN$ with probability at least $1 - \delta'$:
    \begin{align*}
        \left| \sum_{j=1}^n \left( R^j + V_{h+1}^*(s^j) - Q_{h+1}^*(s, a) \right) \right| \land n\Vmax \le \sigsub \sqrt{n \log \frac{4(\log e^2 n)^2}{\delta'}}
    \end{align*}
    Then, with probability at least $1 - 2 \delta'$, the following inequality holds for all $i \in \NN$ and $n \in \NN$:
    \begin{align*}
        & \left| \sum_{j=1}^n \left( R^j + V_{h+1}^*(s^j) - Q_{h+1}^*(s, a) \right) \right| \land n \Vmax
        \\
        & \le \left( \frac{\lambda_i \sigexp^2(h, s, a) n }{2} + \frac{1}{\lambda_i} \log \frac{4i^2}{\delta'} \right) \land  \sigsub \sqrt{n \log \frac{4(\log e^2 n)^2}{\delta'}}
        \\
        & \le \frac{\lambda_i \sigexp^2(h, s, a) n }{2} + \left( \frac{1}{\lambda_i} \log \frac{4i^2}{\delta'} \right) \land  \sigsub \sqrt{n \log \frac{4(\log e^2 n)^2}{\delta'}}
        \, ,
    \end{align*}
    where the last inequality uses that $\frac{\lambda_i \sigexp^2(h, s, a) n }{2} \ge 0$.
    Dividing both sides by $n$, plugging in $n = N^k(s, a)$ and $i = \iota_k$, and noting that $\frac{1}{N^k(s, a)} \sum_{j=1}^{N^k(s, a)} R^j = \hat{r}^k(s, a)$ and $\frac{1}{N^k(s, a)}\sum_{j=1}^{N^k(s, a)} V_{h+1}^*(s^j) = \hat{P}^kV_{h+1}^*(s, a)$, we have
    \begin{align*}
        & \left| \hat{r}^k(s, a) + \hat{P}^kV_{h+1}^*(s, a) - Q_{h+1}^*(s, a) \right| \land \Vmax
        \\
        & \le \frac{\lambda_{\iota_k} \sigexp^2(h, s, a)}{2} + \left( \frac{1}{ \lambda_{\iota_k} N^k(s, a)} \log \frac{4{\iota_k}^2}{\delta'} \right) \land  \sigsub \sqrt{\frac{\log \frac{4(\log e^2 N^k(s, a))^2}{\delta'}}{N^k(s, a)}}
        \, .
    \end{align*}
    The proof is completed by bounding $\log N^k(s, a) \le \log kH$, taking the union bound over $(h, s, a) \in [H] \times \Scal \times \Acal$, and plugging in $\delta' = \frac{\delta}{8HSA}$.
\end{proof}

\subsubsection{Proof of Lemma~\ref{lma:underestimation concentration event}}
\label{appx:underestimation concentration event}

\begin{proof}[Proof of \lemmaref{lma:underestimation concentration event}]
    Fix $(h, s, a) \in [H] \times \Scal \times \Acal$ and $\delta' \in (0, 1]$.
    For simplicity, we define $W(s') := 2 \Wdiff{h+1}{\pi^*}(s') - \frac{1}{2\Qdiff}\left( \Wdiff{h+1}{\pi^*} \right)^2(s')$ for $s' \in \Scal$.
    First, note that $W_{h+1}^{\pi^*}(s') = \Wdiff{h+1}{\pi^*}(s') + \Wmin{h+1}{\pi^*}$ and $\Wmin{h+1}{\pi^*}$ does not depend on $s'$, so we have
    \begin{align*}
        & (\hat{P}^k - P) \left( 2 W_{h+1}^{\pi^*} - \frac{1}{2 \Qdiff} \left( \Wdiff{h+1}{\pi^*} \right)^2 \right) (s, a)
        \\
        & = (\hat{P}^k - P) \left( 2 \Wdiff{h+1}{\pi^*} - \frac{1}{2 \Qdiff} \left( \Wdiff{h+1}{\pi^*} \right)^2 \right) (s, a) + (\hat{P}^k - P)\Wmin{h+1}{\pi^*}
        \\
        & = (\hat{P}^k - P)W(s, a) + 0
        \, .
    \end{align*}
    Now, we bound $(\hat{P}^k - P)W(s, a)$ in two ways and take the minimum.
    Note that by $0 \le \Wdiff{h+1}{\pi^*}(s') \le \Qdiff$, we have $0 \le W(s') \le \frac{3\Qdiff}{2}$ for all $s' \in \Scal$.
    By Hoeffding's lemma (\lemmaref{lma:hoeffding's lemma}), $W(s')$ is $(\frac{3\Qdiff}{4}, 0)$-sub-exponential, and by \lemmaref{lma:time uniform Hoeffding for exponential}, we have that with probability at least $1 - \delta'$, for all $k \in \NN$,
    \begin{align}
        (\hat{P}^k - P ) W(s, a) & \le \frac{3 \Qdiff}{2} \sqrt{\frac{1}{N^k(s, a)}\log\frac{2 (\log e^2 N^k(s, a))^2}{\delta'}}
        \nonumber
        \\
        & \le \frac{3 \Qdiff}{2} \sqrt{\frac{1}{N^k(s, a)}\log\frac{2 (\log e^2 kH)^2}{\delta'}}
        \label{eq:Hoeffding bound on W}
        \, ,
    \end{align}
    where we bound $\log e^2 N^k(s, a) \le \log e^2 kH$.
    Next, we apply Freedman's inequality.
    We obtain a bound on the variance $\Var_{s'\sim P(\cdot\mid s, a)}(W(s'))$ by \lemmaref{lma:variance of product} as follows:
    \begin{align*}
        \Var_{s' \sim P(\cdot \mid s, a)}\left( W(s') \right)
        & = 
        \Var_{s' \sim P(\cdot \mid s, a)}\left( 2 \Wdiff{h+1}{\pi^*}(s') - \frac{1}{2\Qdiff}\left( \Wdiff{h+1}{\pi^*} \right)^2(s') \right)
        \\
        & = 
        \Var_{s' \sim P(\cdot \mid s, a)}\left( \Wdiff{h+1}{\pi^*}(s')\left( 2 - \frac{1}{2\Qdiff} \Wdiff{h+1}{\pi^*} (s') \right) \right)
        \\
        & \le 2 \Var_{s' \sim P(\cdot \mid s, a)}\left( \Wdiff{h+1}{\pi^*}(s') \right)\cdot 2^2 + 2 {\Qdiff}^2 \Var_{s' \sim P(\cdot \mid s, a)}\left( 2 - \frac{1}{2\Qdiff}\Wdiff{h+1}{\pi^*}(s') \right) 
        \\
        & = \frac{17}{2} \Var_{s' \sim P(\cdot \mid s, a)}\left( \Wdiff{h+1}{\pi^*}(s') \right)
    \end{align*}
    From \lemmaref{lma:bounded is subexponential}, we obtain that $W(s') - PW(s, a)$ is $(13 \Var( \Wdiff{h+1}{\pi^*} )(s, a), \frac{3}{2} \Qdiff)$-sub-exponential, where the constant $13$ comes from $(e - 2) \cdot 2 \cdot \frac{17}{2} \le 13$.
    Applying \lemmaref{lma:subexponential concentration} with $\lambda = \frac{1}{13\Qdiff}$, we have that for all $k \in \NN$, the following inequality holds with probability at least $1 - \delta'$:
    \begin{align}
    \label{eq:Freedman bound on W}
        ( \hat{P}^k - P ) W(s, a) \le \frac{1}{2 \Qdiff} \Var(\Wdiff{h+1}{\pi^*})(s, a) + \frac{13 \Qdiff}{N^k(s, a)}\log \frac{1}{\delta'}
        \, .
    \end{align}
    Taking the minimum over the bounds of Eq.~\eqref{eq:Hoeffding bound on W} and Eq.~\eqref{eq:Freedman bound on W}, and by the union bound, we obtain that for all $k \in \NN$, with probability at least $1 - 2 \delta'$, we have
    \begin{align*}
        & (\hat{P}^k - P ) W(s, a) 
        \nonumber
        \\
        & \le \left(  \frac{1}{2 \Qdiff} \Var(\Wdiff{h+1}{\pi^*})(s, a) + \frac{13 \Qdiff}{N^k(s, a)}\log \frac{1}{\delta'} \right) \land \frac{3 \Qdiff}{2} \sqrt{\frac{1}{N^k(s, a)}\log\frac{2 (\log e^2kH)^2}{\delta'}}
        \nonumber
        \\
        & \le \frac{1}{2 \Qdiff} \Var(\Wdiff{h+1}{\pi^*})(s, a) + \left(  \frac{13 \Qdiff}{N^k(s, a)}\log \frac{1}{\delta'} \right) \land \frac{3 \Qdiff}{2} \sqrt{\frac{1}{N^k(s, a)}\log\frac{2 (\log e^2kH)^2}{\delta'}}
        \, ,
    \end{align*}
    where the last inequality uses that $\frac{1}{2 \Qdiff} \Var(\Wdiff{h+1}{\pi^*})(s, a)\ge 0$.
    The proof is completed by plugging in $\delta' = \frac{\delta}{8HSA}$ and taking the union bound over $(h, s, a) \in [H] \times \Scal \times \Acal$.
\end{proof}

\subsubsection{Proof of Lemma~\ref{lma:sum of J with eta}}
\label{appx:proof of sum of J with eta}

Before proving \lemmaref{lma:sum of J with eta}, we present the following lemma.
This lemma is a minor generalization of Lemma 15 in~\citet{lee2025minimax}, based on an observation that its proof holds for any random variables adapted to $\{ \Fcal_h^k \}_{k, h}$.
While the proof directly follows that of~\citet{lee2025minimax}, we provide it here for completeness.

\begin{lemma}[Restatement of Lemma~\ref{lma:sum of J informal}]
\label{lma:sum of J}
    Let $\{X_h^k\}_{k, h}$ be a sequence of non-negative random variables adapted to filtration $\{\Fcal_h^k\}_{k, h}$.
    Let $c > 0$ be a constant.
    Recursively define $\{J_h^k\}_{h, k}$ as $J_{H+1}^k := 0$ and
    \begin{align*}
        J_h^k := \left( X_h^k + \EE [ J_{h+1}^k \mid \Fcal_h^k] \right) \land c
    \end{align*}
     for all $k \in \NN$ and $h \in [H]$.
     Then, for any $\delta \in (0, 1]$, the following inequality holds for all $K \in \NN$ with probability at least $1 - \delta$:
     \begin{align*}
         \sum_{k=1}^K J_1^k \le 2 \sum_{k=1}^K \sum_{h=1}^H X_h^k + 6c \log \frac{2}{\delta}
         \, .
     \end{align*}
\end{lemma}

\begin{proof}[Proof of \lemmaref{lma:sum of J}]
    We expand $J_1^k$ as follows:
    \begin{align*}
        J_1^k
        & = \sum_{h=1}^H J_h^k - \sum_{h=2}^{H} J_h^k
        \\
        & \le \sum_{h=1}^H \left( X_h^k + \EE[ J_{h+1}^k \mid \Fcal_h^k ] \right) - \sum_{h=2}^{H} J_h^k
        \\
        & = \sum_{h=1}^H X_h^k + \sum_{h=1}^H \left( \EE[ J_{h+1}^k \mid \Fcal_h^k ] - J_{h+1}^k \right)
        \, .
    \end{align*}
    Taking the sum over $k \in [K]$, we obtain
    \begin{align}
        \sum_{k=1}^K J_1^k \le \sum_{k=1}^K \sum_{h=1}^H X_h^k + \sum_{k=1}^K \sum_{h=1}^H \left( \EE[ J_{h+1}^k \mid \Fcal_{h}^k ] - J_{h+1}^k \right)
        \label{eq:sum lemma 0}
        \, .
    \end{align}
    Note that $\{ \EE[ J_{h+1}^k \mid \Fcal_{h}^k ] - J_{h}^k \}_{h, k}$ is a martingale sequence with a finite range $[-c, c]$.
    By \lemmaref{lma:freedman inequality} with $\lambda = \frac{1}{4(e-2)c}$, the following inequality holds for all $K \in \NN$ with probability at least $1 - \frac{\delta}{2}$:
    \begin{align}
        \sum_{k=1}^K \sum_{h=1}^{H} \left( \EE[ J_{h+1}^k \mid \Fcal_{h}^k ] - J_{h+1}^k \right)
        & \le \sum_{k=1}^K \sum_{h=1}^{H} \frac{1}{4c} \Var( J_{h+1}^k \mid \Fcal_{h}^k) + 4(e-2)c \log \frac{2}{\delta}
        \label{eq:sum lemma 1}
        \, .
    \end{align}
    By \lemmaref{lma:variance decomposition}, we have $\Var(J_{h+1}^k \mid \Fcal_{h}^k) \le \EE[ (J_{h+1}^k)^2 \mid \Fcal_{h}^k] - (J_{h}^k)^2 + 2c X_{h}^k $, hence the sum of the variances is bounded as
    \begin{align}
        & \sum_{k=1}^K \sum_{h=1}^{H} \Var(J_{h+1}^k \mid \Fcal_{h}^k)
        \nonumber
        \\
        & \le \sum_{k=1}^K \sum_{h=1}^H \left(  \EE[ (J_{h+1}^k)^2 \mid \Fcal_{h}^k] - (J_{h}^k)^2 + 2c X_{h}^k \right) 
        \nonumber
        \\
        & \le \sum_{k=1}^K \left( (J_{H+1}^k)^2 - (J_1^k)^2 \right) + \sum_{k=1}^K \sum_{h=1}^{H} \left( \EE[ (J_{h+1}^k)^2 \mid \Fcal_{h}^k] - (J_{h+1}^k)^2\right) + 2c \sum_{k=1}^K \sum_{h=1}^{H} X_{h}^k
        \nonumber
        \\
        & \le \sum_{k=1}^K \sum_{h=1}^{H} \left( \EE[ (J_{h+1}^k)^2 \mid \Fcal_{h}^k] - (J_{h+1}^k)^2\right) + 2c \sum_{k=1}^K \sum_{h=1}^{H} X_{h}^k
        \label{eq:sum of variances}
        \, .
    \end{align}
    Note that $\{ \EE[ (J_{h+1}^k)^2 \mid \Fcal_{h}^k] - (J_{h+1}^k)^2 \}_{k, h}$ is a martingale difference sequence with a finite range $[ - c^2, c^2]$. 
    Again, by \lemmaref{lma:freedman inequality} with $\lambda = \frac{1}{8(e-2)c^2}$, the following inequality holds for all $k \in K$ with probability at least $1 - \frac{\delta}{2}$:
    \begin{align}
        \sum_{k=1}^K \sum_{h=1}^{H} \left( \EE[ (J_{h+1}^k)^2 \mid \Fcal_{h}^k] - (J_{h+1}^k)^2\right)
        & \le \frac{1}{8c^2} \sum_{k=1}^K \sum_{h=1}^{H}\Var( (J_{h+1}^k)^2 \mid \Fcal_{h}^k) + 8(e - 2)c^2 \log \frac{2}{\delta}
        \label{eq:sum lemma 2}
        \, .
    \end{align}
    By \lemmaref{lma:variance of product}, we have $\Var((J_{h+1}^k)^2 \mid \Fcal_{h}^k) \le 4 c^2 \Var(J_{h+1}^k \mid \Fcal_{h}^k)$.
    Combining this bound together with inequalities~\eqref{eq:sum of variances} and~\eqref{eq:sum lemma 2}, we obtain that
    \begin{align*}
        & \sum_{k=1}^K \sum_{h=1}^{H} \Var(J_{h+1}^k \mid \Fcal_{h}^k)
        \\
        & \le \sum_{k=1}^K \sum_{h=1}^{H} \left( \EE[ (J_{h+1}^k)^2 \mid \Fcal_{h}^k] - (J_{h+1}^k)^2\right) + 2c \sum_{k=1}^K \sum_{h=1}^{H} X_{h}^k
        \\
        & \le \frac{1}{8c^2} \sum_{k=1}^K \sum_{h=1}^{H}\Var( (J_{h+1}^k)^2 \mid \Fcal_{h}^k) + 8(e - 2)c^2 \log \frac{2}{\delta} + 2c \sum_{k=1}^K \sum_{h=1}^{H} X_{h}^k
        \\
        & \le \frac{1}{2} \sum_{k=1}^K \sum_{h=1}^{H} \Var(J_{h+1}^k \mid \Fcal_{h}^k)  + 2c \sum_{k=1}^K \sum_{h=1}^{H} X_{h}^k + 8(e - 2)c^2 \log \frac{2}{\delta}
        \, ,
    \end{align*}
    which implies that
    \begin{align*}
        \sum_{k=1}^K \sum_{h=1}^H \Var(J_h^k \mid \Fcal_{h-1}^k) & \le 4 c \sum_{k=1}^K \sum_{h=1}^H X_h^k + 16 (e-2)c^2 \log \frac{2}{\delta}
        \, .
    \end{align*}
    Plugging this bound into inequality~\eqref{eq:sum lemma 1} and then into inequality~\eqref{eq:sum lemma 0}, we conclude that
    \begin{align*}
        \sum_{k=1}^K J_1^k
        & \le \sum_{k=1}^K \sum_{h=1}^H X_h^k + \sum_{k=1}^K \sum_{h=1}^H \left( \EE[ J_{h+1}^k \mid \Fcal_{h}^k ] - J_{h+1}^k \right)
        \\
        & \le \sum_{k=1}^K \sum_{h=1}^H X_h^k + \sum_{k=1}^K \sum_{h=1}^{H} \frac{1}{4c} \Var( J_{h+1}^k \mid \Fcal_{h}^k) + 4(e-2)c \log \frac{2}{\delta}
        \\
        & \le 2 \sum_{k=1}^K \sum_{h=1}^H X_h^k + 8(e-2)c \log \frac{2}{\delta}
        \\
        & \le 2 \sum_{k=1}^K \sum_{h=1}^H X_h^k + 6 c \log \frac{2}{\delta}
        \, .
    \end{align*}
    Taking the union bound over the events of inequality~\eqref{eq:sum lemma 1} and inequality~\eqref{eq:sum lemma 2}, we obtain that the inequality above holds for all $K \in \NN$ with probability at least $1 - \delta$.
\end{proof}

\begin{proof}[Proof of \lemmaref{lma:sum of J with eta}]
    Define another sequence of random variable $\tilde{X}_h^k$ as
    \begin{align*}
        \tilde{X}_h^k := \begin{cases}
            X_h^k & (h < \eta^k)
            \\
            c & (h = \eta^k)
            \\
            0 & (h > \eta^k)
        \end{cases}
        \, .
    \end{align*}
    It is clear that $\tilde{X}_h^k$ is also adapted to $\Fcal_h^k$.
    Define $\tilde{J}_{H+1}^k := 0$ and $\tilde{J}_h^k := ( \tilde{X}_h^k  + \EE[ \tilde{J}_{h+1}^k \mid \Fcal_h^k]) \land c$ for $h \in [H]$.
    \\
    First, we prove $J_h^k \ind\{ h \le \eta^k\} \le \tilde{J}_h^k \ind\{ h \le \eta^k\}$ with backward induction on $h$.
    This inequality implies $J_1^k \le \tilde{J}_1^k$ in particular since $\eta^k \ge 1$ always holds.
    The inequality is trivial for $h = H+1$ or for $h > \eta^k$ as $0 \le 0$.
    Suppose $h = \eta^k < H+1$.
    Then, we have
    \begin{align*}
        J_h^k \ind\{h \le \eta^k \} \le c = \tilde{J}_h^k \ind\{h \le \eta^k\}
        \, ,
    \end{align*}
    where the first inequality is by definition, and the second equality is due to that $X_h^k = c$ when $h = \eta^k$ and hence $\tilde{J}_h^k = (X_h^k + \EE[ J_{h+1}^k \mid \Fcal_h^k]) \land c = (c + \EE[ J_{h+1}^k \mid \Fcal_h^k]) \land c = c$.
    When $h < \eta^k$, we have
    \begin{align*}
        J_h^k = ( X_h^k + \EE[ J_{h+1}^k \mid \Fcal_h^k]) \land c = ( \tilde{X}_h^k + \EE[ J_{h+1}^k \mid \Fcal_h^k ] ) \land c \le ( \tilde{X}_h^k + \EE[ \tilde{J}_{h+1}^k \mid \Fcal_h^k ] ) \land c = \tilde{J}_h^k
        \, ,
    \end{align*}
    where the equalities are by definitions, and the inequality is due to the induction hypothesis.
    Therefore, we conclude that $J_h^k\ind\{h \le \eta^k \} \le \tilde{J}_h^k \ind\{ h \le \eta^k \}$, and in particular $J_1^k \le \tilde{J}_1^k$ holds for all $k \in \NN$.
    \\
    Then, we bound $\sum_{k=1}^K \tilde{J}_1^k$.
    By \lemmaref{lma:sum of J}, with probability at least $1 - \frac{\delta}{4}$, the following holds for all $K \in \NN$.
    \begin{align*}
        \sum_{K=1}^K \tilde{J}_1^k & \le 2 \sum_{k=1}^K \sum_{h=1}^H \tilde{X}_h^k + 6 c \log \frac{8}{\delta}
        \\
        & = 2 \sum_{k=1}^K \sum_{h=1}^{H} \left( X_h^k \ind\{ h < \eta^k\} + c \ind\{ h = \eta^k\}\right) + 6 c \log \frac{8}{\delta}
        \\
        & = 2 \sum_{k=1}^K \sum_{h=1}^{\eta^k - 1} X_h^k + \sum_{k=1}^K c \ind\{ \eta^k < H + 1 \} + 6 c \log \frac{8}{\delta}
        \, .
    \end{align*}
    By \lemmaref{lma:bound on eta k}, we have $\sum_{k=1}^K \ind\{ \eta^k < H+1\} \le SA \log_2 2H \le 2 SA \log 2H$.
    The proof is completed by that $2 c SA \log 2H + 6 c \log \frac{1}{\delta} \le 3 c SA \log 2H + 3 c SA \log \frac{8}{\delta} \le 6c SA \log \frac{16H}{\delta}$.
\end{proof}

\subsection{General Properties of \texorpdfstring{$\AlgName$}{EQO+}}

\label{appx:general properties of EQO}

In this section, we prove the following lemma that bounds the instantaneous regret of $\AlgName$ of an episode.
The proof is simply combining \lemmaref{lma:quasi optimism} and \lemmaref{lma:est true diff bound}, where Lemma~\ref{lma:quasi optimism} is a formal version of Lemma~\ref{lma:quasi optimism informal}.
We also prove these lemmas in this section.

First, we formally define a function $U_h^k(s)$.
Let $\beta^k(s, a) := 2 b^k(s, a) + \frac{9 \Vmax S \ell_2(k, \delta)}{N^k(s, a)}$.
Then, $U_h^k(s)$ is iteratively defined starting from $U_{H+1}^k(s) := 0$ and for $a = \pi_h^k(s)$,
\begin{align*}
    U_h^k(s) := \left( 2 \beta^k(s, a) + PU_{h+1}^k(s, a) \right) \land 2 \Vmax
    \, .
\end{align*}

Lemma~\ref{lma:episode regret} states that the instantaneous regret is bounded by a quasi-optimism term and $U_1^k(s_1^k)$.

\begin{lemma}
\label{lma:episode regret}
    Fix any $\delta \in (0, 1]$ and $k \in \NN$.
    Assume the event of $\Ecal_1(\frac{\delta}{4}) \cap \Ecal_2(\frac{\delta}{4}) \cap \Ecal_3(\frac{\delta}{4})$.
    Suppose $k > \kap(\delta)$.
    Then, the instantaneous regret of episode $k$ is bounded as
    \begin{align*}
        V_1^*(s_1^k) - V_1^{\pi^k}(s_1^k) \le \frac{18 \Qrange \ell_1(\iota_k, \delta)}{c_{1, k}} + U_1^k(s_1^k)
        \, .
    \end{align*}
\end{lemma}

\begin{proof}
We present the following two key lemmas.

\begin{lemma}[Quasi-optimism, formal version of Lemma~\ref{lma:quasi optimism informal}]
\label{lma:quasi optimism}
    Fix $\delta \in (0, 1]$ and $k > \kap(\delta)$.
    Then, under the event $\Ecal_1(\delta) \cap \Ecal_2(\delta)$, the estimated value $V_h^k(s)$ in Algorithm~\ref{alg:eqo2} satisfies
    \begin{align*}
        V_h^*(s) - V_h^k(s) \le ( 2 \lambda_{\iota_k} \Qrange) \land \Vmax
    \end{align*}
    for all $h \in [H]$ and $s \in \Scal$.
\end{lemma}
The proof of Lemma~\ref{lma:quasi optimism} is provided in Appendix~\ref{appx:proof of quasi optimism}.

\begin{lemma}
\label{lma:est true diff bound}
    Take any $\delta \in (0, 1]$, $h \in [H]$, $s \in \Scal$, and $k > \kap(\delta)$.
    Then, under the event of $\Ecal_1(\frac{\delta}{4}) \cap \Ecal_3(\frac{\delta}{4})$, we have
    \begin{align*}
        V_h^k(s) -V_h^{\pi^k}(s) \le \frac{5}{2} \lambda_{\iota_k} \Qrange + U_h^k(s)
        \, .
    \end{align*}
\end{lemma}
The proof of Lemma~\ref{lma:est true diff bound} is provided in Appendix~\ref{appx:proof of est true diff bound}.

    By combining \lemmaref{lma:quasi optimism,lma:est true diff bound}, we obtain that
    \begin{align*}
        V_1^*(s_1^k) - V_1^{\pi^k}(s_1^k)
        & \le 2 \lambda_{\iota_k}\Qrange + V_1^k(s_1^k) - V_1^{\pi^k}(s_1^k)
        \\
        & \le 2 \lambda_{\iota_k}\Qrange + \frac{5}{2} \lambda_{\iota_k}\Qrange + U_1^k(s_1^k)
        \\
        & = \frac{9}{2} \lambda_{\iota_k}\Qrange + U_1^k(s_1^k)
        \, .       
    \end{align*}
    By \lemmaref{lma:choice of lambda}, we have $\lambda_{\iota_k} \le \frac{4\ell_1(\iota_k, \delta)}{c_{1, k}}$, which yields that
    \begin{align*}
        V_1^*(s_1^k) - V_1^{\pi^k}(s_1^k) \le \frac{18 \Qrange \ell_1(\iota_k, \delta)}{c_{1, k}} + U_1^k(s_1^k)
        \, .
    \end{align*}
\end{proof}

\subsubsection{Proof of Lemma~\ref{lma:quasi optimism}}
\label{appx:proof of quasi optimism}
\begin{proof}[Proof of Lemma~\ref{lma:quasi optimism}]
    We prove the following stronger inequality by backward induction on $h$.
    \begin{align*}
        V_h^*(s) - V_h^k(s) \le \lambda_{\iota_k} \left( 2 W_h^{\pi^*}(s) - \frac{1}{2 \Qdiff} (\Wdiff{h}{\pi^*}(s))^2 \right) \land \Vmax
    \end{align*}
    The inequality holds trivially for $h = H+1$ as $0 \le 0$.
    Now, suppose the inequality holds for $h + 1$.
    If $V_h^k(s) = \Vmax$, then the inequality holds since we have $V_h^*(s) - V_h^k(s) \le 0$ and
    \begin{align*}
        2W_h^{\pi^*}(s) - \frac{1}{2\Qdiff}(\Wdiff{h}{\pi^*}(s))^2
        \ge 
        2\Wdiff{h}{\pi^*}(s) - \frac{1}{2\Qdiff}(\Wdiff{h}{\pi^*}(s))^2
        \ge 0
        \, ,
    \end{align*}
    where the last inequality uses that $2x - \frac{x^2}{2c} \ge 0$ holds for all $x \in [0, c]$, where we plug in $x = \Wdiff{h}{\pi^*}$ and $c = \Qdiff$.
    \\
    Suppose $V_h^k(s) < \Vmax$.
    It implies that for $a:= \pi_h^k(s)$ and $a^*:=\pi_h^*(s)$, we have $V_h^k(s) = Q_h^k(s, a) \ge Q_h^k(s, a^*) = (\hat{r}^k(s, a^*) + \hat{P}^kV_{h+1}^k(s, a^*))_+ + b^k(s, a^*)$, where the condition $V_h^k(s) < \Vmax$ is used for the last equality to ensure that the clipping by $\Vmax$ did not happen.
    Then, we have
    \begin{align}
        V_h^*(s) - V_h^k(s)
        & \le V_h^*(s) - \left(\hat{r}^k(s, a^*) + \hat{P}^kV_{h+1}^k(s, a^*) \right)_+ - b^k(s, a^*)
        \nonumber
        \\
        & = V_h^*(s) \land \Vmax  + \left( - \hat{r}^k(s, a^*) - \hat{P}^kV_{h+1}^k(s, a^*) \right) \land 0 - b^k(s, a^*)
        \nonumber
        \\
        & \le \underbrace{\left( V_h^*(s) - \hat{r}^k(s, a^*) - \hat{P}^kV_{h+1}^k(s, a^*) \right)}_{I_1} \land \Vmax - b^k(s, a^*)
        \, .
        \label{eq:quasi optimism 1}
    \end{align}
    We focus on bounding $I_1$.
    By adding and subtracting $\hat{P}^kV_{h+1}^*(s, a^*)$, we obtain that
    \begin{align}
        I_1 & = V_h^*(s)- \hat{r}^k(s, a^*) - \hat{P}^kV_{h+1}^*(s, a^*)+ \underbrace{\hat{P}^k(V_{h+1}^* - V_{h+1}^k)(s, a^*)}_{I_2}
        \label{eq:quasi optimism 2}
        \, .
    \end{align}
    We bound $I_2$ as follows.
    Using the induction hypothesis, we have $V_{h+1}^*(s') - V_{h+1}^{k}(s') \le \lambda_{\iota_k} ( 2 W_{h+1}^{\pi^*}(s') - \frac{1}{2\Qdiff} (\Wdiff{h+1}{\pi^*}(s'))^2)$ for all $s' \in \Scal$, and hence
    \begin{align}
        I_2 \le \lambda_{\iota_k} \hat{P}^k \left( 2 W_{h+1}^{\pi^*} - \frac{1}{2 \Qdiff} \left( \Wdiff{h+1}{\pi^*} \right)^2 \right) (s, a^*)
        \label{eq:quasi optimism 2.5}
        \, .
    \end{align}
    Under $\Ecal_2(\frac{\delta}{4})$ (\lemmaref{lma:underestimation concentration event}), we have
    \begin{align}
        & (\hat{P}^k - P) \left( 2 W_{h+1}^{\pi^*} - \frac{1}{2 \Qdiff} \left( \Wdiff{h+1}{\pi^*} \right)^2 \right) (s, a^*)
        \nonumber
        \\
        & \le \frac{1}{2\Qdiff} \Var(\Wdiff{h+1}{\pi^*})(s, a^*) + \frac{13 \Qdiff \ell_1(1, \delta)}{N^k(s, a^*)} \land \frac{3 \Qdiff}{2} \sqrt{\frac{ \ell_2(k, \delta)}{N^k(s, a^*)}}
        \nonumber
        \\
        & =: \frac{1}{2\Qdiff} \Var(\Wdiff{h+1}{\pi^*})(s, a^*) + b_2^k(s, a^*)
        \label{eq:quasi optimism 3}
        \, ,
    \end{align}
    where we denote $ \frac{13 \Qdiff \ell_1(1, \delta)}{N^k(s, a^*)} \land \frac{3 \Qdiff}{2} \sqrt{\frac{ \ell_2(k, \delta)}{N^k(s, a^*)}}$ by $b_2^k(s, a^*)$.
    We apply \lemmaref{lma:variance decomposition} to $X = \Wdiff{h}{\pi^*}$, $Y = \frac{1}{2}\sigexp^2(h, s, a^*) - \Wmin{h}{\pi^*} + \Wmin{h+1}{\pi^*}$, and $Z = \Wdiff{h+1}{\pi^*}(s')$ with $s' \sim P(s, a^*)$, and obtain that
    \begin{align*}
        \Var(\Wdiff{h+1}{\pi^*})(s, a^*) \le P(\Wdiff{h+1}{\pi^*})^2(s, a^*) - (\Wdiff{h}{\pi^*}(s))^2 + \Qdiff \sigexp^2(h, s, a^*)
        \, ,
    \end{align*}
    where we bound $Y \le \frac{1}{2}\sigsub^2(h, s, a^*)$ using that $\Wmin{h}{\pi^*} = \min_{s' \in \Scal} W_h^{\pi^*}(s') \ge \min_{s' \in \Scal} PW_{h+1}^{\pi^*}(s', \pi_h^*(s')) \ge \Wmin{h+1}{\pi^*}$.
    Then, when taking the sum of the variance term and the expectation of the amount of underestimation, $\frac{1}{2\Qdiff}P(\Wdiff{h+1}{\pi^*})^2(s, a^*)$ terms cancel out as follows:
    \begin{align*}
        &  \frac{1}{2\Qdiff} \Var(\Wdiff{h+1}{\pi^*})(s, a^*) + P \left( 2 W_{h+1}^{\pi^*} - \frac{1}{2 \Qdiff} \left( \Wdiff{h+1}{\pi^*} \right)^2 \right) (s, a^*)
        \\
        & \le \frac{1}{2} \sigexp^2(h, s, a^*) +  2 P W_{h+1}^{\pi^*}(s, a^*)  - \frac{1}{2\Qdiff}(\Wdiff{h}{\pi^*}(s))^2
        \\
        & = W_h^{\pi^*}(s) +  P W_{h+1}^{\pi^*}(s, a^*)  - \frac{1}{2\Qdiff}(\Wdiff{h}{\pi^*}(s))^2
        \, .
    \end{align*}
    Plugging this bound into inequality~\eqref{eq:quasi optimism 3} and rearranging, we obtain that
    \begin{align*}
        & \hat{P}^k \left( 2 W_{h+1}^{\pi^*} - \frac{1}{2 \Qdiff} \left( \Wdiff{h+1}{\pi^*} \right)^2 \right) (s, a^*)
        \nonumber
        \\
        & \le P \left( 2 W_{h+1}^{\pi^*} - \frac{1}{2 \Qdiff} \left( \Wdiff{h+1}{\pi^*} \right)^2 \right) (s, a^*) + \frac{1}{2\Qdiff} \Var(\Wdiff{h+1}{\pi^*})(s, a^*) + b_2^k(s, a^*)
        \nonumber
        \\
        & \le  W_h^{\pi^*}(s) +  P W_{h+1}^{\pi^*}(s, a^*) - \frac{1}{2\Qdiff} (\Wdiff{h}{\pi^*}(s))^2 + b_2^k(s, a^*)
        \\
        & =: I_2'
        \, ,
    \end{align*}
    where we denote the final bound by $I_2'$.
    Then, we have $I_2 \le \lambda_{\iota_k} I_2'$ by inequality~\eqref{eq:quasi optimism 2.5}.
    We note that $I_2'$ is always greater than $0$ since we have $W_h^{\pi^*}(s) - \frac{1}{2\Qdiff}(\Wdiff{h}{\pi^*}(s))^2 \ge \Wdiff{h}{\pi^*}(s) - \frac{1}{2\Qdiff}(\Wdiff{h}{\pi^*}(s))^2 \ge 0$, which comes from that $x - \frac{x^2}{2c} \ge 0$ for $x \in [0, c]$, and the other terms are also greater than $0$.
    Applying the bound $I_2 \le \lambda_{\iota_k} I_2'$ to inequality~\eqref{eq:quasi optimism 2}, we obtain that
    \begin{align*}
        I_1 \land \Vmax 
        & \le (V_h^*(s) - \hat{r}^k(s, a^*) - \hat{P}^k V_{h+1}^*(s, a^*) + \lambda_{\iota_k} I_2' ) \land \Vmax
        \\
        & \le (V_h^*(s) - \hat{r}^k(s, a^*) - \hat{P}^k V_{h+1}^*(s, a^*)) \land \Vmax + \lambda_{\iota_k} I_2' 
        \, .
    \end{align*}
    The remaining terms are bounded under $\Ecal_1(\frac{\delta}{4})$ (\lemmaref{lma:sub-exponential concentration event}) as follows.
    \begin{align}
        & (V_h^*(s) - \hat{r}^k(s, a^*) - \hat{P}^k V_{h+1}^*(s, a^*)) \land \Vmax
        \nonumber
        \\
        & \le \left| Q_h^*(s, a^*)  - \hat{r}^k(s, a^*) - \hat{P}^kV_{h+1}^*(s, a^*) \right| \land \Vmax
        \nonumber
        \\
        & \le \frac{\lambda_{\iota_k} \sigexp^2(h, s, a^*)}{2} + \frac{\ell_1({\iota_k}, \delta)}{\lambda_{\iota_k} N^k(s, a^*)}  \land \sigsub\sqrt{ \frac{\ell_{2}(k, \delta)}{N^k(s, a^*)}}
        \nonumber
        \\
        & =: \frac{\lambda_{\iota_k} \sigexp^2(h, s, a^*)}{2} + b_1^k(s, a^*)
        \label{eq:quasi optimism 5}
        \, ,
    \end{align}
    where we define $b_1^k(s, a^*) := \frac{\ell_1({\iota_k}, \delta)}{\lambda_{\iota_k} N^k(s, a^*)}  \land \sigsub\sqrt{ \frac{\ell_{2}(k, \delta)}{N^k(s, a^*)}}$.
    Therefore, we obtain that
    \begin{align*}
        & I_1 \land \Vmax 
        \\
        & \le \frac{\lambda_{\iota_k} \sigexp^2(h, s, a^*)}{2} + b_1^k(s, a^*) + \lambda_{\iota_k} I_2'
        \\
        & = \frac{\lambda_{\iota_k} \sigexp^2(h, s, a^*)}{2} + b_1^k(s, a^*)
        \\
        & \qquad + \lambda_{\iota_k} \left( W_h^{\pi^*}(s) + P W_{h+1}^{\pi^*}(s, a^*) - \frac{1}{2\Qdiff} (\Wdiff{h}{\pi^*}(s))^2 + b_2^k(s, a^*)\right)
        \\
        & = \lambda_{\iota_k} \left(2W_h^{\pi^*}(s) - \frac{1}{2\Qdiff} (\Wdiff{h}{\pi^*}(s))^2 \right) + b_1^k(s, a^*) + \lambda_{\iota_k}b_2^k(s, a^*)
        \, .
    \end{align*}
    Plugging this bound into inequality~\eqref{eq:quasi optimism 1}, we obtain that
    \begin{align}
        V_h^*(s) - V_h^k(s) & \le \lambda_{\iota_k} \left( 2 W_{h}^{\pi^*}(s) - \frac{1}{2\Qdiff} (\Wdiff{h}{\pi^*}(s))^2 \right)
        \\
        & \qquad  + b_1^k(s, a^*) + \lambda_{\iota_k}b_2^k(s, a^*) - b^k(s, a^*)
        \label{eq:quasi optimism 4}
        \, .
    \end{align}
    Now, we show that $b^k(s, a^*) \ge b_1^k(s, a^*) + \lambda_{\iota_k}b_2^k(s, a^*)$.
    We first have
    \begin{align*}
        b_1^k(s, a^*) + \lambda_{\iota_k} b_2^k(s, a^*)
        & \le \frac{\ell_1({\iota_k}, \delta)}{\lambda_{\iota_k} N^k(s, a^*)}  \land \sigsub\sqrt{ \frac{\ell_{2}(k, \delta)}{N^k(s, a^*)}}
        \\
        &\qquad 
        + \frac{13 \lambda_{\iota_k}\Qdiff \ell_1(1, \delta)}{N^k(s, a^*)} \land \frac{3 \lambda_{\iota_k}\Qdiff}{2} \sqrt{\frac{ \ell_2(k, \delta)}{N^k(s, a^*)}}
        \\
        & \le \frac{1}{N^k(s, a^*)} \left(\frac{\ell_1({\iota_k}, \delta)}{\lambda_{\iota_k}} + 13 \lambda_{\iota_k}\Qdiff \ell_1(1, \delta) \right)
        \\
        & \qquad \land \sqrt{\frac{\ell_2(k, \delta)}{N^k(s, a^*)}} \left( \sigsub + \frac{3 \lambda_{\iota_k}\Qdiff}{2}\right)
    \end{align*}
    From $k > \kap(\delta)$, we have $c_{1, k} \ge (2 \sqrt{13 \Qdiff} \lor 2 \Valpha ) \ell_1(1, \delta)$.
    Then, we can use \lemmaref{lma:choice of lambda} to guarantee that $\frac{\ell_1(\iota_k, \delta)}{\lambda_{\iota_k}} + 13 \lambda_{\iota_k} \Qdiff \ell_1(1, \delta) \le c_{1, k}$ and $\lambda_{\iota_k} \le \frac{4 \ell_1(\iota_k, \delta)}{c_{1, k}}$.
    We also have $c_{2, k} \ge (\sigsub + \frac{6 \Qdiff \ell_1(\iota_k, \delta)}{c_{1, k}}) \sqrt{\ell_2(k, \delta)}$ when $k > \kap(\delta)$, so we obtain that
    \begin{align*}
        & \frac{1}{N^k(s, a^*)} \left( \frac{\ell_1(\iota_k, \delta)}{\lambda_{\iota_k}} + 13 \lambda_{\iota_k} \Qdiff \ell_1(1, \delta) \right) \land \sqrt{\frac{\ell_2(k, \delta)}{N^k(s, a^*)}} \left( \sigsub + \frac{3 \lambda_{\iota_k} \Qdiff}{2} \right)
        \\
        & \le \frac{c_{1, k}}{N^k(s, a^*)}  \land \sqrt{\frac{\ell_2(k, \delta)}{N^k(s, a^*)}} \left( \sigsub + \frac{6 \Qdiff \ell_1(\iota_k, \delta)}{c_{1, k}} \right)
        \\
        & \le \frac{c_{1, k}}{N^k(s, a^*)} \land \frac{c_{2, k}}{N^k(s, a)}
        \\
        & = b^k(s, a^*)
        \, .
    \end{align*}
    We have proved that $b^k(s, a^*) \ge b_1^k(s, a^*) + \lambda_{\iota_k} b_2^k(s, a^*)$, so from inequality~\eqref{eq:quasi optimism 4}, we obtain that
    \begin{align*}
        V_h^*(s) - V_h^k(s) \le \lambda_{\iota_k} \left( 2 W_h^{\pi^*}(s, a^*) - \frac{1}{2\Qdiff} (\Wdiff{h}{\pi^*}(s))^2 \right)
        \, .
    \end{align*}
    Lastly, we apply the trivial bound of $\Vmax$ that comes from $V_h^*(s) \le \Vmax$ and $V_h^k(s) \ge 0$ and obtain
    \begin{align*}
        V_h^*(s) - V_h^k(s) \le \lambda_{\iota_k} \left( 2 W_{h}^{\pi^*}(s, a^*) - \frac{1}{2\Qdiff} (\Wdiff{h}{\pi^*}(s))^2 \right) \land \Vmax
        \, ,
    \end{align*}
    which completes the induction step.
\end{proof}

\subsubsection{Proof of Lemma~\ref{lma:est true diff bound}}
\label{appx:proof of est true diff bound}
To prove the lemma, we require the following lemma.
It is an analogue of Lemma 13 in \citet{lee2025minimax}, but we obtain improved dependence on the variance term $\sigexp^2(h, s, a)$.

\begin{lemma}
\label{lma:est true diff diff bound}
    Take any $h \in [H]$, $s \in \Scal$, and $k > \kap(\delta)$.
    Let $a:=\pi_h^k(s)$ and $\varepsilon := 2 \lambda_{\iota_k} \Qrange \land \Vmax$.
    Recall that $\beta^k(s, a) := 2b^k(s, a) + \frac{9\Vmax S \ell_2(k, \delta)}{N^k(s, a)} $.
    Then, under $\Ecal_1(\frac{\delta}{4}) \cap \Ecal_3(\frac{\delta}{4})$, we have
    \begin{align*}
        & V_h^k(s) - r(s, a) - PV_{h+1}^k(s, a)
        \\
        & \le \lambda_{\iota_k} \sigexp^2(h, s, a) + 2 \beta^k(s, a)
        \\
        & \qquad  + \frac{1}{4\Vmax} \left( - (V_h^k(s) - V_h^*(s) + \varepsilon)^2 + P (V_{h+1}^k - V_{h+1}^* + \varepsilon)^2(s, a) \right) 
    \end{align*}
\end{lemma}

\begin{proof}
    Starting from the definition, we have
    \begin{align}
        & V_h^k(s) - r(s, a) - PV_{h+1}^k(s, a)
        \nonumber
        \\
        & = \left( (\hat{r}^k(s, a) + \hat{P}^k V_{h+1}^k(s, a) )_+ + b^k(s, a) \right)\land \Vmax - r(s, a) - PV_{h+1}^k(s, a)
        \nonumber
        \\
        & \le (\hat{r}^k(s, a) + \hat{P}^k V_{h+1}^k(s, a) )_+\land \Vmax - r(s, a) - PV_{h+1}^k(s, a) + b^k(s, a)
        \nonumber
        \\
        & \le (\underbrace{\hat{r}^k(s, a) + \hat{P}^k V_{h+1}^k(s, a) - r(s, a) - PV_{h+1}^k(s, a)}_{I_1})_+ \land \Vmax + b^k(s, a)
        \label{eq:est true diff 1}
        \, ,
    \end{align}
    where we use that $r(s, a) + PV_{h+1}^k(s, a) \ge 0$ for the last inequality.
    We focus on bounding $I_1$.
    By adding and subtracting $(\hat{P}^k - P)V_{h+1}^*(s, a)$, we obtain that
    \begin{align*}
        I_1& = \underbrace{\hat{r}^k(s, a) + \hat{P}^kV_{h+1}^*(s, a) - r(s, a) - PV_{h+1}^*(s, a)}_{I_2} + \underbrace{(\hat{P}^k - P)( V_{h+1}^k(s, a) - V_{h+1}^*)(s, a) }_{I_3}
        \, .
    \end{align*}
    Under $\Ecal_3(\frac{\delta}{4})$, using \lemmaref{lma:concentration for TV} with $\rho = 8$ and $c = \Vmax$, we bound $I_3$ as follows:
    \begin{align*}
        I_3 & \le \frac{1}{8\Vmax}\Var(V_{h+1}^k - V_{h+1}^*)(s, a) + \frac{9 \Vmax S \ell_2(k, \delta)}{N^k(s, a)}
        \\
        & =: I_3'
        \, .
    \end{align*}
    Noting that $I_3' \ge 0$, we have
    \begin{align*}
        (I_1 \lor 0) \land \Vmax
        & = \left( \left( I_2 + I_3\right) \lor 0 \right) \land \Vmax
        \\
        & \le \left( \left( I_2 + I_3' \right) \lor 0 \right) \land \Vmax
        \\
        & \le (I_2 \lor 0)\land \Vmax + I_3'
        \\
        & = (I_2 \land \Vmax) \lor 0 + I_3'
        \, .
    \end{align*}
    $I_2 \land \Vmax$ is bounded under $\Ecal_1(\frac{\delta}{4})$ (\lemmaref{lma:sub-exponential concentration event}) as follows:
    \begin{align*}
        I_2 \land \Vmax & \le \frac{\lambda_{\iota_k} \sigexp^2(s, a, V_{h+1}^*)}{2} + \frac{\ell_1({\iota_k}, \delta)}{\lambda_{\iota_k} N^k(s, a)} \land \sigsub \sqrt{\frac{\ell_2(k, \delta)}{N^k(s, a)}}
        \\
        & \le \frac{\lambda_{\iota_k} \sigexp^2(s, a, V_{h+1}^*)}{2} + b^k(s, a)
        \, ,
    \end{align*}
    where we use $k > \kap(\delta)$ for the last inequality.
    Combining the bounds, we obtain that
    \begin{align*}
        (I_1 \lor 0) \land \Vmax
        & \le \frac{\lambda_{\iota_k} \sigexp^2(s, a, V_{h+1}^*)}{2} + b^k(s, a) + \frac{1}{8\Vmax} \Var(V_{h+1}^k - V_{h+1}^*)(s, a) + \frac{9 \Vmax S \ell_2(k, \delta)}{N^k(s, a)}
        \, .
    \end{align*}
    Plugging this bound into inequality~\eqref{eq:est true diff 1}, we derive that
    \begin{align}
         & V_h^k(s) - r(s, a) - PV_{h+1}^k(s, a)
         \nonumber
         \\
         & \le \frac{\lambda_{\iota_k} \sigexp^2(s, a, V_{h+1}^*)}{2} + \frac{1}{8\Vmax} \Var(V_{h+1}^k - V_{h+1}^*)(s, a)  + \beta^k(s, a)
         \label{eq:est true diff 1.5}
         \, ,
    \end{align}
    where we use that $\beta^k(s, a) = 2 b^k(s, a) + \frac{9 \Vmax S \ell_2(k, \delta)}{N^k(s, a)}$.
    We define $I_4 := (V_h^k(s) - r(s, a) - PV_{h+1}^k(s, a))_+$.
    Noting that the right-hand side of inequality~\eqref{eq:est true diff 1.5} is at least 0, we have
    \begin{align}
        I_4 \le \frac{\lambda_{\iota_k} \sigexp^2(s, a, V_{h+1}^*)}{2} + \frac{1}{8\Vmax} \Var(V_{h+1}^k - V_{h+1}^*)(s, a)  + \beta^k(s, a)
        \label{eq:est true diff 2}
    \end{align}
    Recall that by \lemmaref{lma:quasi optimism}, we have $V_{h+1}^k(s') - V_{h+1}^*(s') + \varepsilon \ge 0$ for all $s' \in \Scal$.
    In addition, we have $V_{h+1}^k(s') - V_{h+1}^*(s') + \varepsilon \le 2 \Vmax$.
    By \lemmaref{lma:variance decomposition}, we have
    \begin{align*}
        \Var(V_{h+1}^k - V_{h+1}^*)(s, a)
        & = \Var(V_{h+1}^k - V_{h+1}^* + \varepsilon)(s, a)
        \\
        & \le - (V_h^k(s) - V_h^*(s) + \varepsilon)^2 + P (V_{h+1}^k - V_{h+1}^k + \varepsilon)^2(s, a)
        \\
        & \qquad + 4 \Vmax \left( V_h^k(s) - V_h^*(s) + \varepsilon - P (V_{h+1}^k - V_{h+1}^k + \varepsilon)(s, a) \right)_+
        \\
        & =  - (V_h^k(s) - V_h^*(s) + \varepsilon)^2 + P (V_{h+1}^k - V_{h+1}^* + \varepsilon)^2(s, a)
        \\
        & \qquad + 4 \Vmax \left( V_h^k(s) - V_h^*(s) - P (V_{h+1}^k - V_{h+1}^*)(s, a) \right)_+
        \, .
    \end{align*}
    Then, denoting $a^* := \pi_h^*(s)$, we have
    \begin{align*}
        -V_h^*(s) + P V_{h+1}^*(s, a)
        & = - Q_h^*(s, a^*) + Q_h^*(s, a) - r(s, a)  \le - r(s, a)
        \, ,
    \end{align*}
    and hence
    \begin{align*}
        \left( V_h^k(s) - V_h^*(s) - P (V_{h+1}^k - V_{h+1}^*)(s, a)\right)_+
        & \le \left( V_h^k(s) - r(s, a) - PV_{h+1}^k(s, a) \right)_+
        \\
        & = I_4
        \, .
    \end{align*}
    Therefore, we have
    \begin{align*}
        & \Var(V_{h+1}^k - V_{h+1}^*)(s, a)
        \\
        & \le - (V_h^k(s) - V_h^*(s) + \varepsilon)^2 + P (V_{h+1}^k - V_{h+1}^* + \varepsilon)^2(s, a) + 4 \Vmax I_4
        \, .
    \end{align*}
    Plugging this bound into inequality~\eqref{eq:est true diff 2}, we obtain
    \begin{align*}
        I_4 & \le \frac{\lambda_{\iota_k} \sigexp^2(s, a, V_{h+1}^*)}{2} + \frac{1}{2} I_4 + \beta^k(s, a)
        \\
        & \qquad  + \frac{1}{8\Vmax} \left( - (V_h^k(s) - V_h^*(s) + \varepsilon)^2 + P (V_{h+1}^k - V_{h+1}^* + \varepsilon)^2(s, a) \right) 
        \, .
    \end{align*}
    Solving the inequality with respect to $I_4$, we conclude that
    \begin{align*}
        I_4 & \le \lambda_{\iota_k} \sigexp^2(s, a, V_{h+1}^*) + 2 \beta^k(s, a)
        \\
        & \qquad  + \frac{1}{4\Vmax} \left( - (V_h^k(s) - V_h^*(s) + \varepsilon)^2 + P (V_{h+1}^k - V_{h+1}^* + \varepsilon)^2(s, a) \right) 
        \, .
    \end{align*}
\end{proof}

\begin{proof}[Proof of \lemmaref{lma:est true diff bound}]
    We use backward induction on $h$ to prove the following stronger inequality:
    \begin{align*}
        V_h^k(s) - V_h^{\pi^k}(s)
        & \le 2\lambda_{\iota_k}W_h^{\pi^k}(s) + \frac{1}{4\Vmax}(\varepsilon^2 - (V_h^k(s) - V_h^*(s) + \varepsilon)^2) + U_h^k(s)
        \, .
    \end{align*}
    This is a stronger inequality since we have $W_h^{\pi^k}(s) \le \Qrange$ and $\frac{\varepsilon^2}{4\Vmax}\le \frac{\varepsilon}{4} \le \frac{\lambda_{\iota_k} \Qrange}{2}$, where we use $\varepsilon \le \Vmax$ for the first inequality, and $\varepsilon \le 2 \lambda_{\iota_k} \Qrange$ for the second inequality.
    The inequality we want to prove holds when $h = H+1$ as $0 \le 0$.
    Suppose the inequality holds for $h + 1$.
    Then, we have
    \begin{align*}
        & V_h^k(s) - V_h^{\pi^k}(s)
        \\
        & = V_h^k(s) - r(s, a) - PV_{h+1}^{\pi^k}(s, a)
        \\
        & \le \underbrace{V_h^k(s) - r(s, a) - PV_{h+1}^k(s, a)}_{I_1} + \underbrace{P(V_{h+1}^k - V_{h+1}^{\pi^k})(s, a)}_{I_2}
        \, .
    \end{align*}
    $I_1$ is bounded by \lemmaref{lma:est true diff diff bound}.
    $I_2$ is bounded by the induction hypothesis.
    Combining the two, we derive that
    \begin{align*}
        & V_h^k(s) - V_h^{\pi^k}(s)
        \\
        & \le \lambda_{\iota_k} \sigexp^2(h, s, a) + 2 \beta^k(s, a)  + \frac{1}{4\Vmax} \left( - (V_h^k(s) - V_h^*(s) + \varepsilon)^2 + P (V_{h+1}^k - V_{h+1}^* + \varepsilon)^2(s, a) \right)
        \\
        & \qquad + P \left( 2\lambda_{\iota_k}W_{h+1}^{\pi^k} + \frac{1}{4\Vmax}(\varepsilon^2 - (V_{h+1}^k - V_{h+1}^* + \varepsilon)^2) + U_{h+1}^k \right)(s, a)
        \\
        & = 2 \lambda_{\iota_k} W_h^k(s) + \frac{1}{4\Vmax} \left( \varepsilon^2 - (V_h^k(s) - V_h^*(s)+ \varepsilon)^2 \right) + 2 \beta^k(s, a) + PU_{h+1}^k(s, a)
        \, .
    \end{align*}
    We need a proper clipping for the $2 \beta^k(s, a) + PU_{h+1}^k(s, a)$ term to bound it with $U_h^k(s)$.
    We note that $2 \lambda_{\iota_k} W_h^k(s) \ge 0$ and $\frac{1}{4\Vmax}(\varepsilon^2 - (V_h^k(s) - V_h^*(s)+ \varepsilon)^2) \ge - \frac{1}{4\Vmax}(V_h^k(s) - V_h^*(s)+ \varepsilon)^2 \ge - \frac{1}{4\Vmax} (2 \Vmax)^2 = -\Vmax$.
    Therefore, we have that
    \begin{align*}
        & V_h^k(s) - V_h^{\pi^k}(s)
        \\
        & = ( V_h^k(s) - V_h^{\pi^k}(s) ) \land \Vmax
        \\
        & \le \left( 2 \lambda_{\iota_k} W_h^k(s) + \frac{1}{4\Vmax} \left( \varepsilon^2 - (V_h^k(s) - V_h^*(s)+ \varepsilon)^2 \right) + 2 \beta^k(s, a) + PU_{h+1}^k(s, a) \right) \land \Vmax
        \\
        & \le  2 \lambda_{\iota_k} W_h^k(s) + \frac{1}{4\Vmax} \left( \varepsilon^2 - (V_h^k(s) - V_h^*(s)+ \varepsilon)^2 \right) \lor ( - \Vmax) + ( 2 \beta^k(s, a) + PU_{h+1}^k(s, a) ) \land 2\Vmax
        \\
        & = 2 \lambda_{\iota_k} W_h^k(s) + \frac{1}{4\Vmax} \left( \varepsilon^2 - (V_h^k(s) - V_h^*(s)+ \varepsilon)^2 \right) + U_h^k(s)
        \, ,.
    \end{align*}
    where we use the third property in Fact~\ref{fact:max and min}.
    The induction step is proved, so the proof is complete.
\end{proof}

\subsection{Proof of Theorem~\ref{thm:gap-independent bound}}
\label{appx:proof of gap-independent bound}

In this section, we restate Theorem~\ref{thm:gap-independent bound} with specific logarithmic factors and provide its proof.

\begin{theorem}[Formal statement of Theorem~\ref{thm:gap-independent bound}]
\label{thm:gap-independent bound formal}
    Let $\{c_{1, k}\}_{k=1}^\infty$ and $\{c_{2, k}\}_{k=1}^\infty$ be positive non-decreasing sequences.
    Fix any $\delta \in (0, 1]$ and $K \in \NN$.
    The regret bound of $\Alg$ satisfies
    \begin{align*}
        \Reg_{\Mcal}^{\Alg} (K, \delta) & \le 18 \Qrange \sum_{k=1}^K \frac{\ell_1(\iota_k, \delta)}{c_{1, k}} + ( 16 c_{1, k} SA \log KH)) \land \left( 16 \sqrt{2} c_{2, k} \sqrt{HSAK} \right)
        \\
        &\qquad + \Vmax (\kap(\delta) \land K) + 72 \Vmax S^2 A \ell_{2}(K, \delta) \log 2KH
        \, .
    \end{align*}
\end{theorem}

\begin{proof}[Proof of Theorem~\ref{thm:gap-independent bound formal}]
    By Lemma~\ref{lma:episode regret}, we have that for $k > \kap(\delta)$,
    \begin{align*}
        V_1^*(s_1^k) - V_1^{\pi^k}(s_1^k) \le \frac{18 \Qrange \ell_1(\iota_k, \delta)}{c_{1, k}} + U_1^k(s_1^k)
        \, .
    \end{align*}
    We bound the instantaneous regret of the first $\kap(\delta)$ episodes by $\Vmax$, and use Lemma~\ref{lma:episode regret} for the following episodes.
    Then, we obtain that under $\Ecal_1(\frac{\delta}{4}) \cap \Ecal_2(\frac{\delta}{4}) \cap \Ecal_3(\frac{\delta}{4})$, it holds that
    \begin{align*}
        \Reg_M^{\Alg}(K) \le \Vmax (\kap(\delta) \land K) + 18 \Qrange \sum_{k=\kap(\delta) + 1}^K \frac{\ell_1(\iota_k, \delta)}{c_{1, k}} + \sum_{k=\kap(\delta)+1}^K U_1^k(s_1^k)
        \, .
    \end{align*}
    By applying Lemma~\ref{lma:sum of U 1}, which bounds the sum of $U_1^k(s_1^k)$, we obtain that under an additional event of $\Ecal_4(\{2 \beta^k(s_h^k, a_h^k)\}_{k, h}, 2 \Vmax,  \frac{\delta}{4})$, we have
    \begin{align*}
        \Reg_M^{\Alg}(K) & \le 18 \Qrange \sum_{k=\kap(\delta) + 1}^K \frac{\ell_1(\iota_k, \delta)}{c_{1, k}} + ( 16 c_{1, K} SA \log KH)) \land \left( 16 \sqrt{2} c_{2, K} \sqrt{HSAK} \right)
        \\
        &\qquad + \Vmax (\kap(\delta) \land K) + 72 \Vmax S^2 A \ell_{2}(K, \delta) \log 2KH
        \, .
    \end{align*}
    Taking the union bound over the events, we conclude that the bound holds with probability at least $1 - \delta$.
    Also, note that the right-hand side is no longer a random variable.
    Therefore, this bound is an upper bound for $\Reg_M^{\Alg}(K, \delta)$.
\end{proof}

The following lemma is an analogue of Lemma 4 in~\citet{lee2025minimax}.
\begin{lemma}
\label{lma:sum of U 1}
    Under $\Ecal_4(\{2 \beta^k(s_h^k, a_h^k)\}_{k, h}, 2\Vmax, \frac{\delta}{4})$, the following inequality holds for all $K \in \NN$.
    \begin{align*}
        \sum_{k=1}^K U_1^k(s_1^k)&  \le ( 16 c_{1, K} SA \log KH)) \land \left( 16 \sqrt{2} c_{2, K} \sqrt{HSAK} \right)
        \\
        &\qquad + 72 \Vmax S^2 A \ell_{2}(K, \delta) \log 2KH
        \, .
    \end{align*}
\end{lemma}

\begin{proof}
    By \lemmaref{lma:sum of J with eta}, under $\Ecal_4(\{2 \beta^k(s_h^k, a_h^k)\}_{k, h}, 2 \Vmax, \frac{\delta}{4})$, it holds that
    \begin{align*}
        \sum_{k=1}^K U_1^k(s_1^k) & \le  4 \sum_{k=1}^K \sum_{h=1}^{\eta^k - 1} \beta^k(s_h^k, a_h^k) + 12 \Vmax SA \log \frac{16 H}{\delta}
        \, .
    \end{align*}
    Recall that
    \begin{align*}
        \beta^k(s, a)
        & = 2 b^k(s, a) + \frac{9 \Vmax S \ell_2(k, \delta)}{N^k(s, a)}
        \\
        & = \frac{2 c_{1, k}}{N^k(s, a)} \land \frac{2 c_{2, k}}{\sqrt{N^k(s, a)}} + \frac{9\Vmax S \ell_2(k, \delta)}{N^k(s, a)}
        \, .
    \end{align*}
    Furthermore, using that $N^k(s_h^k, a_h^k) \ge \frac{1}{2}N_h^k(s_h^k, a_h^k)$ when $h < \eta^k$, where $N_h^k(s, a)$ and $\eta^k$ are defined in Appendix~\ref{appx:notations and definitions}, we have
    \begin{align}
        \beta^k(s_h^k, a_h^k) \ind \{ h < \eta^k\}
        & \le \frac{4 c_{1, k}}{N_h^k(s, a)} \land \frac{2 \sqrt{2} c_{2, k}}{\sqrt{N_h^k(s, a)}} + \frac{18\Vmax S \ell_2(k, \delta)}{N_h^k(s, a)}
        \label{eq:beta eta}
        \, .
    \end{align}
    When $h < \eta^k$, we have $N^k(s_h^k, a_h^k) + 1 \le N_h^k(s_h^k, a_h^k) \le 2 N^k(s_h^k, a_h^k)$, which implies $N_h^k(s_h^k, a_h^k) \ge 2$ for such $h$.
    Now, we take the sum over $(h, k) \in [H] \times [K]$.
    Note that one has
    \begin{align*}
        &\sum_{n=2}^N \frac{1}{n} \le \log (1 \lor N)\, , & \sum_{n=2}^N \frac{1}{\sqrt{n}} \le 2 \sqrt{N} \, .
    \end{align*}
    Hence, we have
    \begin{align*}
        \sum_{k=1}^K \sum_{h=1}^H \frac{\ind\{ h < \eta^k\}}{N_h^k(s_h^k, a_h^k)} \le\sum_{(s, a) \in \Scal \times \Acal} \sum_{n=2}^{N^{K+1}(s, a)} \frac{1}{n} \le \sum_{(s, a) \in \Scal \times \Acal} \log (1 \lor N^{K+1}(s, a))\le SA \log KH
    \end{align*}
    and
    \begin{align*}
        \sum_{k=1}^K \sum_{h=1}^H \frac{\ind\{ h < \eta^k\}}{\sqrt{N_h^k(s_h^k, a_h^k)}}
        & \le\sum_{(s, a) \in \Scal \times \Acal} \sum_{n=2}^{N^{K+1}(s, a)} \frac{1}{\sqrt{n}}
        \\
        & \le \sum_{(s, a) \in \Scal \times \Acal} 2 \sqrt{N^{K+1}(s, a)}
        \\
        & \le 2 \sqrt{SA \sum_{(s, a) \in \Scal \times \Acal} N^{K+1}(s, a)}
        \\
        & = 2 \sqrt{HSAK}
        \, ,
    \end{align*}
    where the third inequality uses the Cauchy-Schwarz inequality.
    Therefore, the sum of $\beta^k(s, a) \ind\{ h < \eta^k\}$ is bounded as
    \begin{align*}
        & \sum_{k=1}^K \sum_{h=1}^H \beta^k(s, a) \ind\{h< \eta^k\}
        \\
        & \le \sum_{k=1}^K \sum_{h=1}^H\left( \frac{4 c_{1, k}}{N_h^k(s_h^k, a_h^k)} \land \frac{2\sqrt{2} c_{2, k}}{\sqrt{N_h^k(s_h^k, a_h^k)}} + \frac{18 \Vmax S \ell_2(k, \delta)}{N_h^k(s_h^k, a_h^k)} \right) \ind\{ h < \eta^k\}
        \\
        & \le \left( 4 c_{1, K} \sum_{k=1}^K \sum_{h=1}^H \frac{\ind\{h< \eta^k\}}{N_h^k(s_h^k, a_h^k)} \right) \land \left(2 \sqrt{2} c_{2, K} \sum_{k=1}^K \sum_{h=1}^H \frac{\ind\{ h < \eta^k\}}{\sqrt{N_h^k(s_h^k, a_h^k)}} \right) 
        \\
        & \qquad + 18\Vmax S \ell_2(K, \delta)  \sum_{k=1}^K \sum_{h=1}^H \frac{\ind\{h< \eta^k\}}{N_h^k(s_h^k, a_h^k)}
        \\
        & \le \left( 4 c_{1, K} SA \log KH \right) \land \left(4 \sqrt{2} c_{2, K} \sqrt{HSAK} \right) + 18 \Vmax S^2 A \ell_2(K, \delta) \log KH
        \, ,
    \end{align*}
    where the second inequality uses that $\{c_{1, k}\}_{k=1}^\infty$ and $\{ c_{2, k} \}_{k=1}^\infty$ are non-decreasing and $\sum_{i=1}^n a_i \land b_i \le (\sum_{i=1}^n a_i) \land (\sum_{i=1}^n b_i)$ for any sequences $\{a_i\}_{i=1}^n$ and $\{b_i\}_{i=1}^n$.
    Finally, we obtain the following bound on the sum $\sum_{k=1}^K U_1^k(s_1^k)$:
    \begin{align*}
        \sum_{k=1}^K U_1^k(s_1^k)
        & \le  4 \sum_{k=1}^K \sum_{h=1}^{\eta^k - 1} \beta^k(s_h^k, a_h^k) + 12 \Vmax SA \log \frac{16 H}{\delta}
        \\
        & \le \left( 16 c_{1, K} SA \log KH \right) \land \left(16 \sqrt{2} c_{2, K} \sqrt{HSAK} \right)
        \\
        & \qquad + 72 \Vmax S^2 A \ell_2(K, \delta) \log KH + 12 \Vmax SA \log \frac{16 H}{\delta}
        \, .
    \end{align*}
    To obtain a simpler form, we bound the last term by $12 \Vmax SA \log \frac{16 H}{\delta} \le 12 \Vmax S^2 A \ell_2(K, \delta) \le 72(\log 2) \Vmax S^2 A \ell_2(K, \delta)$, which bounds the last two terms as
    \begin{align*}
        72 \Vmax S^2 A \ell_2(K, \delta) \log KH + 12 \Vmax SA \log \frac{16 H}{\delta}
        & \le 72 \Vmax S^2 A \ell_2(K, \delta) \log 2 KH
        \, .
    \end{align*}
\end{proof}

\subsection{Proof for Theorem~\ref{thm:gap-dependent bound}}

\label{appx:proof of gap-dependent bound}

In this section, we restate Theorem~\ref{thm:gap-dependent bound} with specific logarithmic factors and provide its proof.

\citet{dann2021beyond} show that the true gap $\gap_h(s, a) := V_h^*(s) - Q_h^*(s, a)$ can be replaced with another function, which may improve the overall regret bound depending on reachability conditions of the MDP.
We also use the following generalized definition, which comes from the condition of Proposition 3.3 in~\citet{dann2021beyond}.
First, we define a stopping time $B = \min( \{ h \in \{1, 2, \ldots, H\} : \gap_{h}(s_{h}, a_{h}) > 0 \} \cup \{H+1\})$, which is the first time step that the policy takes a sub-optimal action.
We denote the $B$ value at the $k$-th episode by $B^k$.
In addition, we define $\Scal_{\mathrm{initial}}$ as the set of possible initial states, that is, the environment choose $s_1^k$ from $\Scal_{\mathrm{initial}}$.

\begin{definition}[Effective gap]
    A sequence of functions $\rgap_h : \Scal \times \Acal \rightarrow \RR_{> 0}$ for $h \in [H]$ is called \textbf{effective gap} if it satisfies
    \begin{align*}
        \Expec_{\pi(\cdot \mid s_1 = s)} \left[ \sum_{h=B}^H \rgap_{h}(s_{h}, a_{h}) \right] \le V_1^*(s) - V_1^{\pi}(s)
    \end{align*}
    for all $s \in \Scal_{\mathrm{initial}}$, and $\pi \in \Pi$.
\end{definition}

There can be many different functions that are effective gaps, and we note that it is better to set the effective gaps as large as possible.
A common choice of effective gap is $\rgap_h(s, a) := \frac{1}{2} \gap_h(s, a) \lor \frac{1}{2H} \gap_{\min}$, where $\gap_{\min} := \min\{ \gap_h(s, a) \mid \gap_h(s, a) > 0, (h, s, a) \in [H] \times \Scal \times \Acal \}$.
\citet{dann2021beyond} provide another choice of effective gap named \emph{return gap}, defined as 
\begin{align*}
    \rgap_h(s, a) := \frac{1}{2} \gap_h(s, a) \lor \min_{\substack{\pi \in \Pi \\ \PP_{\pi}((s_h, a_h) = (s, a)) > 0 \\ \EE_{\pi} [\sum_{h' = 1}^h \gap_{h'}(s_{h'}, a_{h'}) \mid (s_h, a_h) = (s, a) ] > 0} } \frac{1}{2H} \Expec_{\pi} \left[ \sum_{h' = 1}^h \gap_{h'}(s_{h'}, a_{h'}) \middle | (s_h, a_h) = (s, a)  \right]
    \, .
\end{align*}
It is possible to define effective gap to even depend on the policy of the agent, but we do not go further in this direction.
Refer to~\citet{dann2021beyond} for more discussions.
We define $\rgap(s, a) := \min_{h \in [H]} \rgap_h(s, a)$.

In addition we let $\vgap := \min_{\pi \in \Pi} \gap_B(s_B, a_B)$.
This is the minimum regret the agent can incur by making one suboptimal selection.
While this may be equal to $\gap_{\min}$ is most cases, it may be larger when state-action pairs with gap $\gap_h(s, a) = \gap_{\min}$ can be reached only if the agent takes another sub-optimal action before time step $h$.
We note that $\vgap \ge \gap_{\min}$.
We replace $\kapgap(\gap_{\min}, \delta)$ in Theorem~\ref{thm:gap-dependent bound} with $\kapgap(\vgap, \delta)$, leading to a potential improvement.

We restate Theorem~\ref{thm:gap-dependent bound} using the effective gap and $\vgap$, which leads to a more concise and potentially improved result.

\begin{theorem}[Formal statement of Theorem~\ref{thm:gap-dependent bound}]
\label{thm:instance dependent bound formal}
    Let $\{c_{1, k}\}_{k=1}^\infty$ and $\{c_{2, k}\}_{k=1}^\infty$ be positive non-decreasing sequences.
    Fix any $\delta \in (0, 1]$ and $K \in \NN$.
    Then, we have
    \begin{align*}
        \Reg_{\Mcal}^{\Alg} (K, \delta) & \le \Vmax (\kap(\delta) \land K)  + \sum_{(s, a) \in \Scal \times \Acal}  \left( \frac{2048 c_{2, K}^2}{\rgap(s, a)} \land 64 c_{1, K} \log \frac{32 c_{1, K} }{\rgap(s, a)} \right)
        \\
        & \qquad + \sum_{k = \kap(\delta)+1}^{\kapgap(\vgap, \delta) \land K} \frac{144 \Qrange \ell_1(\iota_k, \delta)}{c_{1, k}} + 288 \Vmax S^2 A \ell_2(K, \delta) \log 2KH
    \end{align*}
\end{theorem}

\begin{proof}
    The proof has two main steps.
    First, we define a variant of the $U_h^k$ function.
    For $k \in \NN$, define an event $E^k := \{ \omega : V_{B^k}^*(s_{B^k}^k) - V_{B^k}^{\pi^k}(s_{B^k}^k) > \frac{36 \Qrange \ell_1(\iota_k, \delta)}{c_{1, k}} \}$.
    Define $\ddot{\beta}_h^k(s, a)$ in the following way:
    
    \begin{align*}
        \ddot{\beta}_h^k(s, a) = \begin{cases}
            \frac{72 \Qrange \ell_1(\iota_k, \delta)}{c_{1, k}} & (h = B^k, (E^k)^{\mathsf{c}})
            \\
            (8 \beta^k(s, a) - \rgap_h(s, a))_+ & (h \ge B^k, E^k)
            \\
            0 & (\text{otherwise})\, .
        \end{cases}
    \end{align*}
    In words, $\ddot{\beta}_h^k$ is $0$ until the agent makes an sub-optimal action.
    If a sub-optimal selection happens for the first time at time step $h$, where $h = B^k$ is implied by definition, denote the difference between the optimal value and the policy's value by $\textsf{Gap} := V_h^*(s_h^k) - V_h^{\pi^k}(s_h^k)$.
    If $\frac{1}{2}\textsf{Gap} < \frac{18\Qrange \ell_1(\iota_k, \delta)}{c_{1, k}}$, we set $\ddot{\beta}_h^k(s, a)$ to $\frac{72 \Qrange\ell_1(\iota_k, \delta)}{c_{1, k}}$, and all subsequent $\ddot{\beta}_{h'}^k$ in the same episode are set to $0$.
    Otherwise, the following sequence of $\ddot{\beta}_{h'}^k$ are set to $(8 \beta^k(s, a) - \rgap_{h'}(s, a))_+$ for $h' \ge h$.
    We define the clipped sum $\ddot{U}_h^k(s)$ of $\ddot{\beta}_h^k$ iteratively starting from $\ddot{U}_{H+1}^k(s) := 0$ and
    \begin{align*}
        \ddot{U}_h^k(s) := ( \ddot{\beta}_h^k(s, a) + P \ddot{U}_{h+1}^k(s, a)) \land 8 \Vmax
        \, .
    \end{align*}

    Then, we have the following lemma:

    \begin{lemma}
    \label{lma:bound by ddotU}
        Assume the event of $\Ecal_1(\frac{\delta}{4}) \cap \Ecal_2(\frac{\delta}{4}) \cap \Ecal_3(\frac{\delta}{4})$ and assume that $k > \kap(\delta)$.
        Then, for all $s \in \Scal$, it holds that
        \begin{align*}
            V_1^*(s) - V_1^{\pi^k}(s) \le \ddot{U}_1^k(s)
            \, .
        \end{align*}
    \end{lemma}
    The proof of Lemma~\ref{lma:bound by ddotU} is deferred to Appendix~\ref{appx:proof of bound by ddotU}.
\\
    By \lemmaref{lma:bound by ddotU}, we have
    \begin{align*}
        \Reg_{\Mcal}^{\Alg}(K, \delta) \le \Vmax\kap(\delta) + \sum_{k=\kap(\delta)+1}^K \ddot{U}_1^k(s_1^k)
        \, .
    \end{align*}

    Then, we bound the sum of $\ddot{U}_1^k(s_1^k)$ using the following lemma.

    \begin{lemma}
    \label{lma:sum of ddotU}
        Under the event of $\Ecal_4(\{\ddot{\beta}_h^k\}_{k, h}, 8 \Vmax, \frac{\delta}{4})$, we have that
        \begin{align*}
            \sum_{k=\kap(\delta) + 1}^K \ddot{U}_1^k(s_1^k)
            & \le \sum_{(s, a) \in \Scal \times \Acal}  \left( \frac{2048 c_{2, K}^2}{\rgap(s, a)} \land 64 c_{1, K} \log \frac{32 c_{1, K} }{\rgap(s, a)} \right)
            \\ & \qquad + \sum_{k=\kap(\delta)+1}^{\kapgap(\vgap, \delta) \land K} \frac{144 \Qrange \ell_1(\iota_k, \delta)}{c_{1, k}}  + 288 \Vmax S^2 A \ell_2(K, \delta) \log 2KH
            \, .
        \end{align*}
    \end{lemma}
    The proof of Lemma~\ref{lma:sum of ddotU} is deferred to Appendix~\ref{appx:proof of sum of ddotU}.
    The theorem is proved by combining the two lemmas.
\end{proof}

\subsubsection{Proof of Lemma~\ref{lma:bound by ddotU}}
\label{appx:proof of bound by ddotU}

We first prove the following two simple lemmas.

\begin{lemma}
\label{lma:opt est value diff event E}
    For $h \in [H]$ and $s \in \Scal$, let $E_h^k(s) := \{ \omega \in \Omega : V_h^*(s) - V_h^{\pi^k}(s) \ge \frac{36 \Qrange \ell_1 (\iota_k, \delta)}{c_{1, k}} \}$.
    Then, under the event of Lemma~\ref{lma:episode regret}, we have
    \begin{align*}
        V_h^*(s) - V_h^{\pi^k}(s) \le \ind\{ E_h^k(s) \} \cdot 2 U_h^k(s) + \ind\{ {E_h^k(s)}^{\mathsf{C}} \} \left(  \frac{36 \Qrange \ell_1(\iota_k, \delta)}{c_{1, k}}  \land 4 \Vmax \right)
        \, .
    \end{align*}
\end{lemma}
\begin{proof}
    By \lemmaref{lma:episode regret}, we have
    \begin{align*}
        V_h^*(s) - V_h^{\pi^k}(s) \le U_h^k(s) + \frac{18 \Qrange \ell_1(\iota_k, \delta)}{c_{1, k}} 
        \, .
    \end{align*}
    The lemma we have to prove is stating that
    \begin{align*}
        V_h^*(s) - V_h^{\pi^k}(s) \le \begin{cases}
            2 U_h^k(s) & (V_h^*(s) - V_h^{\pi^k}(s) \ge \frac{36 \Qrange\ell_1(\iota_k, \delta)}{c_{1, k}})
            \\
            \frac{36 \Qrange \ell_1(\iota_k, \delta)}{c_{1, k}}  \land 4 \Vmax  & (V_h^*(s) - V_h^{\pi^k}(s) < \frac{36 \Qrange\ell_1(\iota_k, \delta)}{c_{1, k}})
        \end{cases}
        \, .
    \end{align*}
    The first case is true since otherwise we would have $V_h^*(s) - V_h^{\pi^k} >\frac{1}{2}(2 U_h^k(s) + \frac{36 \Qrange\ell_1(\iota_k, \delta)}{c_{1, k}}) = U_h^k(s) + \frac{18\Qrange\ell_1(\iota_k, \delta)}{c_{1, k}}$, which is a contradiction.
    The latter case is trivial.
\end{proof}

\begin{lemma}
\label{lma:dotU}
    Let $\dot{\beta}_h^k(s, a) = (8 \beta^k(s, a) - \rgap_h(s, a))_+$ and $\dot{U}_h^k(s)$ be defined iteratively starting from $\dot{U}_{H+1}^k(s) := 0$ and
    \begin{align*}
        \dot{U}_h^k(s) := (\dot{\beta}_h^k(s, a) + P \dot{U}_{h+1}(s, a)) \land 8 \Vmax
        \, .
    \end{align*}
    Then, for all $h\in [H]$ and $s \in \Scal$, we have
    \begin{align*}
        4U_h^k(s) - \Expec_{\pi^k(\cdot \mid s_h = s)} \left[ \sum_{h' = h}^H \rgap_{h'}(s_{h'}, a_{h'}) \right] \le \dot{U}_h^k(s)
        \, .
    \end{align*}
\end{lemma}
\begin{proof}
    We use backward induction on $h$.
    The inequality is trivial for $h = H+1$ as $0 \le 0$.
    Suppose the inequality holds for $h+1$,
    Then, denoting $a := \pi_h^k(s)$, we have
    \begin{align*}
        & 4 U_h^k(s) - \Expec_{\pi^k(\cdot \mid s_h = s)} \left[ \sum_{h' = h}^H \rgap_{h'}(s_{h'}, a_{h'} ) \right]
        \\
        & \le 4(2 \beta^k(s, a) + P U_{h+1}^k(s, a))  - \Expec_{\pi^k(\cdot \mid s_h = s)} \left[ \sum_{h' = h}^H \rgap_{h'}(s_{h'}, a_{h'}) \right]
        \\
        & = (8\beta^k(s, a) - \rgap_h(s, a)) + \Expec_{\pi^k(\cdot \mid s_h = s)} \left[ 4 U_{h+1}^k(s_{h+1}) - \sum_{h' = h+1}^H \rgap_{h'}(s_{h'}, a_{h'}) \right]
        \\
        & \le (8\beta^k(s, a) - \rgap_h(s, a)) + \Expec_{\pi^k(\cdot \mid s_h = s)} \left[ \dot{U}_{h+1}^k(s_{h+1}) \right]
        \\
        & \le \dot{\beta}_h^k(s, a) + P \dot{U}_{h+1}^k(s, a)
        \, .
    \end{align*}
    In addition, since $U_h^k(s) \le 2\Vmax$, we have $4U_h^k(s) - \Expec_{\pi^k(\cdot \mid s_h = s)} \left[ \sum_{h' = h}^H \rgap_{h'}(s_{h'}, a_{h'}) \right] \le 8 \Vmax$.
    Combining these, we conclude that
    \begin{align*}
        4U_h^k(s) - \Expec_{\pi^k(\cdot \mid s_h = s)} \left[ \sum_{h' = h}^H \rgap_{h'}(s_{h'}, a_{h'}) \right] &\le (\dot{\beta}_h^k(s, a) + P \dot{U}_{h+1}^k(s, a)) \land 8 \Vmax = \dot{U}_h^k(s)
        \, .
    \end{align*}
\end{proof}

\begin{proof}[Proof of Lemma~\ref{lma:bound by ddotU}]
    For any policy $\pi$, We have
    \begin{align*}
        V_h^*(s) - V_h^{\pi}(s) & = Q_h^*(s, \pi_h^*(s)) - Q_h^{\pi}(s, \pi_h(s))
        \\
        & = Q_h^*(s, \pi_h^*(s)) - Q_h^*(s, \pi_h(s)) +Q_h^*(s, \pi_h(s)) - Q_h^{\pi}(s, \pi_h(s))
        \\
        & = \gap_h(s, \pi_h(s)) + P(V_{h+1}^* - V_{h+1}^{\pi})(s, \pi_h(s))
        \, .
    \end{align*}
    Since $\gap_h(s_h, a_h) = 0$ for $h < B$ by the definition of $B$, iteratively expanding this equation up to $h = B$ and using the optional stopping theorem yields
    \begin{align*}
        V_1^*(s) - V_1^{\pi}(s) = \Expec_{\pi(\cdot \mid s_1 =s)} \left[ V_{B}^*(s_h) - V_{B}^{\pi}(s_h) \right]
        \, .
    \end{align*}
    Lemma~\ref{lma:opt est value diff event E} states that for all $h \in [H]$ and $s \in \Scal$, we have
    \begin{align*}
        V_h^*(s) - V_h^{\pi^k}(s) \le \ind\{ E_h^k(s) \} \cdot 2 U_h^k(s) + \ind\{ {E_h^k(s)}^{\mathsf{C}} \}  \left(  \frac{36 \Qrange \ell_1(\iota_k, \delta)}{c_{1, k}}  \land 4 \Vmax \right)
        \, .
    \end{align*}
    Observe that $E^k$ can be regarded as $E^k_{B^k}( s_{B^k}^k)$, or more precisely, $E^k = \cup_{h, s} (E^k(h, s) \cap \{ h = B^k, s = s_{B^k}^k \} )$.
    Then,     
    \begin{align*}
        & V_1^*(s_1^k) - V_1^{\pi^k}(s_1^k) 
        \\
        & = \Expec_{\pi^k} \left[ V_{B^k}^*(s_h^k) - V_{B^k}^{\pi}(s_h^k) \right]
        \\
        & \le \Expec_{\pi^k} \left[  \left(  \frac{36 \Qrange \ell_1(\iota_k, \delta)}{c_{1, k}}  \land 4 \Vmax \right) \ind\{ (E^k)^{\mathsf{c}}, B^k \le H\} + 2 U_{B^k}^k(s_{B^k}^k) \ind\{ E^k \}  \right]
        \, .
    \end{align*}
    From the definition of effective gap, we have $\EE_{\pi^k} \left[ \sum_{h=B^k}^H \rgap_h(s_h^k, a_h^k) \right] \le V_1^*(s_1^k) - V_1^{\pi^k}(s_1^k)$.
    Then, it holds that
    \begin{align*}
        & V_1^*(s_1^k) - V_1^{\pi^k}(s_1^k)
        \\
        & = 2 (V_1^*(s_1^k) - V_1^{\pi^k}(s_1^k) )  - ( V_1^*(s_1^k) - V_1^{\pi^k}(s_1^k))
        \\
        & \le 2 \Expec_{\pi^k} \left[   \left(  \frac{36 \Qrange \ell_1(\iota_k, \delta)}{c_{1, k}}  \land 4 \Vmax \right) \ind\{ (E^k)^{\mathsf{c}}, B^k \le H\} + 2 U_{B^k}^k(s_{B^k}^k) \ind\{ E^k \} \right] - \EE_{\pi^k} \left[ \sum_{h=B^k}^H \rgap_h(s_h^k, a_h^k) \right]
        \\
        & =  \Expec_{\pi^k} \left[ \left(  \frac{72 \Qrange \ell_1(\iota_k, \delta)}{c_{1, k}}  \land 8 \Vmax \right)\ind\{ (E^k)^{\mathsf{c}}, B^k \le H\} + 4 U_{B^k}^k(s_{B^k}^k) \ind\{ E^k \} - \sum_{h=B^k}^H \rgap_h(s_h^k, a_h^k)  \right]
        \, .
    \end{align*}
    We bound the last two terms as
    \begin{align*}
        \Expec_{\pi^k} \left[ 4U_{B^k}^k(s_{B^k}^k) \ind\{ E^k\} - \sum_{h=B^k}^H \rgap_h(s_h^k, a_h^k)  \right]
        & \le \Expec_{\pi^k} \left[ \ind\{E^k \} \left( 4U_{B^k}^k(s_{B^k}^k) -  \sum_{h=B^k}^H \rgap_h(s_h^k, a_h^k)  \right) \right]
        \\
        & \le \Expec_{\pi^k} \left[ \ind\{ E^k \} \dot{U}_{B^k}^k(s_{B^k}^k) \right]
        \, ,
    \end{align*}
    where we use Lemma~\ref{lma:dotU} for the last inequality.
    Therefore, we obtain that
    \begin{align*}
        V_1^*(s_1^k) - V_1^{\pi^k}(s_1^k) \le \Expec_{\pi^k} \left[ \left(  \frac{72 \Qrange \ell_1(\iota_k, \delta)}{c_{1, k}}  \land 8 \Vmax \right) \ind\{ (E^k)^{\mathsf{c}}, B^k \le H\} + \ind\{ E^k \} \dot{U}_{B^k}^k(s_{B^k}^k)  \right]
        \, .
    \end{align*}
    The proof is completed by noting that the right-hand side is equivalent to $\ddot{U}_1^k(s_1^k)$.
\end{proof}

\subsubsection{Proof of Lemma~\ref{lma:sum of ddotU}}
\label{appx:proof of sum of ddotU}
\begin{proof}[Proof of Lemma~\ref{lma:sum of ddotU}]
    Under $\Ecal_4(\{\ddot{\beta}_h^k\}_{k, h}, 8 \Vmax, \delta)$ (\lemmaref{lma:sum of J with eta}), we have
    \begin{align}
        \sum_{k=\kap(\delta)+1}^K \ddot{U}_1^k(s_1^k) \le 2 \sum_{k=\kap(\delta) + 1}^K \sum_{h=1}^{\eta^k - 1} \ddot{\beta}_h^k(s_h^k, a_h^k) + 48 \Vmax SA \log \frac{16H}{\delta}
        \label{eq:sum of ddotU 1}
        \, .
    \end{align}
    Note that $\sum_{h=1}^{\eta^k - 1} \ddot{\beta}_h^k(s_h^k, a_h^k)$ is either $\frac{72 \Qrange \ell_1(\iota_k \delta)}{c_{1, k}}$ or $\sum_{h = B^k}^{\eta^k - 1} ( 8 \beta^k(s_h^k, a_h^k) - \rgap_h(s_h^k, a_h^k))_+$.
    For large enough $c_{1, k}$, specifically if $\vgap > \frac{36 \Qrange \ell_1(\iota_k, \delta)}{c_{1, k}}$, then $E^k$ always occur.
    Therefore, we bound the sum of $\ddot{\beta}_h^k(s_h^k, a_h^k)$ as
    \begin{align}
        \sum_{h=1}^{\eta^k - 1} \ddot{\beta}_h^k(s_h^k, a_h^k) & \le \ind \left\{ \vgap \le \frac{36 \Qrange \ell_1(\iota_k, \delta)}{c_{1, k}} \right\} \frac{72 \Qrange \ell_1(\iota_k, \delta)}{c_{1, k}} 
        \nonumber
        \\
        & \qquad + \sum_{h=B^k}^{\eta^k - 1} (8 \beta^k(s_h^k, a_h^k) - \rgap_h(s_h^k, a_h^k))_+
        \label{eq:sum of ddotU 2}
    \end{align}
    By the definition of $\kapgap(\vgap, \delta)$, we have
    \begin{align}
        \sum_{k=\kap(\delta) + 1}^K \ind \left\{ \vgap \le \frac{36 \Qrange \ell_1(\iota_k, \delta)}{c_{1, k}} \right\} \frac{72 \Qrange \ell_1(\iota_k, \delta)}{c_{1, k}} \le \sum_{k = \kap(\delta) + 1}^{\kapgap(\vgap, \delta) \land K} \frac{72 \Qrange \ell_1(\iota_k, \delta)}{c_{1, k}}
        \label{eq:sum of ddotU 3}
        \, .
    \end{align}
    Now, we rearrange the sum of $(8\beta^k(s_h^k, a_h^k) - \rgap_h(s_h^k, a_h^k))_+$ by grouping the same state-action pairs as follows:
    \begin{align*}
        & \sum_{k=\kap(\delta)+1}^K \sum_{h=B^k}^{\eta^k - 1} (8 \beta^k(s_h^k, a_h^k) - \rgap_h(s_h^k, a_h^k))_+
        \\
        & \le \sum_{(s, a) \in \Scal \times \Acal} \sum_{n=2}^{N^{K+1}(s, a)} \left(8 \left( \frac{4 c_{1, K}}{n} \land \frac{2\sqrt{2} c_{2, K}}{n} + \frac{18\Vmax S \ell_2(K, \delta)}{n} \right)- \rgap(s, a) \right) _+
        \, ,
    \end{align*}
    where we use Eq.~\eqref{eq:beta eta} to bound $\beta^k(s, a) \ind\{ h < \eta^k\}$.
    Focusing on one state-action pair, we have
    \begin{align}
        & \sum_{n=2}^{N^{K+1}(s, a)} \left(8 \left(  \frac{4 c_{1, K}}{n} \land \frac{2\sqrt{2}c_{2, K}}{\sqrt{n}} + \frac{18 \Vmax S \ell_2(K, \delta)}{n} \right) - \rgap(s, a) \right)_+
        \nonumber
        \\
        & \le \sum_{n=2}^{N^{K+1}(s, a)} \left( \frac{ 32 c_{1, K}}{n} \land \frac{ 16 \sqrt{2} c_{2, K}}{\sqrt{n}} - \rgap(s, a) \right)_+ + \sum_{n=2}^{N^{K+1}(s, a)} \frac{144 \Vmax S \ell_2(K, \delta)}{n}
        \label{eq:sum of ddot U one pair}
        \, .
    \end{align}
    Using properties of maximums and minimums, we have
    \begin{align*}
        & \sum_{n=2}^{N^{K+1}(s, a)} \left( \frac{32 c_{1, K}}{n} \land \frac{ 16 \sqrt{2} c_{2, K}}{\sqrt{n}} - \rgap(s, a) \right)_+ 
        \\
        & = \sum_{n=2}^{N^{K+1}(s, a)} \left( \left(\frac{32 c_{1, K}}{n} -\rgap \right) \land \left( \frac{ 16 \sqrt{2} c_{2, K}}{\sqrt{n}} - \rgap(s, a)\right) \right)_+ 
        \\
        & = \sum_{n=2}^{N^{K+1}(s, a)} \left(\frac{32 c_{1, K}}{n} -\rgap \right)_+ \land \left( \frac{ 16 \sqrt{2} c_{2, K}}{\sqrt{n}} - \rgap(s, a)\right)_+
        \\
        & \le \left( \sum_{n=2}^{N^{K+1}(s, a)} \left( \frac{32 c_{1, K}}{n}  - \rgap(s, a) \right)_+ \right) \land \left(\sum_{n=2}^{N^{K+1}(s, a)} \left( \frac{ 16 \sqrt{2} c_{2, K}}{\sqrt{n}} - \rgap(s, a) \right)_+  \right)
        \, .
    \end{align*}
    Using Lemma~\ref{lma:sum and integral clip}, the first sum is bounded as
    \begin{align*}
        \sum_{n=2}^{N^{K+1}(s, a)} \left( \frac{32 c_{1, K}}{n}  - \rgap(s, a) \right)_+ \le 32 c_{1, K} \log \frac{32 c_{1, K}}{\rgap(s, a)}
        \, ,
    \end{align*}
    and the second sum is bounded as
    \begin{align*}
        \sum_{n=2}^{N^{K+1}(s, a)} \left( \frac{ 16 \sqrt{2} c_{2, K}}{\sqrt{n}} - \rgap(s, a) \right)_+ 
        & \le \frac{1024 c_{2, K}^2}{\rgap(s, a)}
        \, .
    \end{align*}
    Now, we bound the remaining sum in Eq.~\eqref{eq:sum of ddot U one pair} as follows:
    \begin{align*}
        \sum_{n=2}^{N^{K+1}(s, a)} \frac{144 \Vmax S \ell_2(K, \delta)}{n}
         \le 144 \Vmax S \ell_2(K, \delta) \log KH
         \, .
    \end{align*}
    Plugging these bounds back to Eq.~\eqref{eq:sum of ddot U one pair}, we obtain that the sum of $(8\beta^k(s_h^k, a_h^k) - \rgap_h(s, a))_+$ for one state-action pair is bounded as
    \begin{align*}
        \left( 32 c_{1, K} \log \frac{32 c_{1, K}}{\rgap(s, a)} \right) \land \left(\frac{1024 c_{2, K}^2}{\rgap(s, a)} \right) + 144 \Vmax S \ell_2(K, \delta) \log KH
        \, ,
    \end{align*}
    and taking the sum over all $(s, a) \in \Scal \times \Acal$, we obtain that
    \begin{align}
        & \sum_{k = \kap(\delta) + 1}^K \sum_{h=B^k}^{\eta^k-1} (8 \beta^k(s_h^k, a_h^k) - \rgap(s_h^k, a_h^k))_+
        \nonumber
        \\
        & \le \sum_{(s, a) \in \Scal \times \Acal} \left( 32 c_{1, K} \log \frac{32 c_{1, K}}{\rgap(s, a)}\right) \land \left( \frac{1024 c_{2, K}^2}{\rgap(s, a)}\right) + 144 \Vmax S^2 A \ell_2(K, \delta) \log KH
        \label{eq:sum of ddotU 4}
        \, .
    \end{align}
    Combining Eq.~\eqref{eq:sum of ddotU 1}, Eq.~\eqref{eq:sum of ddotU 2}, Eq.~\eqref{eq:sum of ddotU 3}, and Eq.~\eqref{eq:sum of ddotU 4}, we obtain that
    \begin{align*}
        & \sum_{k=\kap(\delta)+1}^{K} \ddot{U}_1^k(s_1^k)
        \\
        & \le \sum_{k=\kap(\delta)+1}^{\kapgap(\vgap, \delta) \land K} \frac{144 \Qrange \ell_1(\iota_k, \delta)}{c_{1, k}} 
        \\
        & \qquad + \sum_{(s, a) \in \Scal \times \Acal} \left( 64 c_{1, K} \log \frac{32 c_{1, K}}{\rgap(s, a)}\right) \land \left( \frac{2048 c_{2, K}^2}{\rgap(s, a)}\right) 
        \\
        & \qquad + 288 \Vmax S^2 A \ell_2(K, \delta) \log KH + 48 \Vmax SA \log \frac{16H}{\delta}
        \, .
    \end{align*}
    We bound the last term as $48 \Vmax SA \log \frac{16H}{\delta} \le 48 \Vmax S^2 A \ell_2(K, \delta) \le 288(\log 2) \Vmax S^2 A \ell_2(K, \delta)$, which leads to
    \begin{align*}
        288 \Vmax S^2 A \ell_2(K, \delta) \log KH + 48 \Vmax SA \log \frac{16H}{\delta}
        & \le 288 \Vmax S^2 A \ell_2(K, \delta)\log 2KH
        \, .
    \end{align*}
\end{proof}

\subsection{Possibility of Obtaining Better Second-order Terms}
We note that the ${\Ocal(\Vmax S^2 A \ell_2(K, \delta) \log KH)}$ terms in Theorems~\ref{thm:gap-independent bound formal} and~\ref{thm:instance dependent bound formal}, and consequently those in Corollaries~\ref{cor:c_1 only 1},~\ref{cor:eqo and hoeffding formal}, and~\ref{cor:simchi levi formal}, can be replaced with $\Ocal(\Vmax SA (S\log KH + \log \frac{HSA}{\delta}) \log KH)$ by substituting \lemmaref{lma:concentration for TV} with Lemmas~3 and~10 in \citet{menard2021fast}.
Specifically, a $\Ocal(S \log \frac{HSA\log K}{\delta})$ factor can be replaced by a $\Ocal(S\log KH + \log \frac{HSA}{\delta})$ factor.
This modification improves the coefficient of the $\log (1/\delta)$ term by a factor of $S$ at the cost of an additional $\Ocal(S \log KH)$ term.
While further improvement to $\Ocal(S \log \log KH + \log \frac{HSA}{\delta})$ may be possible, which would yield a strict improvement over our current results, we do not pursue optimizing the second-order terms in this work.

\section{Proof of Corollaries in Section~\ref{sec:corollaries}}
\label{appx:corollaries}
In this section, we prove Corollaries~\ref{cor:c_1 only 1}--\ref{cor:simchi levi}.

While plugging in the values of $c_{1, K}$ and $c_{2, K}$ is straightforward, bounding $\kap(\delta)$ is non-trivial.
First, we present a simple technique that circumvents the direct computation of $\kap(\delta)$ in certain cases.
For simplicity, we denote the condition in the definition of $\kap(\delta)$ by $I(k, \delta) :=\ind\{ c_{1, k} <  (2 \sqrt{13 \Qdiff} \lor 2 \Valpha)\ell_1(1, \delta) \text{ or } c_{2, k} < \left( \sigsub + \frac{6 \Qdiff \ell_1(\iota_k, \delta)}{c_{1, k}} \right) \sqrt{\ell_2(k, \delta)} \} $.
Then, we have

\begin{align*}
    & \Vmax (\kap(\delta) \land K) 
    \\
    & = \sum_{k=1}^{K} \Vmax I(k, \delta)
    \\
    & = \sum_{k=1}^{\kap(\delta)} \frac{18 \Qrange \ell_1(\iota_k, \delta)}{c_{1, k}} + \sum_{k=1}^K \left( \Vmax - \frac{18 \Qrange \ell_1(\iota_k, \delta)}{c_{1, k}} \right) I(k, \delta)
    \\
    & \le \sum_{k=1}^{\kap(\delta)} \frac{18 \Qrange \ell_1(\iota_k, \delta)}{c_{1, k}} + \sum_{k=1}^K \left( \Vmax - \frac{18 \Qrange \ell_1(\iota_k, \delta)}{c_{1, k}} \right)_+ I(k, \delta)
    \, .
\end{align*}
Now, we consider a new condition $\ind\{ \Vmax \ge \frac{18\Qrange\ell_1(\iota_k, \delta)}{c_{1, k}}\} I(k, \delta)$.
The resulting condition on $c_{1, k}$ can be written as $c_{1, k} < (2 \sqrt{13 \Qdiff} \lor 2 \Valpha) \ell_1(1, \delta)$ and $c_{1, k} \ge \frac{18\Qrange \ell_1(\iota_k, \delta)}{\Vmax}$.
When considering the worst-case regret, we must upper-bound $\Qrange$ by $2 \Vmax^2$, and this condition becomes impossible to satisfy since $2\sqrt{13\Qdiff} \lor 2\Valpha \le 11 \Vmax \le \frac{18\Qrange}{\Vmax}$.
The condition on $c_{2, k}$ can be written as $c_{2, k} < ( \sigsub + \frac{6 \Qdiff \ell_1(\iota_k, \delta)}{c_{1, k}})\sqrt{\ell_2(k, \delta)}$ and $c_{1, k} \ge \frac{18\Qrange \ell_1(\iota_k, \delta)}{\Vmax}$.
Using that $\Qdiff \le \Qrange$, these two conditions imply that $c_{2, k} < (\sigsub + \frac{\Vmax}{3})\sqrt{\ell_2(k, \delta)}$, which leads to a simpler calculation.
We summarize this discussion by the following lemma:

\begin{lemma}
\label{lma:kappa}
    Let
    \begin{align*}
        & \kappa_1(\delta) := \max \left\{ k \in \NN : \frac{18 \Qrange \ell_1(\iota_k, \delta)}{\Vmax} <  c_{1, k} <(2 \sqrt{13 \Qdiff} \lor 2 \Valpha) \ell_1(1, \delta) ) \right\}
        \\
        & \kappa_2(\delta) := \max \left\{ k \in \NN : c_{2, k} < (\sigsub + \frac{\Vmax}{3})\sqrt{\ell_2(k, \delta)} \right\}
        \, .
    \end{align*}
    Then, we have
    \begin{align*}
        \Vmax \kap(\delta) \le \sum_{k=1}^{\kap(\delta)} \frac{18 \Qrange \ell_1(\iota_k, \delta)}{c_{1, k}} + \Vmax( \kap_1(\delta) \lor \kap_2(\delta))
        \, . 
    \end{align*}
\end{lemma}

\subsection{Proof of Corollary~\ref{cor:c_1 only 1}}

\begin{proof}[Proof of Corollary~\ref{cor:c_1 only 1}]
    By Lemma~\ref{lma:kappa}, we upper bound $\Vmax \kap(\delta)$ by $\sum_{k=1}^{\kap(\delta)} \frac{18 \Vmax^2}{c_{1, k}}$, where plugging in $\Qrange = 2 \Vmax^2$ and $c_{2, k} = \infty$ yields $\kappa_1(\delta) = \kappa_2(\delta) = 0$.
    Then, Theorem~\ref{thm:gap-independent bound formal} implies that
    \begin{align*}
        \Reg_{\Mcal}^{\Alg}(K, \delta) \le \sum_{k=1}^{K} \frac{18 \Vmax^2 \ell_1(\iota_k, \delta)}{c_{1, k}} +  16 c_{1, K}SA \log KH + 72 \Vmax S^2 A \ell_2(K, \delta) \log 2KH
    \end{align*}
    We plug in $c_{1, k} = c_1 \Vmax \sqrt{\frac{K \ell_1(\iota_k, 1)}{SA \log kH}}$ and bound each term.
    The first sum is bounded as
    \begin{align*}
        \sum_{k=1}^K \frac{18 \Vmax^2 \ell_1(\iota_k, \delta)}{c_{1, k}}
        & \le 18 \Vmax \sqrt{SA \log KH} \sum_{k=1}^K \frac{\ell_1(\iota_k, \delta)}{\sqrt{k \ell(\iota_k, 1)}}
        \\
        & \le 18 \Vmax \sqrt{SA \log KH} \sum_{k=1}^K  \left( \frac{\log \frac{1}{\delta}}{c_1 \sqrt{k \ell(1, 1)}} + \frac{\sqrt{\ell_1(\iota_k, 1)}}{c_1 \sqrt{k}} \right)
        \\
        & \le 18 \Vmax \sqrt{SA \log KH} \cdot 2 \sqrt{K} \left( \frac{\log \frac{1}{\delta}}{c_1\sqrt{\ell_1(1, 1)} } + \frac{\sqrt{\ell_1(\iota_K, 1)}}{c_1} \right)
        \\
        &  =  36 \left( \frac{\log \frac{1}{\delta}}{c_1\sqrt{\ell_1(1, 1)}} + \frac{\sqrt{\ell_1(\iota_K, 1)}}{c_1} \right)\Vmax \sqrt{SAK \log KH}
        \, .
    \end{align*}
    The second term becomes
    \begin{align*}
        8c_{1, K} SA \log KH  & = 8 c_1 \Vmax \sqrt{SA K(\log KH) \ell_1(\iota_K, 1)} 
        \, .
    \end{align*}
    Therefore, we derive the total bound as
    \begin{align*}
        \Reg_M^{\Alg}(K, \delta) & \le \left( \frac{36 \log \frac{1}{\delta}}{c_1 \ell(1, 1)} + \frac{36}{c_1} + 8 c_1 \right)\Vmax \sqrt{KSA \ell(\iota_K, 1) \log KH}
        \\
        & \qquad + 72 \Vmax S^2 A \ell_2(K, \delta) \log 2KH
        \, .
    \end{align*}
\end{proof}

\subsection{Proof of Corollary~\ref{cor:eqo and hoeffding}}

\begin{corollary}[Restatement of Corollary~\ref{cor:eqo and hoeffding}]
\label{cor:eqo and hoeffding formal}
    Assume that $\sigsub^2 \le 2 \Vmax^2$.
    Set $c_{1, k} = c_1 \Vmax \sqrt{\frac{k \ell_1(k, 1)}{SA \log KH}}$ and $c_{2, k} = c_2 \Vmax \sqrt{\log 32HSAk}$ for some constants $c_1 > 0$ and $c_2 \ge 2$.
    Then, we have
    \begin{align*}
        \Reg_M^{\Alg}(K, \delta) & \le \Vmax\exp\left( \Ocal\left(\frac{1}{c_2^2} \log \frac{1}{\delta} \right)\right) +\Ocal \left( \sum_{(s, a) \in \Scal \times \Acal}\frac{c_2^2 \Vmax^2 \log HSAK}{\rgap(s, a)} \right)
        \\
        & \qquad + \Ocal\left( \left( \frac{{\Qrange}^2}{\vgap \Vmax^2 } + \frac{\Qdiff \lor \Valpha^2}{\Vmax} \right) \left( \frac{(\ell_1(\iota_K, \delta))^2}{c_1^2 \ell_1(1, 1)} \right) SA( \log KH) \right)
        \\
        & \qquad +\Ocal \left(\Vmax S^2 A \ell_2(K, \delta) \log KH \right)
        \, .
    \end{align*}
\end{corollary}

\begin{proof}
    By Theorem~\ref{thm:gap-independent bound formal}, we have
    \begin{align*}
        \Reg_M^{\Alg}(K, \delta)
        & = \Vmax (\kap(\delta) \land K) + \sum_{k = \kap(\delta)+1}^{\kapgap(\vgap, \delta) \land K} \frac{144 \Qrange \ell_1(\iota_k, \delta)}{c_{1, k}}
        \\
        & \qquad + \sum_{(s, a) \in \Scal \times \Acal}  \left( \frac{2048 c_{2, K}^2}{\rgap(s, a)} \land 64 c_{1, K} \log \frac{32 c_{1, K} }{\rgap(s, a)} \right) + 288 \Vmax S^2 A \ell_2(K, \delta) \log 2KH
        \\
        & = \Ocal \Bigg( \Vmax(\kap(\delta) \land K) + \sum_{k=\kap(\delta)}^{\kapgap(\vgap, \delta) \land K} \frac{\Qrange \ell_1(\iota_k, \delta)}{c_{1, k}}
        \\
        & \qquad + \sum_{(s, a) \in \Scal \times \Acal}\frac{\Vmax^2 \log HSAK}{\rgap(s, a)} + \Vmax S^2 A \ell_2(K, \delta) \log KH \Bigg)
        \, .
    \end{align*}
    We first bound $\kap(\delta) \land K$ using Lemma~\ref{lma:kappa}.
    First, $c_{1, k} = c_1 \Vmax \sqrt{\frac{k \ell_1(\iota_k, 1)}{SA \log kH}}< (2 \sqrt{13 \Qdiff} \lor 2\Valpha)\ell_1(1, \delta)$ and $k \le K$ implies 
    \begin{align*}
        k & < \frac{SA \log kH}{\ell_1(\iota_k, 1)} \left( \frac{(2 \sqrt{13 \Qdiff} \lor 2\Valpha)\ell_1(1, \delta)}{c_1 \Vmax} \right)^2
        \\
        & = \Ocal \left(  \frac{SA (\Qdiff \lor \Valpha^2) \log KH }{\ell_1(1, 1)} \left( \frac{\ell_1(1, \delta)}{c_1 \Vmax} \right)^2 \right)
        \, ,
    \end{align*}
    which implies that $\kappa_1(\delta) \land K = \Ocal \left(  \frac{SA (\Qdiff \lor \Valpha^2) \log KH }{\ell_1(1, 1)} \left( \frac{\ell_1(1, \delta)}{c_1 \Vmax} \right)^2 \right)$.
    For $\kappa_2(\delta)$, we may relax the condition on $c_{2, k}$ using the assumption $\sigsub^2 \le 2\Vmax^2$ to $c_{2, k} = c_2 \Vmax \sqrt{\log 32HSAk} < 2 \Vmax \sqrt{\log \frac{32HSA(2 + \log kH)^2}{\delta}}$.
    Using that $c_2 \ge 2$, this condition implies that $\log k  - 2 \log (2 + \log kH) < \frac{4}{c_2^2} \log \frac{1}{\delta}$, which in turn yields $\kappa_2(\delta) = \exp( \Ocal(\frac{1}{c_2^2} \log \frac{1}{\delta}))$.
    Hence, we have
    \begin{align*}
        \Vmax (\kap(\delta) \land K) = \Vmax\exp\left( \Ocal\left(\frac{1}{c_2^2} \log \frac{1}{\delta} \right)\right)  +  \Ocal \left( \frac{SA (\Qdiff \lor \Valpha^2) (\log KH) (\ell_1(1, \delta))^2 }{c_1^2 \Vmax \ell_1(1, 1)} \right)
        \, .
    \end{align*}
    Now, we bound $\Qrange \sum_{k=1}^{\kapgap(\vgap, \delta) \land K} \frac{\ell_1(\iota_k, \delta)}{c_{1, k}}$.
    First, for $k \le K$, we have
    \begin{align*}
        \frac{\ell_1(\iota_k, \delta)}{c_{1, k}}
        & = \frac{(\ell_1(\iota_k, 1) + \log \frac{1}{\delta}) \sqrt{SA \log kH}}{c_1 \Vmax \sqrt{k \ell_1(\iota_k, 1)}}
        \\
        & \le \frac{\sqrt{SA \log KH}}{c_1 \Vmax} \left(\sqrt{\ell_1(\iota_K, 1)} + \frac{\log \frac{1}{\delta}}{\sqrt{\ell_1(1, 1)}} \right) \frac{1}{\sqrt{k}}
        \, .
    \end{align*}
    Therefore, we have
    \begin{align*}
        & \sum_{k=1}^{\kapgap(\vgap, \delta) \land K} \frac{144 \Qrange \ell_1(\iota_k, \delta)}{c_{1, k}}
        \\
        & \le \sum_{k=1}^{\kapgap(\vgap, \delta) \land K} \frac{144 \Qrange  \sqrt{SA \log KH}}{c_1 \Vmax} \left( \sqrt{\ell_1(\iota_K, 1)} + \frac{\log \frac{1}{\delta}}{\sqrt{\ell_1(1, 1)}} \right)\frac{1}{\sqrt{k}}
        \, .
    \end{align*}
    Applying Lemma~\ref{lma:sum and integral clip} with $\varepsilon = 4 \vgap$ and $\alpha = \frac{1}{2}$, we have
    \begin{align*}
        & \sum_{k=1}^{\kapgap(\vgap, \delta) \land K} \frac{144 \Qrange  \sqrt{SA \log KH}}{c_1 \Vmax} \left( \sqrt{\ell_1(\iota_K, 1)} + \frac{\log \frac{1}{\delta}}{\sqrt{\ell_1(1, 1)}} \right)\frac{1}{\sqrt{k}}
        \\
        & = \Ocal\left( \frac{{\Qrange}^2 SA( \log KH) (\ell_1(\iota_K, \delta))^2}{c_1^2  \vgap \Vmax^2\ell_1(1, 1)} \right)
        \, .
    \end{align*}
    Combining all the bounds, we obtain that
    \begin{align*}
        \Reg_M^{\Alg}(K, \delta)
        & \le \Vmax\exp\left( \Ocal\left(\frac{1}{c_2^2} \log \frac{1}{\delta} \right)\right) +\Ocal \left( \sum_{(s, a) \in \Scal \times \Acal}\frac{\Vmax^2 \log HSAK}{\rgap(s, a)} \right)
        \\
        & \qquad + \Ocal\left( \left( \frac{{\Qrange}^2}{\vgap \Vmax^2 } + \frac{\Qdiff \lor \Valpha^2}{\Vmax} \right) \left( \frac{(\ell_1(\iota_K, \delta))^2}{c_1^2 \ell_1(1, 1)} \right) SA( \log KH) \right)
        \\
        & \qquad +\Ocal \left(\Vmax S^2 A \ell_2(K, \delta) \log KH \right)
        \, .
    \end{align*}
\end{proof}

\subsection{Proof of Corollary~\ref{cor:simchi levi}}

\begin{corollary}[Formal statement of Corollary~\ref{cor:simchi levi}]
\label{cor:simchi levi formal}
    Let $c_1 > 0, c_2 > 0$, $\alpha \in [\frac{1}{2}, 1]$, and $0 < \beta \le \alpha$ be constants.
    Set $c_{1, k} = c_1 \Vmax (\frac{k}{SA})^{\alpha}$ and $c_{2, k} = c_2 \sqrt{k^{\beta}}$.
    Then, we have the worst-case distributional regret bound of
    \begin{align*}
        \sup_{M \in \Mcal} \Reg_M^{\Alg}(K, \delta) & \le \Vmax \left( \frac{4 \Vmax }{c_2} \right)^{\frac{2}{\beta}} \ell_2(K, \delta)^{\frac{1}{\beta}}
        \\
        & \qquad + \frac{36 \Vmax}{c_1} \left( \frac{1}{1-\alpha} \land \log K \right) K^{1 - \alpha} (SA)^\alpha \ell_1(\iota_K, \delta) 
        \\
        & \qquad + \left( 16 c_1 \Vmax K^\alpha (SA)^{1-\alpha} \log KH \right) \land \left( 16 \sqrt{2}c_2 \sqrt{HSA} K^{\beta + \frac{1}{2}} \right)
        \\
        & \qquad + 72 \Vmax S^2 A \ell_{2}(K, \delta) \log 2KH
        \\
        & = \tilde{\Ocal}_{K, \delta} \left( \left( \log \frac{1}{\delta} \right)^{\frac{1}{\beta}} + K^{1-\alpha} \log \frac{1}{\delta} + K^{\alpha} \land K^{\beta + \frac{1}{2}} \right)
    \end{align*}
    and the instance-dependent distributional regret bound of
    \begin{align*}
        \Reg_M^{\Alg} (K, \delta) & \le \Vmax \left( \frac{4 \Vmax }{c_2} \right)^{\frac{2}{\beta}} \ell_2(K, \delta)^{\frac{1}{\beta}}
        \\
        &\qquad + 4 SA \left( \frac{1}{1 - \alpha} \land \log K \right) \left(\frac{36 \Qrange \ell_1(\iota_K, \delta)}{c_1 \Vmax} \right)^{\frac{1}{\alpha}} \left( \frac{1}{\vgap} \right)^{\frac{1}{\alpha} -1 }
        \\
        & \qquad + \sum_{(s, a) \times \Scal \times \Acal} \frac{2048 c_2^2 K^\beta}{\rgap(s, a)} + 288 \Vmax S^2 A \ell_2(K, \delta) \log 2KH
        \\
        & = \tilde{\Ocal}_{K, \delta, \rgap} \left( \left( \log \frac{1}{\delta} \right)^{\frac{1}{\beta}} + \left(\frac{1}{\vgap}\right)^{\frac{1}{\alpha} - 1} \left( \log \frac{1}{\delta} \right)^{\frac{1}{\alpha}} + \sum_{s, a} \frac{K^\beta}{\rgap(s, a)} \right)
        \, .
    \end{align*}
\end{corollary}

\begin{proof}
    \textbf{Worst-case bound.}
    By Theorem~\ref{thm:gap-independent bound formal}, we have
    \begin{align*}
        \Reg_{M}^{\Alg}(K, \delta)
        & \le \Vmax (\kappa(\delta) \land K) + \sum_{k=\kap(\delta) + 1}^K \frac{36 \Vmax^2 \ell_1(\iota_k, \delta)}{c_{1, k}} 
        \\
        & \qquad + \left( 16 c_{1, k} SA \log KH \right) \land \left ( 16 \sqrt{2} c_{2, K} \sqrt{HSAK} \right)
        \\
        & \qquad + 72 \Vmax S^2 \ell_2(K, \delta) \log 2KH
        \, .
    \end{align*}
    Using Lemma~\ref{lma:kappa}, we bound the first two terms as
    \begin{align*}
        \Vmax (\kappa(\delta) \land K) + \sum_{k=\kap(\delta) + 1}^K \frac{36 \Vmax^2 \ell_1(\iota_k, \delta)}{c_{1, k}}
        & \le \Vmax (\kappa_2(\delta) \land K) + \sum_{k= 1}^K \frac{36 \Vmax^2 \ell_1(\iota_k, \delta)}{c_{1, k}}
        \, .
    \end{align*}
    We first bound $\kappa_2(\delta)$.
    $c_{2, k} = c_2 \sqrt{k^\beta} < ( \sigsub + \frac{\Vmax}{3}) \sqrt{\ell_2(k, \delta)} \le 2 \Vmax \sqrt{\ell_2(k, \delta)}$ and $k \le K$ implies that $ k \le \left( \left(\frac{2\Vmax}{c_2} \right)^2 \ell_2(K, \delta)  \right)^{\frac{1}{\beta}}$, so we have
    $\kappa_2(\delta) \land K \le \left( \frac{2\Vmax}{c_2} \right)^{\frac{2}{\beta}} (\ell_2(K, \delta))^{\frac{1}{\beta}}$.
    \\
    Next, we bound $\sum_{k=1}^K \frac{36 \Vmax^2 \ell_1(\iota_k, \delta)}{c_{1, k}}$.
    The summand is upper-bounded by $\frac{36 \Vmax (SA)^{\alpha} \ell_1(\iota_K, \delta)}{c_1 k^\alpha}$.
    By Lemma~\ref{lma:sum and integral}, we have
    \begin{align*}
        \sum_{k=1}^K \frac{36 \Vmax (SA)^{\alpha} \ell_1(\iota_K, \delta)}{c_1  k^\alpha} \le \left( \frac{1}{1 - \alpha} \land \log K \right)\frac{36 \Vmax (SA)^{\alpha} K ^{1-\alpha} \ell_1(\iota_K, \delta)}{c_1 }
        \, .
    \end{align*}
    For the third term, we have $16 c_{1, k}SA \log KH = 16 c_1 \Vmax (SA)^{1-\alpha} K^\alpha \log KH$ and $16\sqrt{2} c_{2, K} \sqrt{HSAK} = 16 \sqrt{2}c_2 \sqrt{HSA} K^{\beta + \frac{1}{2}}$, the the fourth term does not need further modification.

    \textbf{Instance-dependent bound.}
    By Theorem~\ref{thm:instance dependent bound formal}, we have
    \begin{align*}
        \Reg_{\Mcal}^{\Alg} (K, \delta) & \le \Vmax ( \kap(\delta) \land K)  + \sum_{(s, a) \in \Scal \times \Acal}  \frac{2048 c_{2, K}^2}{\rgap(s, a)} 
        \\
        & \qquad + \sum_{k = \kap(\delta)+1}^{\kapgap(\vgap, \delta) \land K} \frac{144 \Qrange \ell_1(\iota_k, \delta)}{c_{1, k}} + 288 \Vmax S^2 A \ell_2(K, \delta) \log 2KH
        \, .
    \end{align*}
    We use Lemma~\ref{lma:kappa} to bound $\kap(\delta) \land K$.
    We have $\kappa_2(\delta) \land K \le \left( \frac{2\Vmax}{c_2} \right)^{\frac{2}{\beta}} (\ell_2(K, \delta))^{\frac{1}{\beta}}$ from the previous case.
    For $\kappa_1(\delta)$, we have that $c_{1, k} = c_1 \Vmax ( \frac{k}{SA})^\alpha < (2 \sqrt{13 \Qdiff} \lor 2 \Valpha)\ell_1(\iota_k, \delta)$ and $k \le K$ implies $k < SA \left( \frac{ (2 \sqrt{13 \Qdiff} \lor 2 \Valpha)\ell_1(\iota_K, \delta)}{c_1 \Vmax} \right)^{\frac{1}{\alpha}}$, which yields $\kappa_1(\delta) \land K < SA \left( \frac{ (2 \sqrt{13 \Qdiff} \lor 2 \Valpha)\ell_1(\iota_K, \delta)}{c_1 \Vmax} \right)^{\frac{1}{\alpha}}$.
    \\
    Next, we bound $\sum_{k = 1}^{\kapgap(\vgap, \delta) \land K} \frac{144 \Qrange \ell_1(\iota_k, \delta)}{c_{1, k}}$.
    We have
    \begin{align*}
        \sum_{k = 1}^{\kapgap(\vgap, \delta) \land K} \frac{144 \Qrange \ell_1(\iota_k, \delta)}{c_{1, k}}
        & =
        \sum_{k = 1}^{K} \frac{144 \Qrange \ell_1(\iota_k, \delta)}{c_{1, k}} \ind\left \{ \frac{144 \Qrange \ell_1(\iota_k, \delta)}{c_{1, k}}  \ge 4 \vgap \right\}
        \\
        & \le \sum_{k = 1}^{K} \frac{144 \Qrange (SA)^{\alpha} \ell_1(\iota_K, \delta)}{c_1 \Vmax k^\alpha} \ind\left\{\frac{ 144 \Qrange (SA)^{\alpha} \ell_1(\iota_K, \delta)}{c_1 \Vmax k^\alpha} \ge 4 \vgap \right\} 
        \\
        & \le SA \left( \frac{1}{1 - \alpha} \land \log K \right) \left(\frac{144 \Qrange \ell_1(\iota_K, \delta)}{c_1 \Vmax} \right)^{\frac{1}{\alpha}} \left( \frac{1}{4\vgap} \right)^{\frac{1}{\alpha} -1 }
        \\
        & = 4 SA \left( \frac{1}{1 - \alpha} \land \log K \right) \left(\frac{36 \Qrange \ell_1(\iota_K, \delta)}{c_1 \Vmax} \right)^{\frac{1}{\alpha}} \left( \frac{1}{\vgap} \right)^{\frac{1}{\alpha} -1 }
        \, ,
    \end{align*}
    where the last inequality is due to Lemma~\ref{lma:sum and integral clip}.
    \\
    For the remaining terms, we have $\sum_{s, a} \frac{2048 c_{2, K}^2}{\rgap(s, a)} = \sum_{s, a} \frac{2048 c_2^2 K^\beta}{\rgap(s, a)}$ and $288 \Vmax S^2 A \ell_2(K, \delta) \log 2KH$ does not need further modification.
\end{proof}

\section{Comparison with Existing Distributional Regret Bounds}
\label{appx:comparison}

In this section, we discuss the distributional regret bounds in \citet{simchi2023regret} and \citet{khodadadian2025tail} in more details and compare them with our results.
The bounds in \citet{simchi2022simple,khodadadian2025tail} are expressed in the form of $\PP( \Reg_M^{\Alg}(K) > x)$ as a function of $x$, so we restate their results to be consistent with our presentation, inverting their bounds from a function of $x$ to a function of $\delta$.

We present Theorem 3 in \citet{simchi2023regret}, which considers the MAB setting.
We note that this result is more general than the results in \citet{simchi2025simple}.

\begin{theorem}[Theorem 3 in \citet{simchi2023regret}, restated with respect to $\delta$]
\label{thm:simchi levi}
    Suppose $M \in \Bcal(\sigma)$.
    Set $c_{1, t} = \eta_1 \left( \frac{t}{A} \right)^\alpha \sqrt{\log A}$ and $c_{2, t} = \eta_2 \sqrt{t^\beta}$ for some constants $\eta_1, \eta_2 > 0$, $\alpha \in [\frac{1}{2}, 1)$, and $0 < \beta \le \alpha$.
    Then, the worst-case distributional regret bound is given as
    \begin{align*}
        \sup_{M \in \Bcal(\sigma)} \Reg_M^\Alg(T, \delta) & = \Ocal \Bigg( 
         A + \sigma \sqrt{AT \log \frac{A}{\delta}} + \left( \sigma^2 \log \frac{A}{\delta} \right)^{\frac{1}{\beta}}
        \\
        & \qquad \qquad  + \frac{1}{1 - \alpha}A^{1-\alpha}T^{\alpha} \sqrt{\log A} + \frac{1}{1-\alpha} A^\alpha T^{1 - \alpha}\frac{ \log \frac{A}{\delta} }{\sqrt{\log A}}
        \Bigg) 
        \\
        & = \tilde{\Ocal}_{T, \delta} \left( \left( \log \frac{1}{\delta} \right)^{\frac{1}{\beta}} + T^{1-\alpha} \log \frac{1}{\delta} + T ^\alpha \right)
    \end{align*}
    and the instance-dependent distributional regret bound is given as
    \begin{align*}
        \Reg_M^{\Alg}(T, \delta) & = \Ocal \Bigg( 
         A + \sum_{a \in \Acal} \frac{\sigma^2 \log \frac{A}{\delta} + T^\beta}{\gap(a)} 
        \\
        & \qquad \qquad + A \left( \frac{\sigma^2 \log \frac{A}{\delta}}{\gap_{\min} \sqrt{\log A}} \right)^{\frac{1}{\alpha}} + \left( \sigma^2 \log \frac{A}{\delta} \right)^{\frac{1}{\beta}} \Bigg)
        \\
        & = \tilde{\Ocal}_{T, \delta, \gap} \left( \left( \log \frac{1}{\delta} \right)^{\frac{1}{\beta}} + \left(\frac{1}{\gap_{\min} }\right)^{\frac{1}{\alpha}} \left( \log \frac{1}{\delta} \right)^{\frac{1}{\alpha}} + \sum_{a \in \Acal} \frac{\log \frac{1}{\delta} + T^\beta}{\gap(a)} \right)
        \, .
    \end{align*}
\end{theorem}

\paragraph{Comparison of Corollary~\ref{cor:simchi levi formal} and Theorem~\ref{thm:simchi levi}.}
We compare Corollary~\ref{cor:simchi levi formal} and Theorem~\ref{thm:simchi levi}, which take the same parameters, assuming that the instance considered in Corollary~\ref{cor:simchi levi formal} is an MAB instance.
We achieve the same order of regret in the worst-case bound up to logarithmic factors.
When considering the logarithmic factors, while Corollary~\ref{cor:simchi levi formal} includes an additional $\log T$ factor, our analysis for bandits (Theorem~\ref{thm:bandit 1 formal}) shows that the logarithmic factor can be improved in the MAB setting, even achieving a better logarithmic dependence than Theorem~\ref{thm:simchi levi}.
\\
We make several improvements to the instance dependent bound.
First, we reduce the $(\frac{1}{\gap_{\min}})^{\frac{1}{\alpha}}$ dependence to $(\frac{1}{\gap_{\min}})^{\frac{1}{\alpha} - 1}$, which is an improvement from $(\frac{1}{\gap_{\min}})^2$ to $\frac{1}{\gap_{\min}}$ under the standard choice of $\alpha = \frac{1}{2}$.
Additionally, we do not incur a $\sum_{a \in \Acal} \frac{\log \frac{1}{\delta}}{\gap(a)}$ term, thereby improving the coefficient of the $\log \frac{1}{\delta}$ term.

\begin{remark}
    The recent preprint of \citet{simchi2023regret} considers a more general parameter choice of $c_{2, t} = \sqrt{f(t)}$ for any increasing function $f(t) = \omega(\log t)$.
    Their bounds for $\PP(\Reg_M^{\Alg}(T) > x)$ involve an integral $\int_0^T \exp( - \frac{f(x \lor y)}{2\sigma^2}) \, dy$.
    While this term is difficult to invert directly, a coarse approximation of $T \exp( - \frac{f(x)}{2\sigma^2})$ yields that it translates to $f^{-1}(2\sigma^2\log \frac{T}{\delta})$, which corresponds with our $\tau_2(\delta)$ or $\kappa(\delta)$ in the sense that it coincides with the number of time steps required for $c_{2, t} \ge \sigma\sqrt{2 \log \frac{T}{\delta}}$ to hold.
    We take $f(t) =t^\beta$ in Theorem~\ref{thm:simchi levi}, an example considered in \citet{simchi2023regret}, for a more concrete comparison, and our improvement is valid even if we consider arbitrary parameters of $c_{2, k}$.
\end{remark}

We present Theorem 1 in \citet{khodadadian2025tail}, which considers the RL setting.

\begin{theorem}[Theorem 1 in \citet{khodadadian2025tail}, restated with respect to $\delta$]
\label{thm:khodadadian}
    Let $\Vmax = H$.
    Take $c_{1, k} = \infty$ and $c_{2, k} = H \sqrt{ (0.5\log 2) S + \mu (1 + k)^{\beta}}$ for some constants $\mu > 0$ and $\beta \in [0, 1]$.
    Then, it holds that
    \begin{align*}
        \Reg_M^{\Alg}(K, \delta) = \Ocal \left(  \frac{ H^6 SA (S + \mu K^{\beta})}{\gap_{\min}} + H^2 \sqrt{HK \log \frac{1}{\delta}} + H\left(\frac{1}{\mu} \log \frac{1}{\delta} \right)^{\frac{1}{\beta}} \right)
        \, .
    \end{align*}
\end{theorem}

\paragraph{Comparison of Corollary~\ref{cor:simchi levi formal} and Theorem~\ref{thm:khodadadian}.}
While we do not show the bounds for the exact same parameter choice with \citet{khodadadian2025tail}, we compare their results with our Corollary~\ref{cor:simchi levi formal}.
Taking $c_1 = \infty$ and $c_2 = H$ recovers their parameter except for an $S$ factor, which is not necessary in our work.
We observe multiple improvements over their results.
First, the bound of Theorem~\ref{thm:khodadadian} has a term linear in $\frac{SA}{\gap_{\min}}$, whereas we provide a fine-grained gap dependence of $\sum_{s, a} \frac{1}{\gap(s, a)}$.
Furthermore, while we have at most $H^2$ dependence when $\Vmax = H$, the order of $H$ is $H^6$ in Theorem~\ref{thm:khodadadian}.
In addition, their bound further incurs a $H^2\sqrt{HK \log \frac{1}{\delta}}$ factor, which does not appear in our bound.

\section{Concentration Inequalities}
\label{appx:concentration inequalities}

In this section, we provide general concentration inequalities for sub-exponential random variables.

\begin{lemma}
\label{lma:subexponential concentration}
    Let $\{ X_t \}_{t=1}^\infty$ be a martingale difference sequence adapted to filtration $\{ \Fcal_t \}_{t=0}^\infty$.
    Suppose there exists a sequence of predictable random variables $\{ \sigma_t\}_{t=1}^\infty$ with respect to $\{ \Fcal_t\}_{t=0}^\infty$ and a constant $\alpha \ge 0$ such that $X_t$ is $\Fcal_{t-1}$-conditionally $(\sigma_t, \alpha)$-sub-exponential for all $t \in \NN$.
    Then, for any $\lambda \in (0, 1/\alpha]$ and $\delta \in (0, 1]$, it holds that
    \begin{align*}
        \PP \left( \exists n \in \NN : \sum_{t=1}^n X_t \ge \frac{\lambda}{2} \sum_{t=1}^n \sigma_t^2 + \frac{1}{\lambda} \log \frac{1}{\delta} \right) \le \delta
        \, .
    \end{align*}
\end{lemma}

\begin{proof}
    The proof is a standard supermartingale method using Ville's inequality.
    \\
    Let $M_n = \exp\left(\sum_{t=1}^n \left( \lambda X_t - \frac{\lambda^2 \sigma_t^2}{2} \right) \right)$.
    Since $X_n$ is $\Fcal_{n-1}$-conditionally $(\sigma_n, \alpha)$-sub-exponential, we have $\EE[M_n \mid \Fcal_{n-1}] = M_{n-1} \EE[ \exp(\lambda X_n - \frac{\lambda^2 \sigma_n^2}{2}) \mid \Fcal_{n-1}] \le M_{n-1}$, showing that $\{M_n\}_{n=0}^\infty$ is a non-negative supermartingale.
    By Ville's inequality, we have $\PP(\exists n \in \NN : M_n \ge \frac{1}{\delta}) \le \delta$.
    Rearranging $M_n \ge \frac{1}{\delta}$ yields $\sum_{t=1}^n X_t \ge \frac{\lambda}{2}\sum_{t=1}^n \sigma_t^2 + \frac{1}{\lambda} \log \frac{1}{\delta}$, which completes the proof.
\end{proof}

Hoeffding's inequality can be viewed as a special case of Lemma~\ref{lma:subexponential concentration}.

\begin{lemma}[Hoeffding's inequality]
\label{lma:Hoeffdings inequality}
    Let $\{X_t \}_{t=1}^\infty$ be a sequence of $\Fcal_{t-1}$-conditionally $\sigma^2$-sub-Gaussian random variables adapted to a filtration $\{\Fcal_t\}_{t=0}^\infty$, where $\sigma^2 \ge 0$ is a constant.
    For fixed $n \in \NN$, we have
    \begin{align*}
        \PP \left(\sum_{t=1}^n X_t \ge \sigma \sqrt{ 2 n \log \frac{1}{\delta}} \right) \le \delta
        \, .
    \end{align*}
\end{lemma}

The following lemma shows that bounded random variables are sub-exponential.

\begin{lemma}
\label{lma:bounded is subexponential}
    Suppose a random variable $X$ lies in an interval $[-c, c]$ almost surely for some constant $ c \ge 0$.
    Suppose $\EE[ X] = 0$ and denote $V := \Var(X)$.
    Then, for any $\alpha > 0$, the random variable $X$ is $( \frac{e^{c/\alpha} - 1 - (c/\alpha)}{(c/\alpha)^2} \cdot 2 V, \alpha)$-sub-exponential.
    In particular, if $\alpha = c$, then $X$ is $(2(e-2)V, c)$-sub-exponential.
\end{lemma}

\begin{proof}
    By \lemmaref{lma:bounded is subpoisson}, we have $\EE [ \exp(\lambda' (X / c)) ] \le \exp((e^{\lambda'} - 1 - \lambda') \Var(X/c))$ for $\lambda' \ge 0$.
    It is also possible to obtain $\EE [ \exp( - \lambda' (X / c)) ] \le \exp((e^{\lambda'} - 1 - \lambda') \Var(X/c))$ by applying the lemma to $ - X / c$.
    By defining $\lambda = \lambda' / c$ in the first case and $\lambda = - \lambda' / c$ in the latter case, we obtain that
    \begin{align*}
        \EE[ \exp( \lambda X )] \le \exp \left( \frac{e^{c |\lambda|} - 1 - c |\lambda|}{c^2} \Var (X) \right)
    \end{align*}
    for all $\lambda \in \RR$.
    By \lemmaref{lma:g is increasing}, one has $(e^{c|\lambda|} - 1 - c|\lambda| )/ (c\lambda)^2 \le (e^{c / \alpha} - 1 - c \alpha) / (c/\alpha)^2$ when $|\lambda| \le \frac{1}{\alpha}$.
    Therefore, for all $\lambda \in [-1/\alpha, 1 / \alpha]$, it holds that
    \begin{align*}
        \EE[ \exp( \lambda X )] \le \exp \left( \lambda^2  \cdot \frac{e^{\frac{c}{\alpha}} - 1 - \frac{c}{\alpha}}{\left(\frac{c}{\alpha}\right)^2} V \right)
        \, ,
    \end{align*}
    which proves that $X$ is $( \frac{e^{c/\alpha} - 1 - (c/\alpha)}{(c/\alpha)^2} \cdot 2 V, \alpha)$-sub-exponential.
\end{proof}
By combining \lemmaref{lma:subexponential concentration} and \lemmaref{lma:bounded is subexponential}, we obtain the well-known variant of Freedman's inequality.

\begin{lemma}
\label{lma:freedman inequality}
    For a constant $c \ge 0$, let $\{X_t\}_{t=1}^\infty$ be a martingale difference sequence adapted to filtration $\{ \Fcal_t \}_{t=0}^\infty$ that satisfies $X_t \in [- c, c]$ almost surely for all $t \in \NN$.
    Then, for any $\lambda \in (0, 1 / c]$ and $\delta \in (0, 1]$, it holds that
    \begin{align*}
        \PP \left( \exists n \in \NN : \sum_{t=1}^T X_t \ge (e - 2) \lambda  \sum_{t=1}^n \Var(X_t \mid \Fcal_{t-1}) + \frac{1}{\lambda} \log \frac{1}{\delta} \right)
        \, .
    \end{align*}
\end{lemma}

The following inequality allows us to use a Hoeffding-like concentration bound by clipping by a constant $c$.

\begin{lemma}
\label{lma:time uniform Hoeffding for exponential}
    Let $\{X_t\}_{t=1}^\infty$ be a martingale difference sequence adapted to filtration $\{ \Fcal_t \}_{t=1}^\infty$.
    Suppose there exist constants $\sigma, \alpha > 0 $ such that $X_t$ is $\Fcal_{t-1}$-conditionally $(\sigma, \alpha)$-sub-exponential for all $t \in \NN$.
    Then, for any $c > 0$ and $\delta \in (0, 1]$, it holds that
    \begin{align*}
        \PP \left( \exists n \in \NN : \left( \frac{1}{n} \sum_{t=1}^n X_t \right) \land c \ge 2 \left(\sigma \lor \sqrt{\frac{\alpha c}{2}} \right) \sqrt{\frac{1}{n} \log \frac{2 (\log e^2 n)^2}{\delta}} \right) \le \delta
        \, .
    \end{align*}
\end{lemma}
\begin{remark}
    It is common to use $\frac{\sigma}{\sqrt{n}} + \frac{\alpha}{n}$-type bounds for sub-exponential random variables~\citep{jia2021multi}.
    While this lemma provides a simpler bound, we note that its rate is not necessarily optimal, especially when the value of $\alpha$ is large and known.
    We use this lemma for a simpler design of the algorithm.
\end{remark}
\begin{proof}
    Let $\sigma' = \sigma \lor \sqrt{\frac{\alpha c}{2}}$, $r = \sqrt{\frac{2\alpha}{c} \lor \frac{1}{\log \frac{2}{\delta}}}$, and $\alpha' = \sigma' r$.
    Note that $\sigma' \ge \sqrt{\frac{\alpha c}{2}}$ and $r \ge \sqrt{\frac{2 \alpha}{c}}$ imply $\alpha' \ge \sqrt{\frac{\alpha c}{2}} \cdot \sqrt{\frac{2 \alpha}{c}} = \alpha$.
    We apply \lemmaref{lma:subexponential concentration} to a sequence of $\{ (\lambda_j, \delta_j)\}_{j=0}^\infty$ with $\delta_j := \delta / (2(j+1)^2)$ and $\lambda_j := \frac{e^{-j/2}}{\alpha'}$.
    By the union bound, we have
    \begin{align*}
        \PP \left( \exists n \in \NN, \exists j \in \NN \cup \{0 \} : \sum_{t=1}^n X_t \ge \frac{\lambda_j \sigma^2 n}{2} + \frac{1}{\lambda_j} \log \frac{1}{\delta_j} \right) 
        \le \delta
        \, .
    \end{align*}
    The right-hand side is equal to
    \begin{align}
        \frac{\lambda_j \sigma^2 n}{2} + \frac{1}{\lambda_j} \log \frac{1}{\delta_j}
        & = \frac{e^{-j/2} \sigma^2n}{2 \alpha'} + e^{j/2} \alpha' \log \frac{2(j+1)^2}{\delta}
        \label{eq:hoeffding-like bound 1}
        \, .
    \end{align}
    Suppose $n \ge r^2 \log \frac{2}{\delta} = ( \frac{2 \alpha}{c}\log \frac{2}{\delta}) \lor 1$.
    Then, there exists an integer $j_n \ge 0$ such that
    \begin{align}
       r^2 e^{j_n} \log \frac{2(j_n+1)^2}{\delta} \le n < r^2 e^{j_n+1} \log \frac{2(j_n + 2)^2}{\delta}
       \label{eq:choice of j_n}
        \, .
    \end{align}
    By the first part of inequality~\eqref{eq:choice of j_n}, we bound the last term in Eq.~\eqref{eq:hoeffding-like bound 1} as
    \begin{align*}
        e^{j_n/2}\alpha' \log \frac{2(j_n+1)^2}{\delta}
        & = \left(e^{j_n} \sigma'^2 r^2 \log \frac{2(j_n+1)^2}{\delta} \right)^{\frac{1}{2}} \sqrt{\log \frac{2(j_n+1)^2}{\delta}}
        \\
        & \le \sigma' \sqrt{n \log \frac{2(j_n+1)^2}{\delta}}
        \, .
    \end{align*}
    By the second part of inequality~\eqref{eq:choice of j_n}, we bound the first term in the right-hand side of Eq.~\eqref{eq:hoeffding-like bound 1} as
    \begin{align*}
        \frac{e^{-j_n/2} \sigma^2 n}{2 \alpha'}
        & \le \frac{e^{-j_n/2}{\sigma'}^2 n}{2\alpha'} 
        \\
        & = \frac{\sigma' n}{2} \left( e^{j_n} r^2 \right)^{-\frac{1}{2}}
        \\
        & \le \frac{\sigma' n}{2} \sqrt{ \frac{e}{n} \log \frac{2(j_n+2)^2}{\delta}}
        \\
        & = \frac{\sqrt{e} \sigma'}{2} \sqrt{n \log \frac{2(j_n+2)^2}{\delta}}
        \, .
    \end{align*}
    In addition, the first part of inequality~\eqref{eq:choice of j_n} also implies $j_n \le \log n$ as
    \begin{align*}
        e^{j_n} \le e^{j_n} \cdot \frac{\log \frac{2(j_n+1)^2}{\delta}}{\log \frac{2}{\delta}} \le e^{j_n} r^2 \log \frac{2(j_n+1)^2}{\delta} \le n
        \, .
    \end{align*}
    Therefore, for $n \ge r^2 \log \frac{2}{\delta}$, we have
    \begin{align*}
        \frac{\lambda_{j_n} \sigma^2 n}{2} + \frac{1}{\lambda_{j_n}} \log \frac{1}{\delta_{j_n}}
        & \le \frac{\sqrt{e} \sigma'}{2} \sqrt{n \log \frac{2(j_n+2)^2}{\delta}} + \sigma' \sqrt{n \log \frac{2(j_n+1)^2}{\delta}}
        \\
        & \le 2 \sigma' \sqrt{n \log \frac{2(j_n+2)^2}{\delta}}
        \\
        & \le 2 \sigma' \sqrt{n \log \frac{2(\log n + 2)^2}{\delta}}
        \, .
    \end{align*}
    It implies that
    \begin{align*}
        & \PP \left( \exists n \in \NN, n \ge r^2 \log \frac{2}{\delta}: \sum_{t=1}^n X_t \ge 2 \sigma' \sqrt{n \log \frac{2(\log n + 2)^2}{\delta}} \right)
        \\
        & \le \PP \left( \exists n \in \NN, n \ge r^2 \log \frac{2}{\delta}: \sum_{t=1}^n X_t \ge \frac{\lambda_{j_n} \sigma^2 n}{2} + \frac{1}{\lambda_{j_n}} \log \frac{1}{\delta_{j_n}} \right)
        \\
        & \le 
        \PP \left( \exists n \in \NN, \exists j \in \NN \cup \{0 \} : \sum_{t=1}^n X_t \ge \frac{\lambda_j \sigma^2 n}{2} + \frac{1}{\lambda_j} \log \frac{1}{\delta_j} \right) 
        \\
        & \le \delta
    \end{align*}
    Finally, suppose $n < r^2 \log \frac{2}{\delta} = (\frac{2 \alpha}{c} \log \frac{2}{\delta}) \lor 1$, which is equivalent to $n < \frac{2 \alpha}{c} \log \frac{2}{\delta}$ since we must have $n \ge 1$.
    Then, we have
    \begin{align*}
        2 \sigma' \sqrt{\frac{1}{n} \log \frac{2(\log n +2)^2}{\delta}}
        & > 2 \sigma' \sqrt{ \frac{c\log \frac{2(\log n +2)^2}{\delta}}{2 \alpha \log \frac{2}{\delta}}}
        \\
        & \ge 2 \sqrt{\frac{\alpha c}{2}} \sqrt{\frac{c}{2 \alpha}}
        \\
        & = c
        \, .
    \end{align*}
    Hence, $\left( \frac{1}{n} \sum_{t=1}^n X_t \right) \land c \ge 2 \sigma' \sqrt{\frac{1}{n} \log \frac{2(\log n +2)^2}{\delta}}$ is possible only when $n \ge r^2 \log \frac{2}{\delta}$.
    Finally, we have
    \begin{align*}
        & \PP \left( \exists n \in \NN: \left( \frac{1}{n} \sum_{t=1}^n X_t \right) \land c \ge 2 \sigma' \sqrt{\frac{1}{n} \log \frac{2(\log n +2)^2}{\delta}} \right)
        \\
        & = \PP \left( \exists n \in \NN, n \ge r^2 \log \frac{2}{\delta}: \left( \frac{1}{n} \sum_{t=1}^n X_t \right) \land c \ge 2 \sigma' \sqrt{\frac{1}{n} \log \frac{2(\log n +2)^2}{\delta}} \right)
        \\
        & \le \PP \left( \exists n \in \NN, n \ge r^2 \log \frac{2}{\delta}: \sum_{t=1}^n X_t \ge 2 \sigma' \sqrt{n \log \frac{2(\log n +2)^2}{\delta}} \right)
        \\
        & \le \delta
        \, .
    \end{align*}
\end{proof}

The following is a time-uniform concentration inequality for sub-Gaussian random variables whose constant factor is asymptotically tight.

\begin{lemma}
\label{lma:all eta concentration}
    Let $\{ X_t \}_{t=1}^\infty$ be a martingale difference sequence adapted to filtration $\{ \Fcal_t \}_{t=0}^\infty$.
    Suppose $X_t$ is $\Fcal_{t-1}$-conditionally $1$-sub-Gaussian for all $ t \in \NN$, meaning that for all $\lambda \in \RR$, we have $\EE[\exp(\lambda X_t) \mid \Fcal_{t-1}] \le \frac{\lambda^2}{2}$ almost surely.
    Then, for a constant $\alpha > 1$ and $\eta \in (0, e)$, we have
    \begin{align*}
        \PP \left( \exists n \in \NN : \sum_{t=1}^n X_t \ge \sqrt{2 ( 1 + \eta) n \log \frac{4  (2 + \log \frac{1}{\eta})^{2}(1 + \frac{2e}{\eta} \log n)^{2}}{\delta}} \right) \le \delta
        \, .
    \end{align*}
\end{lemma}

\begin{proof}
    For fixed $\lambda > 0$, \lemmaref{lma:subexponential concentration} implies that the following inequality holds for all $n \ge \NN$ with probability at least $1 - \delta$:
    \begin{align*}
        \sum_{t=1}^n X_t \le \frac{\lambda n}{2} + \frac{1}{\lambda} \log \frac{1}{\delta}
        \, .
    \end{align*}
    We first fix $\eta \in (0, 1]$ and apply the lemma for a sequence of $\{(\lambda_j, \delta_j)\}_{j=1}^\infty$ with $\lambda_j = \sqrt{2 ( 1 + \eta)^{-j + 1} \log \frac{2j^2}{\delta}}$ and $\delta_j = \frac{\delta}{2 j^2}$.
    Taking the union bound, we obtain that $\sum_{t=1}^n X_t \le \frac{\lambda_j n}{2} + \frac{1}{\lambda_j} \log \frac{2 j^2}{\delta}$ holds for all $n \in \NN$ and $j \in \NN$ with probability at least $1 - \delta$,
    which is equivalent to
    \begin{align}
        \sum_{t=1}^n X_t \le \left( \frac{n}{(1 + \eta)^{\frac{j-1}{2}}} + (1 + \eta)^{\frac{j-1}{2}} \right)\sqrt{\frac{1}{2} \log \frac{2j^2}{\delta}}
        \label{eq:eta concentration 1}
        \, .
    \end{align}
    Taking $j_n = \lceil \log_{1 + \eta} n \rceil$, we have $(1 + \eta)^{j_n} \ge n$ and $(1 + \eta)^{j_n - 1} \le n$.
    Then, we have
    \begin{align*}
        \frac{n}{(1 + \eta)^{\frac{j_n-1}{2}}} + ( 1 + \eta)^{\frac{j_n - 1}{2}}
        & \le \sqrt{(1 + \eta)n} + \sqrt{n}
        \\
        & \le 2 \sqrt{( 1 + \eta)n}
        \, .
    \end{align*}
    In addition, we have $j_n \le 1 + \log_{1 + \eta}n = 1 + \frac{\log n}{\log (1 + \eta)} \le 1  +\frac{2}{\eta}\log n$, where we use that $\frac{x}{2} \le \log ( 1 +x)$ for $x \in (0, 1]$.
    Hence, under the event of Eq.~\eqref{eq:eta concentration 1}, we have
    \begin{align*}
        \sum_{t=1}^n X_t \le \sqrt{2 ( 1 + \eta) n \log \frac{2(1 + \frac{2}{\eta} \log n)^{2}}{\delta}}
        \,. 
    \end{align*}
    \\
    Now, we take a sequence $\eta_i = e^{-i+1}$ for $i \in \NN$ and take the union bound over $i$, where we assign probability $\frac{\delta}{2i^2}$ for each $i$.
    Then, with probability at least $1 - \delta$, the following inequality holds for all $n \in \NN$ and $i \in \NN$:
    \begin{align*}
        \sum_{t=1}^n X_t \le \sqrt{2 ( 1 + \eta_i) n \log \frac{4 i^{2}(1 + \frac{2}{\eta_i} \log n)^{2}}{\delta}}
        \,. 
    \end{align*}
    For any given $\eta \in (0, e)$, we take $i^* = 1 + \lceil \log \frac{1}{\eta} \rceil$, so that $\frac{\eta}{e} \le e^{-i^*+1} \le \eta$.
    Then, we have
    \begin{align*}
        \sum_{t=1}^n X_t & \le \sqrt{2 ( 1 + \eta_{i^*}) n \log \frac{4 (i^*)^{2}(1 + \frac{2}{\eta_{i^*}} \log n)^{2}}{\delta}}
        \\
        & \le \sqrt{2 ( 1 + \eta) n \log \frac{4 (2 + \log \frac{1}{\eta})^{2}(1 + \frac{2e}{\eta} \log n)^{2}}{\delta}}
        \, .
    \end{align*}
\end{proof}

\section{Technical Lemmas}

The following lemma encapsulates a procedure that appears multiple times in this paper when bounding the variance of a random variable, and is a minor generalization of Lemma 27 in~\citet{lee2025minimax}.
\begin{lemma}
\label{lma:variance decomposition}
    Let $c \ge 0$ be a constant and $\Fcal$ be a $\sigma$-algebra.
    Let $Z$ be a random variable such that $0 \le \EE[Z\mid \Fcal] \le c$ holds almost surely.
    Let $X$ and $Y$ be random variables that satisfy 
    \begin{align*}
        X = (Y + \EE[Z \mid \Fcal] ) \land c    
    \end{align*}
    and $X \ge 0$.
    Then, the variance of $Z$ is upper bounded by
    \begin{align*}
        \Var(Z \mid \Fcal) & \le \EE[ Z^2 \mid \Fcal] - X^2 + 2 c ( Y \lor 0)
        \, .
    \end{align*}
\end{lemma}
\begin{proof}
    We have
    \begin{align*}
        \Var(Z \mid \Fcal)
        & = \EE[ Z^2 \mid \Fcal] - (\EE[ Z \mid \Fcal])^2
        \\
        & = \EE[ Z^2 \mid \Fcal] - X^2 + X^2 - (\EE[ Z \mid \Fcal])^2
        \\
        & = \EE[ Z^2 \mid \Fcal] - X^2 + (X + \EE[ Z \mid \Fcal]) (X - \EE[ Z \mid \Fcal])
        \, .
    \end{align*}
    Note that we have $0 \le X + \EE[ Z \mid \Fcal] \le 2c$.
    If $X - \EE[ Z \mid \Fcal ] \le 0$, the last term is at most 0.
    If $X - \EE[ Z \mid \Fcal ] \ge 0$, then we have $0 \le X- \EE[Z \mid \Fcal] \le Y$.
    Combining these cases, we obtain $(X + \EE[ Z \mid \Fcal]) (X - \EE[ Z \mid \Fcal])\le 2c (Y \lor 0)$, completing the proof.
\end{proof}

\begin{lemma}
\label{lma:sum and integral}
    For $\alpha \in [1/2 , 1]$ and $n \ge 2$, one has
    \begin{align*}
        \sum_{t=2}^n t^{-\alpha} \le \left( \frac{1}{1-\alpha} \land \log n \right) n^{1-\alpha}
        \, .
    \end{align*}
\end{lemma}

\begin{proof}
    The result holds for $\alpha = 1$ since $\sum_{t=2}^n \frac{1}{t} \le \log n$.
    Suppose $\alpha \in [1/2, 1)$.
    By the integration technique, we have
    \begin{align*}
        \sum_{t=2}^n t^{-\alpha} 
        & \le \int_1^n x^{-\alpha} \, dx
        \\
        & = \frac{n^{1-\alpha} - 1}{1 - \alpha} 
        \, .
    \end{align*}
    On the other hand, we have
    \begin{align*}
        \sum_{t=2}^n t^{-\alpha} 
        & \le \sum_{t=2}^n (1 + (1 - \alpha) \log t) t^{-\alpha}
        \\
        & \le \int_1^n \frac{1 + (1 - \alpha) \log x}{x^\alpha} \, dx
        \\
        & = n^{1-\alpha} \log n
        \, ,
    \end{align*}
    where the second inequality holds since $\frac{1 + (1 - \alpha) \log x}{x^\alpha}$ is decreasing on $x \ge 1$ when $\alpha \in (1/2, 1]$, and
    the last inequality is due to $\frac{d}{dx} (x^{1 - \alpha} \log x) = (1 + (1 - \alpha) \log x)x^{-\alpha}$.
    Therefore, we have
    \begin{align*}
        \sum_{t=2}^n t^{-\alpha} \le \frac{n^{1-\alpha}-1}{1-\alpha} \land n^{1-\alpha} \log n \le \left( \frac{1}{1-\alpha} \land \log n \right) n^{1-\alpha}
        \, .
    \end{align*}
\end{proof}

\begin{lemma}
\label{lma:sum and integral clip}
    For $\alpha \in [1/2, 1]$, $\varepsilon > 0$, $c > 0$, and $N \in \NN$, the following inequality holds:
    \begin{align*}
        \sum_{n=1}^{N} \frac{c}{n^{\alpha}} \ind \left\{ \frac{c}{n^{\alpha}} \ge \varepsilon \right\} & \le \left(\frac{1}{1 - \alpha} \land \left(\frac{1}{\alpha} \log \left(\frac{c}{\varepsilon} \right) \right) \land \log N \right) c^{\frac{1}{\alpha}} \left(\frac{1}{\varepsilon} \right)^{\frac{1}{\alpha} - 1}
        \, .
    \end{align*}
\end{lemma}

\begin{proof}
    Let $n_0 = \lfloor (c / \varepsilon )^{1/\alpha}\rfloor \land N$.
    Then, we have
    \begin{align*}
        \sum_{n=1}^{N} \frac{c}{n^{\alpha}} \ind \left\{ \frac{c}{n^{\alpha}} \ge \varepsilon \right\} 
        & = \sum_{n=0}^{n_0} \frac{c}{n^{\alpha}}
        \\
        & \le \left( \frac{1}{1 - \alpha} \land \log n_0 \right) c n_0^{1- \alpha}
        \\
        & \le \left( \frac{1}{1 - \alpha} \land  \log \left( \left( \frac{c}{\varepsilon}\right)^{\frac{1}{\alpha}} \land N \right) \right) c \left(\frac{c}{\varepsilon} \right)^{\frac{1 - \alpha}{\alpha}}
        \\
        & = \left(\frac{1}{1 - \alpha} \land \left( \frac{1}{\alpha} \log \frac{c}{\varepsilon} \right) \land \log N \right) c^{\frac{1}{\alpha}} \left(\frac{1}{\varepsilon} \right)^{\frac{1}{\alpha} - 1}
        \, ,
    \end{align*}
    where the first inequality is due to Lemma~\ref{lma:sum and integral}.
\end{proof}

\section{Auxiliary Lemmas}

\begin{lemma}[Lemma 30 in~\citet{chen2021implicit}]
    \label{lma:variance of product}
    For any two random variables $X, Y$, we have:
    \begin{align*}
        \Var(XY) \le 2 \Var(X) \| Y \|_{\infty}^2 + 2 \EE[X]^2 \Var(Y)
        \, .
    \end{align*}
\end{lemma}

\begin{lemma}[Lemma 30 in~\citet{lee2025minimax}]
\label{lma:bound on eta k}
    For any sequence of $K$ trajectories, we have
    \begin{align*}
        \sum_{k=1}^K \ind\{ \eta^k < H+1\} \le SA \log_2 2H
        \, .
    \end{align*}
\end{lemma}

\begin{lemma}[Lemma (3.1) in \citet{freedman1975tail}]
    \label{lma:g is increasing}
    Let $g(0) = \frac{1}{2}$ and $g(x) = (e^x - 1 - x) / x^2$ for $x \ne 0$.
    Then, $g$ is increasing.
\end{lemma}

\begin{lemma}[Equation (3.5) in \citet{freedman1975tail}]
    \label{lma:bounded is subpoisson}
    For $\lambda \ge 0$ and a random variable $X$ satisfying $X \le 1$ and $\EE [X ] \le 0$, we have $\EE[ \exp( \lambda X) ] \le \exp ( (e^{\lambda} - 1 - \lambda) \Var (X))$.
\end{lemma}

\begin{lemma}[Hoeffding's lemma, Eq. (3.16) in \citet{hoeffding1963probability}]
\label{lma:hoeffding's lemma}
    Let $X$ be a real-valued random variable that satisfies $a \le X \le b$ for some real numbers $a$ and $b$, and assume $\EE[ X] = 0$.
    Then, $X$ is $(\frac{b - a}{2})^2$-sub-Gaussian.
\end{lemma}

\end{document}